\documentclass[english]{article} %

\usepackage{fullpage}

\usepackage[numbers, compress]{natbib}
\usepackage{geometry}
\usepackage[utf8]{inputenc} % allow utf-8 input
\usepackage[T1]{fontenc}    % use 8-bit T1 fonts
\usepackage{hyperref}       % hyperlinks
\usepackage{url}            % simple URL typesetting
\usepackage{booktabs}       % professional-quality tables
\usepackage{amsfonts}       % blackboard math symbols
\usepackage{nicefrac}       % compact symbols for 1/2, etc.
\usepackage{microtype}      % microtypography
\usepackage{xcolor}         % colors

\usepackage{makecell}
\usepackage{graphicx}
\usepackage{xspace}
\xspaceaddexceptions{[]\{\}}
\usepackage{algorithm}
\usepackage{algorithmic}
\usepackage{tablefootnote}
\usepackage{multirow}
\usepackage{colortbl}
\definecolor{bgcolor}{rgb}{0.96, 0.87, 0.7}
\usepackage{amssymb}
\usepackage{amsmath}
\usepackage{amsthm}
\usepackage{mathrsfs}
\usepackage{bm}
\usepackage{cleveref}

\allowdisplaybreaks
\newtheorem{theorem}{Theorem}
\newtheorem{lemma}{Lemma}
\newtheorem{corollary}{Corollary}

\newtheorem{definition}{Definition}
\newtheorem{assumption}{Assumption}
\newtheorem{proposition}{Proposition}

{\hspace*{\fill}$\Box$\par\vspace{4mm}}
{\hspace*{\fill}$\Box$\par}

\newcommand{\CDP}{{\sf CDP-SGD}\xspace}
\newcommand{\DPSGD}{{\sf DP-SGD}\xspace}

\newcommand{\algname}[1]{{\sf #1}\xspace}

\newcommand{\ns}[1]{\| #1 \|^2}
\newcommand{\nsb}[1]{\Big\| #1 \Big\|^2}
\newcommand{\nsB}[1]{\left\| #1 \right\|^2}

\newcommand{\n}[1]{\| #1 \|}
\newcommand{\inner}[2]{\langle #1, #2 \rangle}

\newcommand{\innerB}[2]{\left\langle #1, #2 \right\rangle}
\newcommand{\calC}{{\mathcal C}}
\newcommand{\cC}{{\cal C}}

\newcommand{\eat}[1]{}

\usepackage{threeparttable}
\usepackage{thm-restate}

\usepackage{amsmath,amsfonts,bm}

\def\1{\bm{1}}

\def\vxi{{\bm{\xi}}}

\def\va{{\bm{a}}}

\makeatletter
\newcommand{\ve}{\@ifnextchar\bgroup{\velong}{{\bm{e}}}}
\newcommand{\velong}[1]{{\bm{#1}}}
\makeatother

\def\vg{{\bm{g}}}

\def\vs{{\bm{s}}}

\def\vu{{\bm{u}}}
\def\vv{{\bm{v}}}
\def\vw{{\bm{w}}}
\def\vx{{\bm{x}}}

\def\mI{{\bm{I}}}

\DeclareMathAlphabet{\mathsfit}{\encodingdefault}{\sfdefault}{m}{sl}
\SetMathAlphabet{\mathsfit}{bold}{\encodingdefault}{\sfdefault}{bx}{n}

\def\gA{{\mathcal{A}}}
\def\gB{{\mathcal{B}}}
\def\gC{{\mathcal{C}}}
\def\gD{{\mathcal{D}}}

\def\gI{{\mathcal{I}}}

\def\gM{{\mathcal{M}}}
\def\gN{{\mathcal{N}}}

\def\gP{{\mathcal{P}}}
\def\gQ{{\mathcal{Q}}}
\def\gR{{\mathcal{R}}}
\def\gS{{\mathcal{S}}}

\newcommand{\E}{\mathbb{E}}
\newcommand{\R}{\mathbb{R}}

\newcommand{\dotp}[2]{\left<#1, #2\right>}

\newcommand{\norm}[1]{\left\| #1 \right\|}

\newcommand{\dd}{\textup{\textrm{d}}}

\newcommand{\red}[1]{{\color{red} #1}}

\newcommand{\purple}[1]{{\color{purple} #1}}

\def\eqref#1{(\ref{#1})}

\newcommand{\blue}[1]{\textcolor{blue}{#1}}

\newcommand{\size}[2]{{\fontsize{#1}{0}\selectfont#2}}
\newcommand{\soteriafl}{{\sf SoteriaFL}\xspace}
\newcommand{\dataset}[1]{{\tt #1}\xspace}

\newcommand{\cdpGD}{{\sf CDP-GD}\xspace}
\newcommand{\cdpSGD}{{\sf CDP-SGD}\xspace}
\newcommand{\soteriaflGD}{{\sf SoteriaFL-GD}\xspace}
\newcommand{\soteriaflSGD}{{\sf SoteriaFL-SGD}\xspace}
\newcommand{\soteriaflSVRG}{{\sf SoteriaFL-SVRG}\xspace}
\newcommand{\soteriaflSAGA}{{\sf SoteriaFL-SAGA}\xspace}

\newcommand{\RomanNumeralCaps}[1]
    {\MakeUppercase{\romannumeral #1}}

\usepackage{todonotes}

\title{\bf \textsf{SoteriaFL}: A Unified Framework for Private Federated Learning with Communication Compression}

\author{
	Zhize Li\thanks{Department of Electrical and Computer Engineering, Carnegie Mellon University, Pittsburgh, PA 15213, USA; Emails:
		\texttt{\{zhizel,~boyuel,~yuejiec\}@andrew.cmu.edu}.} \\
	CMU \\
	\and
	Haoyu Zhao\thanks{Department of Computer Science, Princeton University, Princeton, NJ 08540, USA; Email: \texttt{haoyu@princeton.edu}.} \\
	Princeton\\
	\and
	Boyue Li\footnotemark[1] \\
	CMU\\
	\and
	Yuejie Chi\footnotemark[1] \\
	CMU \\
}

\date{June 2022; Revised October 2022}

\begin{document}

\maketitle

\begin{abstract}
To enable large-scale machine learning in bandwidth-hungry environments such as wireless networks, significant progress has been made recently in designing communication-efficient federated learning algorithms with the aid of communication compression. On the other end, privacy-preserving, especially at the client level, is another important desideratum that has not been addressed simultaneously in the presence of advanced communication compression techniques yet. 
In this paper, we propose a unified framework that enhances the communication efficiency of private federated learning with communication compression.
    Exploiting both general compression operators and local differential privacy, we first examine a simple algorithm that applies compression directly to 
     differentially-private stochastic gradient descent, and identify its limitations. We then propose a unified framework \soteriafl for private federated learning, which accommodates a general family of local gradient estimators including popular stochastic variance-reduced gradient methods and the state-of-the-art shifted compression scheme.
     We provide a comprehensive characterization of its performance trade-offs in terms of privacy, utility, and communication complexity, where \soteriafl is shown to achieve better communication complexity without sacrificing privacy nor utility than other private federated learning algorithms without communication compression.
    
\end{abstract}


%\tableofcontents

\section{Introduction}
With the proliferation of mobile and edge devices, federated learning (FL)~\citep{konevcny2016federatedlearning, mcmahan2017communication} has recently emerged as a disruptive paradigm for training large-scale machine learning models over a vast amount of geographically distributed and heterogeneous devices. For instance, Google uses FL in the Gboard mobile keyboard for next word predictions~\citep{hard2018federated}.
FL is often modeled as a distributed optimization problem~\citep{konevcny2016federatedoptimization, konevcny2016federatedlearning, mcmahan2017communication, kairouz2019advances, wang2021field}, aiming to solve
\begin{equation} \label{eq:prob}
\min_{\vx\in\R^d}\left\{ f(\vx; D) := \frac{1}{n}\sum_{i=1}^n f(\vx; D_i)\right\},~~ \text{where~} f(\vx; D_i) := \frac{1}{m}\sum_{j=1}^m f(\vx; d_{i,j}).
\end{equation}
Here, $D$ denotes the entire dataset distributed across all $n$ clients, where each client $i$ has a local dataset $D_i = \{d_{i,j}\}_{j=1}^m$ of equal size $m$,\footnote{This is without loss of generality, since otherwise one can simply adjust the weights of the loss function.} 
$\vx \in \mathbb{R}^d$ denotes the model parameters, $f(\vx; D)$, $f(\vx; D_i)$, and $f(\vx; d_{i,j})$ denote the nonconvex loss function of the current model $\vx$ on the entire dataset $D$, the local dataset $D_i$, and a single data sample $d_{i,j}$, respectively.
For simplicity, we use $f(\vx)$, $f_i(\vx)$ and $f_{i,j}(\vx)$ to denote $f(\vx; D)$, $f(\vx; D_i)$ and $f(\vx; d_{i,j})$, respectively.

\subsection{Motivation: privacy-utility-communication trade-offs}

To unleash the full potential of FL, it is extremely important that the algorithm designed to solve \eqref{eq:prob} needs to meet several competing desiderata.

\paragraph{Communication efficiency.}
Communication between the server and clients is well recognized as the main bottleneck for optimizing the latency of FL systems, especially when the clients---such as mobile devices---have limited bandwidth, the number of clients is large, and/or the machine learning model has a lot of parameters---for example, the language model GPT-3 \citep{brown2020language} has billions of parameters and therefore cumbersome to share directly.

Therefore, it is very important to design FL algorithms to reduce the overall communication cost, which takes into account both \emph{the number of communication rounds} and \emph{the cost per communication round} for reaching a desired accuracy. With these two quantities in mind, there are two principal approaches for communication-efficient FL:
1) \emph{local methods}, where in each communication round, clients run multiple local update steps before communicating with the server, in the hope of reducing the number of communication rounds, e.g.,~\citep{mcmahan2017communication, li2020federatedopt, khaled2020tighter, gorbunov2020local, karimireddy2020scaffold, wang2020tackling, pathak2020fedsplit, cen2020convergence, li2022destress, li2020communication, acar2021federated, zhao2021fedpage, mitra2021linear, mishchenko2022proxskip};
2) \emph{compression methods}, where clients send compressed communication message to the server, in the hope of reducing the cost per communication round, e.g.,~\citep{alistarh2017qsgd, khirirat2018distributed,wang2018atomo,ivkin2019communication,karimireddy2019error,DIANA,reisizadeh2020fedpaq,li2020acceleration,gorbunov2021marina,li2021canita,richtarik2021ef21,fatkhullin2021ef21,lee2021finite,zhao2021faster, zhao2022beer,richtarik20223pc}. While both categories have garnered significant attention in recent years, we focus on the second approach based on communication compression to enhance communication efficiency.

\paragraph{Privacy preserving.}
While FL holds great promise of harnessing the inferential power of private data stored on a large number of distributed clients, these local data at clients often contain sensitive or proprietary information without consent to share. Although FL may appear to protect the data privacy via storing data locally and only sharing the model updates (e.g., gradient information), the training process can nonetheless reveal sensitive information as demonstrated by, e.g., \citet{zhu2019deep}. It is thus desirable for FL to preserve privacy in a guaranteed manner~\citep{geyer2017differentially, kairouz2019advances, sabater2020distributed, wang2021field}.

To ensure the training process does not accidentally leak private information, advanced privacy-preserving tools such as \emph{differential privacy} (DP) \citep{dwork2006calibrating} have been widely integrated into training algorithms~\citep{dwork2008differential, chaudhuri2011differentially, dwork2014algorithmic, abadi2016deep, wang2017differentially, iyengar2019towards, cheu2019distributed, feldman2020private}.
A notable example is \citet{abadi2016deep}, which developed a differentially-private stochastic gradient descent (SGD) algorithm  \DPSGD in the centralized (single-node) setting. More recently, several differentially-private algorithms~\citep{jayaraman2018distributed,wang2019efficient,shang2021differentially,lowy2022private} are proposed for the more general distributed ($n$-node) setting suitable for FL. In this paper, we also follow the DP approach to preserve privacy. In particular, we adopt local differential privacy (LDP) to respect the privacy of each client, which is critical in FL.

\begin{table}[!t]
	\caption{Comparisons among (local) differentially-private algorithms for the nonconvex problem \eqref{eq:prob} in both central (single-node) and distributed ($n$-node) settings. Here, $m$ denotes the number of data stored on a single client, $n$ is the number of clients, $d$ is the dimension, and $\omega$ is the parameter for the compression operator (cf.~Definition~\ref{def:comp}). The communication complexity is computed by $ndT/(1+\omega)$, where $T$ is the total number of communication rounds, and $nd/(1+\omega)$ is the communication cost per round. The utility~/~accuracy measures the average squared gradient norm of the objective function after $T$ rounds. Note that the algorithm is better when the utility/accuracy and the communication complexity are small under the same privacy guarantee.}
	\label{tab:results}
	\vspace{1mm}
	\centering
	\scriptsize
	\begin{threeparttable}
		\renewcommand{\arraystretch}{2}
		\begin{tabular}{|c|c|c|c|c|}
			\hline
			\bf Algorithm & \bf Privacy & \bf Utility/Accuracy & \bf \makecell{Communication Complexity} & \bf Remark\\
			\hline

			\makecell{{\sf RPPSGD} \citep{zhang2017efficient}}  
			& $(\epsilon,\delta)$-DP 
			& $\frac{\sqrt{d\log(m/\delta)\log(1/\delta)}}{m\epsilon}$  
			& --- & single node\\
			\hline
			
			\makecell{{\sf DP-GD/SGD}  \citep{abadi2016deep,wang2017differentially}}  
			& $(\epsilon,\delta)$-DP 
			& $\frac{\sqrt{d\log(1/\delta)}}{m\epsilon}$  
			& --- & single node\\
			\hline
			
			\makecell{{\sf DP-SRM} \citep{wang2019efficient} }  
			& $(\epsilon,\delta)$-DP 
			& $\frac{\sqrt{d\log(1/\delta)}}{m\epsilon}$ 
			& --- & single node\\
			\hline
			\hline
			\gape{\makecell{{\sf Distributed} \\ {\sf DP-SRM} \citep{wang2019efficient}}} \tnote{(1)} 
			& $(\epsilon,\delta)$-DP 
			& $\frac{\sqrt{d\log(1/\delta)}}{\blue{n}m\epsilon}$ 
			& $\frac{\blue{n^{2}} m\epsilon \sqrt{d}}{\sqrt{\log(1/\delta)}}$ 
			& \makecell{\blue{$n$} nodes, \\ no compression}\\
			\hline
			
			\makecell{{\sf LDP SVRG} \\ {\sf LDP SPIDER} \citep{lowy2022private}}  
			& $(\epsilon,\delta)$-LDP
			& $\frac{\sqrt{d\log(1/\delta)}}{\sqrt{\blue{n}}m\epsilon}$ 
			& $\frac{\blue{n^{3/2}} m\epsilon\sqrt{d}}{\sqrt{\log(1/\delta)}}$ 
			& \makecell{\blue{$n$} nodes, \\ no compression}\\
			\hline

			\gape{\makecell{{\sf Q-DPSGD-1}  \citep{ding2021differentially}}} \tnote{(2)} 
			& $(\epsilon,\delta)$-LDP 
			& $\frac{(\purple{\tilde{\sigma}^2}/\blue{n} +1/m)^{2/3}(d \log(1/\delta))^{1/3} }{m^{2/3}\epsilon^{2/3}}$
			& $\frac{(1 +\blue{n}/(m\purple{\tilde{\sigma}^2})) m^2\epsilon^2}{d \log(1/\delta)}$ 
			& \makecell{\blue{$n$} nodes, \\ direct compression }\\
			\hline
			
			\gape{\makecell{{\sf SDM-DSGD}  \citep{zhang2020private}}} \tnote{(3)} 
			& $(\epsilon,\delta)$-LDP 
			& $\tilde{O}\left(\frac{\sqrt{d\log(1/\delta)}}{\sqrt{\blue{n}}m\epsilon}\right)$ 
			&  $\frac{\blue{n^{7/2}} m\epsilon\sqrt{d}}{\purple{(1+\omega)^{3/2}}\sqrt{\log(1/\delta)}} + \frac{\blue{n} m^2\epsilon^2}{\purple{(1+\omega)}\log(1/\delta)}$ 
			& \makecell{\blue{$n$} nodes, \\ direct compression}\\
			\hline

			\rowcolor{bgcolor}
			\gape{\makecell{{\sf CDP-SGD} \\ (Theorem~\ref{thm:cdp-sgd})}} 
			& $(\epsilon,\delta)$-LDP 
			& $\frac{\sqrt{\purple{(1+\omega)}d\log(1/\delta)}}{\sqrt{\blue{n}}m\epsilon}$ 
			& \gape{\makecell{$\frac{\blue{n^{3/2}} m\epsilon\sqrt{d}}{\purple{(1+\omega)^{3/2}}\sqrt{\log(1/\delta)}} +  \frac{\blue{n}m^2\epsilon^2}{(\purple{1+\omega})\log(1/\delta)}$}}
			&  \gape{\makecell{\blue{$n$} nodes, \\ direct compression}}\\
			\hline
			
			\rowcolor{bgcolor}
			\gape{\makecell{\soteriaflSGD\\ \soteriaflGD\\ 
					(Corollary~\ref{cor:sgd})}} \tnote{(4)} 
			& $(\epsilon,\delta)$-LDP 
			& $\frac{\sqrt{\purple{(1+\omega)}d\log(1/\delta)}}{\sqrt{\blue{n}}m\epsilon} (1+ \sqrt{\tau})$ 
			& $\frac{\blue{n^{3/2}} m\epsilon\sqrt{d}}{\purple{(1+\omega)^{3/2}}\sqrt{\log(1/\delta)}} (1 + \sqrt{\tau})$ 
			&  \gape{\makecell{\blue{$n$} nodes, \\ \red{shifted} compression}}\\
			\hline
			
			\rowcolor{bgcolor}
			\gape{\makecell{\soteriaflSVRG\\ \soteriaflSAGA\\ (Corollary~\ref{cor:svrg}, \ref{cor:saga})}} \tnote{(4)} 
			& $(\epsilon,\delta)$-LDP 
			& $\frac{\sqrt{\purple{(1+\omega)}d\log(1/\delta)}}{\sqrt{\blue{n}}m\epsilon}$ 
			& $\frac{\blue{n^{3/2}} m\epsilon\sqrt{d}}{\purple{(1+\omega)^{3/2}}\sqrt{\log(1/\delta)}} (1+ \tau)$ 
			&  \gape{\makecell{\blue{$n$} nodes, \\ \red{shifted} compression}}\\
			\hline
		\end{tabular}\vspace{0.5mm}
		\begin{tablenotes}\small
			\item[(1)] \citet{wang2019efficient} considered the ``global'' $(\epsilon,\delta)$-DP (which only protects the privacy for entire dataset $D$, i.e., the local dataset $D_i$ on node $i$ may leak to other nodes $j\neq i$) without communication compression. However, we consider the ``local'' $(\epsilon,\delta)$-LDP which can protect the local datasets $D_i$'s at the client level.
			\item[(2)] \citet{ding2021differentially} adopted a slightly different compression assumption $\E[\norm{\gC(\vx)-\vx}^2]\leq \tilde{\sigma}^2$, with $\tilde{\sigma}^2$ playing a similar role as $(1+\omega)$ in ours. 
			However, it obtains a worse accuracy $\frac{(\purple{\tilde{\sigma}^2}/\blue{n} +1/m)^{2/3}(d \log(1/\delta))^{1/3} }{m^{2/3}\epsilon^{2/3}} = \frac{\sqrt{(\purple{\tilde{\sigma}^2}/\blue{n} +1/m)d \log(1/\delta)} }{m\epsilon} \cdot \big(\frac{(\purple{\tilde{\sigma}^2}/\blue{n} +1/m) m^2\epsilon^2}{d \log(1/\delta)}\big)^{1/6} =  \frac{\sqrt{(\purple{\tilde{\sigma}^2}/\blue{n} +1/m)d \log(1/\delta)} }{m\epsilon} \cdot T^{1/6}$, a factor of $T^{1/6}$ worse than the utility of the other algorithms including ours, where $T=\frac{(\purple{\tilde{\sigma}^2}/\blue{n}  +1/m) m^2\epsilon^2}{d \log(1/\delta)}$ is the optimal choice to achieve the best accuracy for \algname{Q-DPSGD-1}.
			\item[(3)] \citet{zhang2020private} only considered random-$k$ sparsification, which is a special case of our general compression operator. Moreover, it requires $\purple{1+\omega} \ll \log T$, i.e., at least $k\gg \frac{d}{\log T}$ out of $d$ coordinates need to be communicated, and its utility hides logarithmic factors larger than $\purple{1+\omega}$. The communication complexity  $\blue{n^{7/2}}$ is due to their convergence condition $T>\blue{n^5}$.
			\item[(4)] Here, $\tau :=\frac{\purple{(1+\omega)^{3/2}}}{\blue{n^{1/2}}}$. 
			If $\blue{n}\geq \purple{(1+\omega)^3}$ (which is typical in FL), then 
			$\tau <1$, and we can drop the terms involving $\tau$ from \soteriafl.
		\end{tablenotes}
	\end{threeparttable}
%	\vspace{-0.1in}
\end{table}

\paragraph{Goal.}
Encouraged by recent advances in communication compression techniques, and the widespread success of differentially-private methods, a natural question is 
\begin{center}
\emph{Can we develop a unified framework for private federated learning with communication compression, and understand the trade-offs between privacy, utility, and communication?}
\end{center}
Note that there have been a handful of works that simultaneously address compression and privacy in FL. Unfortunately, they only provide partial answers to the above question. Most of the existing works only consider specific, elementary, or tailored compression schemes that are applied directly to the gradient messages in \DPSGD \citep{agarwal2018cpsgd,wang2020d2p,girgis2021shuffled,zong2021communication,zhang2020private,ding2021differentially}. 
A number of works \citep{suresh2017distributed,chen2020breaking,chen2021communication,kairouz2021distributed,feldman2021lossless,triastcyn2021dp} extended and considered different compression schemes, but did not provide concrete trade-offs in terms of privacy, utility and communication. Furthermore,
existing theoretical analyses can be limited only to convex problems \citep{girgis2021shuffled}, lacking in some aspects such as utility \citep{zong2021communication},  or delivering pessimistic guarantees on utility and/or communication due to strong assumptions \citep{zhang2020private,ding2021differentially}. Finally, existing work only studied the DP framework for direct compression, while it is known that the recently developed shifted compression scheme \citep{DIANA, DIANA2, li2020unified} achieves much better convergence guarantees. Due to noise injection for privacy-preserving, it is a priori unclear if the shifted compression scheme is also compatible with privacy.

\subsection{Our contributions}
In this paper, we answer the above question by providing a general approach that enhances the communication efficiency of private federated learning in the {\em nonconvex} setting, through a unified framework called \soteriafl (see Algorithm~\ref{alg:soteriafl}). Specifically, we have the following contributions.

\begin{enumerate}
    \item We first present a simple algorithm \cdpSGD (Algorithm~\ref{alg:cdp-sgd}) that directly combines communication compression and \DPSGD. We provide theoretical analysis for \cdpSGD in Theorem \ref{thm:cdp-sgd} and show its limitations in communication efficiency. 
    
    \item We then propose a general framework \soteriafl for private FL, 
    which accommodates a general family of local gradient estimators including popular stochastic variance-reduced gradient methods and the state-of-the-art shifted compression scheme.
     We provide a unified characterization of its performance trade-offs in terms of privacy, utility (convergence accuracy), and communication complexity.

    \item We apply our unified analysis for \soteriafl and obtain theoretical guarantees for several new private FL algorithms, including \soteriaflGD, \soteriaflSGD, \soteriaflSVRG, and \soteriaflSAGA. All of these algorithms are shown to perform better than the plain \cdpSGD (Algorithm~\ref{alg:cdp-sgd}), and have lower communication complexity compared with other private FL algorithms without compression. 
    The numerical experiments also corroborate the theory and confirm the practical superiority of \soteriafl.
\end{enumerate}

We provide detailed comparisons between the proposed approach and prior arts in Table~\ref{tab:results}. To the best of our knowledge, \soteriafl is the first unified framework that simultaneously enables local differential privacy and shifted compression, and allows flexible local computation protocols at the client level.

\section{Preliminaries}
\label{sec:pre}

Let $[n]$ denote the set $\{1,2,\cdots,n\}$ and $\norm{\cdot}$ denote the Euclidean norm of a vector.
Let $\dotp{\vu}{\vv}$ denote the standard Euclidean inner product of two vectors $\vu$ and $\vv$. 
Let $f^*:=\min_\vx f(\vx) >-\infty$ denote the optimal value of the objective function in~\eqref{eq:prob}.
In addition, we use the standard order notation $O(\cdot)$ to hide absolute constants. 
We now introduce the definitions of the compression operator and local differential privacy, as well as some standard assumptions for the objective functions.

\paragraph{Compression operator.}
\label{sec:compression}
Let us introduce the notion of a randomized {\em compression operator},  which is used to compress the gradients to save communication. 
The following definition of unbiased compressors is standard and has been used in many distributed/federated learning algorithms \citep{alistarh2017qsgd, khirirat2018distributed, DIANA, DIANA2, li2020unified, li2020acceleration, gorbunov2021marina, li2021canita}. 

\begin{definition}[Compression operator]\label{def:comp}
	A randomized map $\gC: \R^d\mapsto \R^d$ is an $\omega$-compression operator  if for all $\vx\in \R^d$, it satisfies
	\begin{equation}\label{eq:comp}
	\E[\gC(\vx)]=\vx, \qquad \E\left[\norm{\gC(\vx)-\vx}^2\right]\leq \omega\norm{\vx}^2.
	\end{equation}
	In particular, no compression ($\gC(\vx)\equiv \vx$) implies $\omega=0$.
\end{definition}

Note that the conditions \eqref{eq:comp} are satisfied by many practically useful compression operators, e.g., random sparsification and random quantization~\citep{alistarh2017qsgd, li2020acceleration, li2021canita}. A useful rule of thumb is that the communication cost is often reduced by a factor of $\frac{1}{1+\omega}$ due to compression \citep{alistarh2017qsgd}. Next, we briefly discuss an example called random sparsification to provide more intuition.

\vspace{2mm}
{\noindent\bf Example 1} (Random sparsification).\label{example:rs}
Given $\vx\in \R^d$, the random-$k$ sparsification operator is defined by $\calC(\vx):= \frac{d}{k} \cdot (\vxi_k \odot \vx)$, where 
$\odot$ denotes the Hadamard (element-wise) product and $\vxi_k\in \{0,1\}^d$ is a uniformly random binary vector with $k$ nonzero entries ($\n{\vxi_k}_0=k$).
This random-$k$ sparsification operator $\calC$ satisfies \eqref{eq:comp} with $\omega =\frac{d}{k}-1$, and the communication cost  is reduced by a factor of $\frac{1}{1+\omega}$ since we transmit $k = \frac{d}{1+\omega}$ (due to $\omega =\frac{d}{k}-1$) coordinates rather than $d$ coordinates of the message.

\paragraph{Local differential privacy.}
\label{sec:privacy}
We not only want to train the machine learning model using fewer communication bits, but also want to maintain each client's local privacy, which is a key component for FL applications. 
Following the framework of (local) differential privacy~\citep{andres2013geo,chatzikokolakis2013broadening,zhao2020local}, we say that two datasets $D$ and $D'$ are neighbors if they differ by only one entry. We have the following definition for local differential privacy (LDP).

\begin{definition}[Local differential privacy (LDP)]\label{def:dp}
	A randomized mechanism $\gM:\gD\to\gR$ with domain $\gD$ and range $\gR$ is $(\epsilon, \delta)$-locally differentially private for client $i$ if for all neighboring datasets $D_i, D'_i\in \gD$ on client $i$ and for all events $S\in \gR$ in the output space of $\gM$, we have
	\begin{equation*}
	\Pr\{\gM(D_i)\in S\} \le e^{\epsilon}\Pr\{\gM(D'_i)\in S\} + \delta.
	\end{equation*}
\end{definition}

The definition of LDP (Definition~\ref{def:dp}) is very similar to the original definition of $(\epsilon,\delta)$-DP \citep{dwork2006calibrating, dwork2014algorithmic}, except that now in the FL setting, each client protects its own privacy by encoding and processing its sensitive data locally, and then transmitting the encoded information to the server without coordination and information sharing between the clients.

\paragraph{Assumptions about the functions.}\label{sec:ass}

Recalling \eqref{eq:prob}, we consider the \emph{nonconvex} FL setting, where the functions $\{f_{i,j}\}$ are arbitrary functions satisfying the following standard smoothness assumption (Assumption \ref{ass:smoothness}) and bounded gradient assumption (Assumption \ref{ass:bounded-gradient}).

\begin{assumption}[Smoothness]\label{ass:smoothness}There exists some $L \ge 0$, such that for all $i\in [n], j\in [m]$,
	the function $f_{i,j}$ is $L$-smooth, i.e., 
	\[\norm{ \nabla f_{i,j}(\vx_1) - \nabla f_{i,j}(\vx_2)} \le L\norm{\vx_1 - \vx_2}, \qquad \forall \vx_1, \vx_2\in \R^d.\]
\end{assumption}

\begin{assumption}[Bounded gradient]\label{ass:bounded-gradient}
There exists some $G \ge 0$, such that for all $i\in [n], j\in [m]$ and $\vx\in\R^d$, we have $\|\nabla f_{i,j}(\vx)\| \le G$.
\end{assumption}

The smoothness assumption is very standard for the convergence analysis \citep{nesterov2004introductory, ghadimi2013stochastic, li2021page, li2022simple}, and the bounded gradient assumption is also standard for the differential privacy analysis \citep{bassily2014private, wang2017differentially, iyengar2019towards, feldman2020private}.

\section{Warm-up: Plain Compressed Differentially-Private SGD}\label{sec:cdp-sgd}

There are two methods to combine privacy and compression: (1) first perturb and then compress, and (2) first compress and then perturb. The advantage of the first method is that it is very simple and general, since compression will preserve the differential privacy and work seamlessly with any existing privacy mechanisms. 
However, the second method requires carefully designed perturbation mechanisms (otherwise the perturbation might diminish the communication saving of compression), e.g., binomial perturbation~\citep{agarwal2018cpsgd} or discrete Gaussian perturbation~\citep{kairouz2021distributed}. In addition, it is observed that the first method achieves better utility compared with the second one in some settings~\citep{ding2021differentially}. Thus, we also apply the first method in this paper: first perturb then compress. 

\paragraph{Baseline algorithm: \cdpSGD.} As a warm-up, we first introduce a simple algorithm \cdpSGD (described in Algorithm~\ref{alg:cdp-sgd}), which subsumes some existing algorithms as special cases (e.g., \citep{zhang2020private,zong2021communication}) for private FL with better theoretical guarantees. 
The procedure for \cdpSGD is very simple: at round $t$, each client $i$ first computes a local stochastic gradient $\tilde{\vg}_i^t$ using its local dataset $D_i$ (Line~\ref{line:sgd} in Algorithm~\ref{alg:cdp-sgd}). Then, it uses Gaussian mechanism~\citep{abadi2016deep} to achieve LDP (Line~\ref{line:perturb} in Algorithm~\ref{alg:cdp-sgd}) and communicates the compressed perturbed private gradient information to the server (Line~\ref{line:compression} in Algorithm~\ref{alg:cdp-sgd}). Finally, the server aggregates the compressed information and update the model parameters (Line~\ref{line:aggregate}--\ref{line:update} in Algorithm~\ref{alg:cdp-sgd}). 

\begin{algorithm}[ht]
	\caption{~Compressed Differentially-Private Stochastic Gradient Descent (\cdpSGD)}
	\label{alg:cdp-sgd}
	\begin{algorithmic}[1]
	    \REQUIRE ~%Input
		initial point $x^0$, stepsize $\eta_t$, variance $\sigma_p^2$, minibatch size $b$
		\FOR{$t=0,1,2,\dots, T$}
		\STATE {\bf{for each node $i\in [n]$ do in parallel}}
		\STATE \quad~ Sample a random minibatch $\gI_b$ from local dataset $D_i$
		\STATE  \quad~ Compute local stochastic gradient $\tilde{\vg}_i^t =\frac{1}{b} \sum_{j\in \gI_b} \nabla f_{i,j}(\vx^t)$ {\size{8.5}{\quad//~{all nodes use SGD method}}} \label{line:sgd}
		\STATE  \quad~ \emph{Privacy}: $\vg_i^t = \tilde{\vg}_i^t  + \vxi_i^{t}$, where $\vxi^i_{t}\sim\gN(\bm{0},\sigma_p^2\mI)$ \label{line:perturb}
		\STATE  \quad~ \emph{Compression}: let $\vv_i^{t} = \gC_i^t(\vg_i^t)$ and send to the server {\size{8.5}{\qquad//~{direct compression}}} \label{line:compression}
		\STATE \textbf{end each node}
		\STATE Server aggregates compressed information $\vv^t = \frac{1}{n}\sum_{i=1}^n \vv_i^{t}$ \label{line:aggregate}
		\STATE $\vx^{t+1} = \vx^t - \eta_t\vv^t$ \label{line:update}
		\ENDFOR
	\end{algorithmic}
\end{algorithm}

\paragraph{Theoretical guarantee.} We present the theoretical guarantees for \cdpSGD in the following theorem.

\begin{restatable}[Privacy, utility and communication for \CDP]{theorem}{thmnonconvex}\label{thm:cdp-sgd}
     Suppose that Assumptions \ref{ass:smoothness} and \ref{ass:bounded-gradient} hold, and the compression operators $\gC_i^t$ (cf. Line~\ref{line:compression} of Algorithm~\ref{alg:cdp-sgd}) are drawn independently satisfying Definition~\ref{def:comp}.  
     By choosing the algorithm parameters properly and
     letting the total number of communication rounds 
    $$T = O\left(\frac{\sqrt{nL} m\epsilon}{G\sqrt{(1+\omega)d \log(1/\delta)}} +  \frac{m^2\epsilon^2}{d\log(1/\delta)}\right),$$
     \CDP (Algorithm~\ref{alg:cdp-sgd}) satisfies $(\epsilon,\delta)$-LDP and the  utility
     $$\frac{1}{T}\sum_{t=0}^{T-1} \E\|\nabla f(\vx_t)\|^2 \le O\left(\frac{G\sqrt{(1+\omega)Ld\log(1/\delta)}}{\sqrt{n}m\epsilon} \right) .$$
\end{restatable}

The proposed \cdpSGD (Algorithm~\ref{alg:cdp-sgd}) is simple but effective. When the compression parameter $\omega$ is a constant (i.e., constant compression ratio), \cdpSGD achieves the same utility $O\Big(\frac{\sqrt{d\log(1/\delta)}}{m\epsilon}\Big)$ as \DPSGD in the single-node case with $n=1$. In comparison, our utility is better than \cite{ding2021differentially} by a factor of $T^{1/6}$, and our communication complexity is much better than \cite{zhang2020private} (see Table~\ref{tab:results}).

However, the communication complexity of \cdpSGD still has room for improvements due to \emph{direct compression} (Line~\ref{line:compression} in Algorithm~\ref{alg:cdp-sgd}). In particular, if the size of the local dataset $m$ stored on clients is dominating, then \cdpSGD (even if we compute local full gradients as \cdpGD) requires $O(m^2)$ communication rounds (see Theorem~\ref{thm:cdp-sgd}), while previous distributed differentially-private algorithms without communication compression (e.g., \algname{Distributed DP-SRM}~\citep{wang2019efficient}, \algname{LDP SVRG} and \algname{LDP SPIDER}~\citep{lowy2022private}) only need $O(m)$ communication rounds (see Table~\ref{tab:results}). 
 
\section{\soteriafl: A Unified Private FL Framework with Shifted Compression}
\label{sec:unified}

Due to the limitations of plain \cdpSGD, we now present an advanced and unified private FL framework called \soteriafl in this section, which allows a large family of local gradient estimators (Line~\ref{line:gradient-soteriafl} in Algorithm~\ref{alg:soteriafl} and Line~\ref{line:option-sgd}--\ref{line:option-end} in Algorithm~\ref{alg:soteriafl-detail}). Via adopting the advanced \emph{shifted compression} (Line~\ref{line:compression-soteriafl} in Algorithm~\ref{alg:soteriafl}), \soteriafl
 reduces the total number of communication rounds $O(m^2)$ of \cdpSGD to $O(m)$, which matches previous uncompressed DP algorithms (see Table~\ref{tab:results}), and further reduces the total communication complexity due to less communication cost per round.

\subsection{A unified \soteriafl framework}
\label{sec:soteriafl}

Our \soteriafl framework is described in Algorithm~\ref{alg:soteriafl}. 
At round $t$, each client will compute a local (stochastic) gradient estimator $\tilde{\vg}_i^t$ using its local dataset $D_i$ (Line~\ref{line:gradient-soteriafl} in Algorithm~\ref{alg:soteriafl}). 
One can choose several optimization methods for computing this local gradient estimator such as standard gradient descent (\algname{GD}), stochastic GD (\algname{SGD}), stochastic variance reduced gradient (\algname{SVRG})~\citep{johnson2013accelerating, kovalev2020don}, and \algname{SAGA}~\citep{defazio2014saga} (see e.g., Line~\ref{line:option-sgd}--\ref{line:option-end} in Algorithm~\ref{alg:soteriafl-detail}). 
Then, each client adds a Gaussian perturbation $\vxi_i^{t}$ on its gradient estimate $\tilde{\vg}_i^t$ to ensure LDP (Line~\ref{line:perturb-soteriafl} in Algorithm~\ref{alg:soteriafl}). 
However, different from \cdpSGD (Algorithm~\ref{alg:cdp-sgd}) where we directly compress the perturbed stochastic gradients, now each client maintains a reference $\vs_i^t$ and compresses the shifted message $\tilde{\vg}_i^t -\vs_i^t$ (Line~\ref{line:compression-soteriafl} in Algorithm~\ref{alg:soteriafl}). 
This extra shift operation achieves much better convergence behavior (fewer communication rounds) than \cdpSGD, and thus allowing much lower communication complexity.

\begin{algorithm}[ht]
	\caption{~\soteriafl (a unified framework for compressed private FL)}
	\label{alg:soteriafl}
	\begin{algorithmic}[1]
	    \REQUIRE ~%Input
		initial point $\vx^0$, stepsize $\eta_t$, shift stepsize $\gamma_t$, variance $\sigma_p^2$,  initial reference $\vs_i^0=0$
		\FOR{$t=0,1,2,\dots, T$}
		\STATE {\bf{for each node $i\in [n]$ do in parallel}}
		\STATE  \quad~ Compute local gradient estimator $\tilde{\vg}_i^t$     {\size{8.5}{\quad~//~{it allows many methods, e.g., SGD, SVRG, and SAGA}}} \label{line:gradient-soteriafl} 
		\STATE  \quad~ \emph{Privacy}: $\vg_i^t = \tilde{\vg}_i^t  + \vxi_i^{t}$, where $\vxi_i^{t} \sim \gN(\bm{0},\sigma_p^2\mI)$ \label{line:perturb-soteriafl}
		\STATE  \quad~ \emph{Compression}: let $\vv_i^{t} = \gC_i^t(\vg_i^t - \vs_i^t)$ and send to the server {\size{8.5}{\qquad//~{shifted compression}}} \label{line:compression-soteriafl}
		\STATE  \quad~ Update shift $\vs_i^{t+1} = \vs_i^t  + \gamma_t\gC_i^t(\vg_i^t - \vs_i^t)$ \label{line:shift-soteriafl} 
		\STATE \textbf{end each node}
		\STATE Server aggregates compressed information $\vv^t = \vs^t + \frac{1}{n}\sum_{i=1}^n \vv_i^{t}$
		\STATE $\vx^{t+1} = \vx^t - \eta_t\vv^t$ \label{line:update-soteriafl}
		\STATE $\vs^{t+1} = \vs^t + \gamma_t\frac{1}{n}\sum_{i=1}^n \vv_i^{t}$
		\ENDFOR
	\end{algorithmic}
\end{algorithm}

\subsection{Generic assumption and unified theory}
\label{sec:generic-ass}
We provide a generic Assumption~\ref{ass:unified}, which is very flexible to capture the behavior of several existing (and potentially new) gradient estimators, while simultaneously maintaining the tractability to enable a unified and sharp theoretical analysis.

\begin{assumption}[Generic assumption of local gradient estimator for \soteriafl]
\label{ass:unified}
The gradient estimator $\tilde{\vg}_i^t$ (Line~\ref{line:gradient-soteriafl} of Algorithm~\ref{alg:soteriafl}) is unbiased $\E_t[\tilde{\vg}_i^t] = \nabla f_i(\vx^t)$ for $i\in [n]$, where $\E_t$ takes the expectation conditioned on all history before round $t$. Moreover, it can be decomposed into two terms $\tilde{\vg}_i^t := \gA_i^t + \gB_i^t$ and there exist constants $G_A, G_B, C_1, C_2, C_3, C_4, \theta$ and a random sequence $\{\Delta^t\}$ such that
\begin{subequations}
\begin{align}
    &\gA_i^t= \frac{1}{b} \sum_{j\in \gI_b} \varphi_{i,j}^t, 
    \qquad \gB_i^t= \frac{1}{m} \sum_{j=1}^m \psi_{i,j}^t, \label{ass:decompose} \\
    &\E_t\Big[\frac{1}{n}\sum_{i=1}^n \ns{\tilde{\vg}_i^t - \nabla f_i(\vx^t)}\Big] \leq C_1 \Delta^t + C_2, \label{ass:var-g}\\
    &\E_t\left[\Delta^{t+1}\right] \leq (1-\theta) \Delta^t + C_3\ns{\nabla f(\vx^t)} + C_4\E_t\ns{\vx^{t+1}-\vx^t}, \label{ass:var-Delta}
\end{align}
\end{subequations}
where $\varphi_{i,j}^t$ and $\psi_{i,j}^t$ are bounded by $G_A$ and $G_B$ respectively, and $\gI_b$ usually denotes a random minibatch with size $b$. Here, $\varphi_{i,j}^t$ and $\psi_{i,j}^t$ should be viewed as functions related to the $j$-th sample $d_{i,j}$ stored on client $i$.
\end{assumption}

A few comments are in order. Concretely, the decomposition \eqref{ass:decompose} is used for our unified privacy analysis (i.e., Theorem~\ref{thm:privacy}). We can let one of them be $\bf{0}$ if the gradient estimator only contains one term or is not decomposable.
The parameters $C_1$ and $C_2$ in \eqref{ass:var-g} capture the variance of the gradient estimators, e.g., $C_1=C_2=0$ if the client computes local full gradient $\tilde{\vg}_i^t = \nabla f_i(\vx^t)$, and $C_1\neq 0$ (note that $\Delta^t$ will shrink in \eqref{ass:var-Delta}) and $C_2=0$ if the client uses variance-reduced gradient estimators such as 
\algname{SVRG}/\algname{SAGA}.
Finally, the parameters $\theta, C_3$ and $C_4$ in \eqref{ass:var-Delta} capture the shrinking behavior of the variance (incurred by the gradient estimators), where different variance-reduced gradient methods usually have different shrinking behaviors. More concrete examples to follow in Lemma~\ref{lem:para-sgd-svrg-saga} in Section~\ref{sec:algorithms}.

\paragraph{Unified theory for privacy-utility-communication trade-offs.}
Given our generic Assumption \ref{ass:unified}, we can obtain a unified analysis for \soteriafl framework. The following Theorem \ref{thm:privacy} unifies the privacy analysis and Theorem \ref{thm:utility} unifies the utility and communication complexity analysis.

\begin{theorem}[Privacy for \soteriafl]\label{thm:privacy}
    Suppose that Assumption \ref{ass:unified} holds. There exist constants $c$ and $c'$, for any $\epsilon < c'b^2T/m^2$ and $\delta \in (0,1)$, \soteriafl (Algorithm~\ref{alg:soteriafl}) is $(\epsilon, \delta)$-LDP if we choose
    \begin{align}\label{eq:para-sigma-p}
        \sigma_p^2 = c\frac{(G_A^2/4 + G_B^2)T\log({1}/{\delta})}{m^2\epsilon^2}.
    \end{align}
\end{theorem}

\begin{theorem}[Utility and communication for \soteriafl]\label{thm:utility}
    Suppose that Assumptions \ref{ass:smoothness} and \ref{ass:unified} hold, and the compression operators $\gC_i^t$ (cf.~Line~\ref{line:compression-soteriafl} of Algorithm~\ref{alg:soteriafl}) are drawn independently satisfying Definition~\ref{def:comp}.
     Set the stepsize as
	\begin{align*}
		\eta_t \equiv \eta \leq \min\left\{ \frac{1}{(1+2\alpha C_4 + 4\beta(1+\omega) + 2\alpha C_3/\eta^2)L}, 
		\frac{\sqrt{\beta n}}{\sqrt{1+2\alpha C_4 + 4\beta(1+\omega)}(1+\omega)L}\right\},
	\end{align*}
	where $\alpha = \frac{3\beta C_1}{2(1+\omega)\theta L^2}$,  $\forall \beta>0$, the shift stepsize as $\gamma_t\equiv \sqrt{\frac{1+2\omega}{2(1+\omega)^3}}$, and the privacy variance $\sigma_p^2$ according to Theorem~\ref{thm:privacy}. Then, \soteriafl (Algorithm~\ref{alg:soteriafl}) satisfies $(\epsilon,\delta)$-LDP and the following  
	\begin{align*}
		\frac{1}{T}\sum_{t=0}^{T-1} \E\ns{\nabla f(\vx^t)} 
		&\le \frac{2\Phi_0}{\eta T} 
		+ \frac{3\beta}{(1+\omega)L\eta}\left(C_2 + \frac{c(G_A^2/4 +G_B^2)dT\log ({1}/{\delta})}{m^2\epsilon^2}\right),
	\end{align*}
	where $\Phi_0:= f(\vx^0)-f^* + \alpha L \Delta^0 + \frac{\beta}{Ln}\sum_{i=1}^n\ns{\nabla f_i(\vx^0) - \vs_i^0}$. 
	By further choosing the total number of communication rounds $T$ as 
	\begin{align} \label{eq:choice-T}
		T &=  \max\left\{\frac{m\epsilon\sqrt{2(1+\omega)L\Phi_0}}{\sqrt{3\beta cd(G_A^2/4 +G_B^2)\log({1}/{\delta})}}, \frac{C_2m^2\epsilon^2}{c d(G_A^2/4 +G_B^2)\log({1}/{\delta})}\right\}, 
	\end{align}
	\soteriafl has the following utility (accuracy) guarantee: 
	\begin{align} \label{eq:utility}
		\frac{1}{T}\sum_{t=0}^{T-1} \E\ns{\nabla f(\vx^t)} 
		&\le O\left( \max\left\{\frac{\sqrt{\beta d(G_A^2/4 +G_B^2)\log({1}/{\delta})}}{\eta m\epsilon\sqrt{(1+\omega)L}}, \frac{\beta C_2}{(1+\omega)L\eta}\right\}\right). 
	\end{align}
\end{theorem}

Theorem~\ref{thm:utility} is a unified theorem for our \soteriafl framework, which covers a large family of local stochastic gradient methods under the generic Assumption~\ref{ass:unified}.
In the next Section~\ref{sec:algorithms}, we will show that many popular local gradient estimators (\algname{GD}, \algname{SGD}, \algname{SVRG}, and \algname{SAGA}) satisfy  Assumption~\ref{ass:unified}, and thus can be captured by our unified analysis.

\section{Some Algorithms within \soteriafl Framework}
\label{sec:algorithms}

In this section, we propose several new private FL algorithms (\soteriaflGD, \soteriaflSGD, \soteriaflSVRG and \soteriaflSAGA) captured by our \soteriafl framework. 
We give a detailed Algorithm~\ref{alg:soteriafl-detail} which describes all these four \soteriafl-type algorithms in a nutshell.

\begin{algorithm}[h]
	\caption{\soteriaflSGD, \soteriaflSVRG, and \soteriaflSAGA}
	\label{alg:soteriafl-detail}
	\begin{algorithmic}[1]
	    \REQUIRE ~%Input
		initial point $\vx^0$, stepsize $\eta_t$, shift stepsize $\gamma_t$, variance $\sigma_p^2$, minibatch size $b$, initial reference $\vs_i^0=0$, initial $\vw^0=\vx^0$ for SVRG or $\vw^0_{i,j}=\vx^0$ for SAGA, probability $p$
	
		\FOR{$t=0,1,2,\dots, T$}
		\STATE {\bf{for each node $i\in [n]$ do in parallel}}
		
		\STATE \quad~{\blue{Option \RomanNumeralCaps{1}: \sf SGD}} \label{line:option-sgd}
		\STATE  \quad~ \quad~ Compute local SGD estimator $\tilde{\vg}_i^t =\frac{1}{b} \sum_{j\in \gI_b} \nabla f_{i,j}(\vx^t)$ {\size{9}{\quad~//~\blue{\sf GD} if choose $b=m$}} \label{line:gradient-sgd}
		
		\STATE \quad~{\blue{Option \RomanNumeralCaps{2}: \sf SVRG}} \label{line:option-svrg}
		\STATE  \quad~ \quad~ Compute local SVRG estimator $\tilde{\vg}_i^t =\frac{1}{b} \sum_{j\in \gI_b} (\nabla f_{i,j}(\vx^t)- \nabla f_{i,j}(\vw^t)) +\nabla f_i(\vw^t)$   \label{line:gradient-svrg} 
		\STATE  \quad~ \quad~ Update SVRG snapshot point $\vw^{t+1} = \begin{cases}
		\vx^t, &\text{with probability } p\\
		\vw^t, &\text{with probability } 1-p
		\end{cases}$ \label{line:w_prob-svrg}
		
		\STATE \quad~{\blue{Option \RomanNumeralCaps{3}: \sf SAGA}} \label{line:option-saga}
		\STATE \quad~ \quad~ Compute local SAGA estimator:  \\
		    \qquad \qquad \qquad $\tilde{\vg}_i^t =\frac{1}{b} \sum_{j\in \gI_b} (\nabla f_{i,j}(\vx^t)- \nabla f_{i,j}(\vw_{i,j}^t)) + \frac{1}{m}\sum_{j=1}^m\nabla f_{i,j}(\vw_{i,j}^t)$ \label{line:gradient-saga}
		\STATE \quad~ \quad~ Update SAGA variables $\vw_{i,j}^{t+1} = \begin{cases}
		\vx^t, &\text{for } j\in \gI_b\\
		\vw_{i,j}^t, &\text{for } j\notin \gI_b
		\end{cases}$  \label{line:w_prob-saga}
		\STATE \quad~{\blue{End Options}} \label{line:option-end}
		
		\STATE  \quad~ \emph{Privacy}: $\vg_i^t = \tilde{\vg}_i^t  + \vxi_i^{t}$, where $\vxi_i^{t} \sim \gN(\bm{0},\sigma_p^2\mI)$ \label{line:perturb-svrg}
		\STATE  \quad~ \emph{Compression}: let $\vv_i^{t} = \gC_i^t(\vg_i^t - \vs_i^t)$ and send to the server \label{line:compression-svrg}
		\STATE  \quad~ Update shift $\vs_i^{t+1} = \vs_i^t  + \gamma_t\gC_i^t(\vg_i^t - \vs_i^t)$
		\STATE \textbf{end each node}
		\STATE Server aggregates compressed information $\vv^t = \vs^t + \frac{1}{n}\sum_{i=1}^n \vv_i^{t}$
		\STATE $\vx^{t+1} = \vx^t - \eta_t\vv^t$
		\STATE $\vs^{t+1} = \vs^t + \gamma_t\frac{1}{n}\sum_{i=1}^n \vv_i^{t}$
		\ENDFOR
	\end{algorithmic}
\end{algorithm}

To analyze Algorithm~\ref{alg:soteriafl-detail} using our unified \soteriafl framework,
we begin by showing that these local gradient estimators (\algname{GD}, \algname{SGD}, \algname{SVRG}, and \algname{SAGA}) satisfy Assumption \ref{ass:unified} in the following main lemma, detailing the corresponding parameter values (i.e., $G_A, G_B, C_1, C_2, C_3, C_4$, and $\theta$).

\begin{lemma}[SGD/SVRG/SAGA estimators satisfy Assumption~\ref{ass:unified}]
\label{lem:para-sgd-svrg-saga}
    Suppose that Assumptions~\ref{ass:smoothness} and \ref{ass:bounded-gradient} hold. 
    The local SGD estimator $\tilde{\vg}_i^t$ (Option \RomanNumeralCaps{1} in Algorithm~\ref{alg:soteriafl-detail}) satisfies Assumption~\ref{ass:unified} with 
    \begin{align*}
        G_A = G,~
        G_B = C_1 = C_3 = C_4 = 0,~ 
        C_2 = \frac{(m-b)G^2}{mb},~ 
        \theta = 1,~ 
        \Delta^t \equiv 0.
    \end{align*}
    The local SVRG estimator $\tilde{\vg}_i^t$ (Option \RomanNumeralCaps{2} in Algorithm~\ref{alg:soteriafl-detail}) satisfies Assumption~\ref{ass:unified} with 
    \begin{align*}
        G_A = 2G,~ 
        G_B = G,~
        C_1 = \frac{L^2}{b},~
        C_2 = 0,~ 
        C_3 = \frac{2(1-p)\eta^2}{p},~ 
        C_4 = 1,~
        \theta = \frac{p}{2},~ 
        \Delta^t = \ns{\vx^t-\vw^t}.
    \end{align*}
    The local SAGA estimator $\tilde{\vg}_i^t$ (Option \RomanNumeralCaps{3} in Algorithm~\ref{alg:soteriafl-detail}) satisfies Assumption~\ref{ass:unified} with 
    \begin{equation*}
        \begin{split}
            &G_A = 2G,~ 
            G_B = G,~
            C_1 = \frac{L^2}{b},~
            C_2 = 0,~ 
            C_3 = \frac{2(m-b)\eta^2}{b},~ 
            C_4 = 1,~ \\
            &\theta = \frac{b}{2m},~ 
            \Delta^t = \frac{1}{nm}\sum_{i=1}^n\sum_{j=1}^m\ns{\vx^t-\vw_{i,j}^t}.
        \end{split}
    \end{equation*}
\end{lemma}

With Lemma~\ref{lem:para-sgd-svrg-saga} in hand, we can plug their corresponding parameters into the unified Theorem~\ref{thm:utility} to obtain detailed utility and communication bounds for the resulting methods (\soteriaflSGD/\soteriaflGD, \soteriaflSVRG, and \soteriaflSAGA). Formally, we have the following three corollaries.

\begin{corollary}[\soteriaflSGD/\soteriaflGD]\label{cor:sgd}
    Suppose that Assumptions~\ref{ass:smoothness} and \ref{ass:bounded-gradient} hold and we combine Theorem~\ref{thm:utility} and Lemma~\ref{lem:para-sgd-svrg-saga}, i.e., choosing stepsize
    $
    \eta_t \equiv \eta \leq  \frac{1}{(1+2\sqrt{(1+\omega)^3/n})L},
    $
    where we set $\beta=\frac{\tau}{2(1+\omega)}$ and  $\tau:=\frac{(1+\omega)^{3/2}}{n^{1/2}}$, 
    shift stepsize $\gamma_t\equiv \sqrt{\frac{1+2\omega}{2(1+\omega)^3}}$, and privacy variance $\sigma_p^2 = O\big(\frac{G^2T\log({1}/{\delta})}{m^2\epsilon^2}\big)$.
    If we further set the minibatch size $b= \min\Big\{ \frac{m \epsilon G \sqrt{\beta}}{\sqrt{(1+\omega)Ld\log(1/\delta)}}, m\Big\}$ and the total number of communication rounds 
    $
        T = O\Big( \frac{\sqrt{nL} m\epsilon}{G\sqrt{(1+\omega)d \log(1/\delta)}}(1+\sqrt{\tau})\Big),
    $
    then \soteriaflSGD satisfies $(\epsilon,\delta)$-LDP and the following utility guarantee
    $
        \frac{1}{T}\sum_{t=0}^{T-1} \E\|\nabla f(\vx_t)\|^2 \le O\Big(\frac{G\sqrt{(1+\omega)Ld\log(1/\delta)}}{\sqrt{n}m\epsilon} (1+\sqrt{\tau})\Big).
    $
    If we choose a minibatch size $b=m$ (local full gradient) in \soteriaflSGD, the result of \soteriaflSGD leads to that of \soteriaflGD.            
\end{corollary}

\begin{corollary}[\soteriaflSVRG]\label{cor:svrg}
    Suppose that Assumptions~\ref{ass:smoothness} and \ref{ass:bounded-gradient} hold and we combine Theorem~\ref{thm:utility} and Lemma~\ref{lem:para-sgd-svrg-saga}, i.e., choosing stepsize
    $
        \eta_t \equiv \eta \leq  \frac{p^{2/3}b^{1/3}{\min\{1, \sqrt{n/(1+\omega)^3}\}}}{2L},
    $
    where we set $\beta = \frac{p^{4/3}b^{2/3}(1+\omega)^2\min\{1, n/(1+\omega)^3\}}{n}$, $p^{2/3}b^{1/3}\leq 1/4$ and $p\leq 1/4$,
    shift stepsize $\gamma_t\equiv \sqrt{\frac{1+2\omega}{2(1+\omega)^3}}$, and privacy variance $\sigma_p^2 = O\big(\frac{G^2T\log({1}/{\delta})}{m^2\epsilon^2}\big)$.
    If we further let the minibatch size $b=\frac{m^{2/3}}{4}$, the probability $p=b/m$, and the total number of communication rounds 
    $
        T 
         = O\Big( \frac{\sqrt{nL} m\epsilon}{G\sqrt{(1+\omega)d \log(1/\delta)}} \max\big\{1, \tau \big\}\Big),
    $
    where $\tau:=\frac{(1+\omega)^{3/2}}{n^{1/2}}$,
    then \soteriaflSVRG satisfies $(\epsilon,\delta)$-LDP
    and the following utility guarantee
    $
        \frac{1}{T}\sum_{t=0}^{T-1} \E\ns{\nabla f(\vx^t)}
        \le O\Big( \frac{G\sqrt{{(1+\omega)}Ld\log(1/\delta)}}{\sqrt{{n}}m\epsilon}\Big).
    $
\end{corollary}

\begin{corollary}[\soteriaflSAGA]\label{cor:saga}
    Suppose that Assumptions~\ref{ass:smoothness} and \ref{ass:bounded-gradient} hold and we combine Theorem~\ref{thm:utility} and Lemma~\ref{lem:para-sgd-svrg-saga}, i.e., choosing stepsize
    $
        \eta_t \equiv \eta \leq  \frac{\min\{1, \sqrt{n/(1+\omega)^3}\}}{3L},
    $
    where we set $\beta = \frac{(1+\omega)^2\min\{1, n/(1+\omega)^3\}}{3n}$, minibatch size $b=3m^{2/3}$,
    shift stepsize $\gamma_t\equiv \sqrt{\frac{1+2\omega}{2(1+\omega)^3}}$, and privacy variance $\sigma_p^2 = O\big(\frac{G^2T\log({1}/{\delta})}{m^2\epsilon^2}\big)$.
    If we further let the communication rounds 
    $
        T 
         = O\Big( \frac{\sqrt{nL} m\epsilon}{G\sqrt{(1+\omega)d \log(1/\delta)}} \max\big\{1, \tau \big\}\Big),
    $
    where $\tau:=\frac{(1+\omega)^{3/2}}{n^{1/2}}$,
    then \soteriaflSAGA satisfies $(\epsilon,\delta)$-LDP
    and the following utility guarantee
    $
        \frac{1}{T}\sum_{t=0}^{T-1} \E\ns{\nabla f(\vx^t)}
        \le O\Big( \frac{G\sqrt{{(1+\omega)}Ld\log(1/\delta)}}{\sqrt{{n}}m\epsilon}\Big).
    $
\end{corollary}

Interestingly, \soteriafl-style algorithms are more communication-efficient than \cdpSGD when the local dataset size $m$ is large, with a communication complexity of $O(m)$, in contrast to $O(m^2)$ for \cdpSGD. In terms of utility, \soteriaflSVRG and \soteriaflSAGA can achieve the same utility as \cdpSGD, while \soteriaflGD and \soteriaflSGD achieve a slightly worse guarantee than that of \cdpSGD by a factor of $1+\sqrt{\tau}$, where $\tau:=\frac{(1+\omega)^{3/2}}{n^{1/2}}$ is small when the number of clients $n$ is large.
 
\begin{table}[t]
	\caption{Gradient complexity for our proposed \soteriafl-style algorithms, which is computed as the product of the total number of communication rounds $T$ and the minibatch size $b$. Here, for notation simplicity, $K := \frac{\sqrt{nL} m\epsilon}{G\sqrt{(1+\omega)d \log(1/\delta)}}$ and $\tau:=\frac{(1+\omega)^{3/2}}{n^{1/2}}$.}
	\label{tab:gradient-complexity}
	\centering
	\small
	\begin{threeparttable}
	    \renewcommand{\arraystretch}{3}
		\begin{tabular}{|c|c|c|c|}
			\hline
			\bf Algorithms 
			& \makecell{\soteriaflGD \\ 
			(\blue{Option \RomanNumeralCaps{1}} in Algorithm \ref{alg:soteriafl-detail} \\ with $b=m$)} 
			& \makecell{\soteriaflSGD \\ (\blue{Option \RomanNumeralCaps{1}} in Algorithm \ref{alg:soteriafl-detail})} 
			& \makecell{\soteriaflSVRG \\ \soteriaflSAGA \\ (\blue{Option \RomanNumeralCaps{2}, \RomanNumeralCaps{3}} in Algorithm \ref{alg:soteriafl-detail})} \\
			\hline
			
			\makecell{\bf Gradient \\ \bf Complexity}  
			& $K(1+\sqrt{\tau})m$ 
			& $K(1+\sqrt{\tau})b$
% 			& $K(1+\sqrt{\tau})\frac{m \epsilon G \sqrt{1+\omega}}{\sqrt{(1+\tau)nLd\log(1/\delta)}}$
			& $K(1+\tau)m^{2/3}$\\
			\hline
		\end{tabular}
	\end{threeparttable}
\end{table}

\paragraph{Gradient complexity of \soteriafl-style algorithms.} Although the utility and the communication complexity are the most important considerations in private FL, another worth-noting criterion is the \emph{gradient complexity}, which is defined as the total number of stochastic gradients computed by each client. 
Although \soteriaflGD, \soteriaflSGD, \soteriaflSVRG and \soteriaflSAGA have similar communication complexity (see  Table~\ref{tab:results}), they actually have very different gradient complexities---summarized in Table \ref{tab:gradient-complexity}---since the minibatch sizes and gradient update rules for these algorithms vary a lot. 
The gradient complexity of \soteriaflSVRG/\soteriaflSAGA is usually smaller than \soteriaflSGD, and all of them are smaller than \soteriaflGD. 
In sum, we recommend \soteriaflSVRG/\soteriaflSAGA due to its superior utility and gradient complexity while maintaining almost the same communication complexity as \soteriaflSGD/\soteriaflGD.

\section{Numerical Experiments}
\label{sec:exp}

In this section, we conduct experiments on standard real-world datasets to numerically verify privacy-utility-communication trade-offs among different algorithms.  The code can be accessed at:
% \begin{center}
\url{https://github.com/haoyuzhao123/soteriafl}.
% \end{center}
Concretely, we compare the direct compression algorithm \cdpSGD (Algorithm~\ref{alg:cdp-sgd}), shifted compression algorithms \soteriaflSGD (Algorithm~\ref{alg:soteriafl-detail} with Option \RomanNumeralCaps{1}) and \soteriaflSVRG (Algorithm~\ref{alg:soteriafl-detail} with Option \RomanNumeralCaps{2}), and algorithms without compression {\sf LDP-SGD}~\citep{abadi2016deep, lowy2022private} and {\sf LDP-SVRG}~\citep{lowy2022private} on two nonconvex problems (logistic regression with nonconvex regularization in Section~\ref{sec:logistic} and shallow neural network training in Section~\ref{sec:nn}).

\paragraph{Experiment setup.} In our experiments, we use random-$k$ sparsification (see Example 1 in Section~\ref{sec:pre}) as the compression operator, and we set $k=\lfloor \frac{d}{20}\rfloor$, i.e., randomly select 5\% coordinates over $d$ dimension to communicate. In other words, the number of communication bits \emph{per round} of uncompressed algorithms equals to that of \emph{20 rounds} of compressed algorithms. The number of nodes $n$ is 10.  For the algorithmic parameters, we tune the stepsizes (learning rates) for all algorithms for each nonconvex problem and select their best ones from the set $\{0.01, 0.03, 0.06, 0.1, 0.3, 0.6, 1\}$.  Other parameters are set according to their theoretical values. We would like point out that, in order to achieve privacy guarantee, bounded gradient (Assumption~\ref{ass:bounded-gradient}) is required. However, it is not easy to obtain this upper bound $G$ or it is somewhat large especially for neural networks. Thus, following experiments in previous works \cite{wang2019efficient, zhang2020private, ding2021differentially,lowy2022private}, we also apply gradient clipping (i.e. $\text{clip}_G(\vg)= \min (1, \frac{G}{\| \vg \|} ) \cdot \vg$) in our experiments.
In particular, we choose $G=0.5$ for logistic regression with nonconvex regularization in Section~\ref{sec:logistic} and $G=1$ for shallow neural network training in Section~\ref{sec:nn}.
For the Gaussian perturbation $\vxi$, we will run experiments for different levels of $(\epsilon, \delta)$-LDP guarantee, and compute the variance of $\vxi$ according to the theory.

\subsection{{Logistic regression with nonconvex regularization}}
\label{sec:logistic}
The first task is the logistic regression with a nonconvex regularizer, where the objective function over a data sample $(\va,b) \in D$ is defined as
    \[f(\vx; (\va,b)) :=   \log\left(1+\exp(-b \va^\top\vx )\right) + \lambda\sum \limits_{j=1}^d\frac{x_j^2}{1+x_j^2}.\]
Here, $\va\in\R^d$ denotes the features, $b$ is its label, and $\lambda$ is the regularization parameter.  
We choose $\lambda=0.2$ and run the experiments on the standard \dataset{a9a} dataset~\cite{chang2011libsvm}.  
To demonstrate the privacy-utility-communication trade-offs, we consider three levels of $(\epsilon, \delta)$-LDP with different $\epsilon = 1, 5, 10$ and a common $\delta=10^{-3}$, where the experimental results are reported in Figures~\ref{fig:a9a_eps_1}--\ref{fig:a9a_eps_10} respectively.

\begin{figure}[!htb]
\centering
	\begin{tabular}{cc}
		\includegraphics[width=0.34\textwidth]{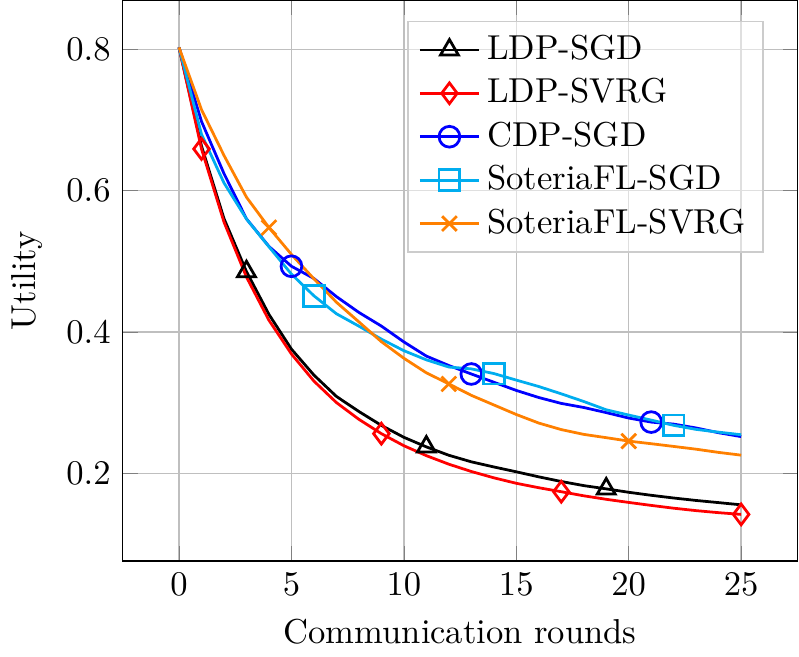} \hspace{5mm}
&
		\includegraphics[width=0.34\textwidth]{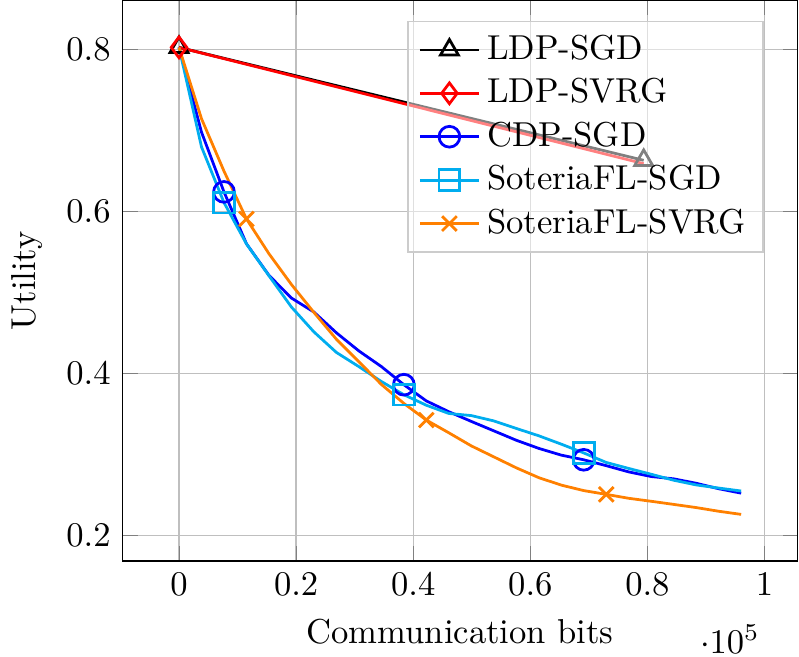} \\
		\includegraphics[width=0.34\textwidth]{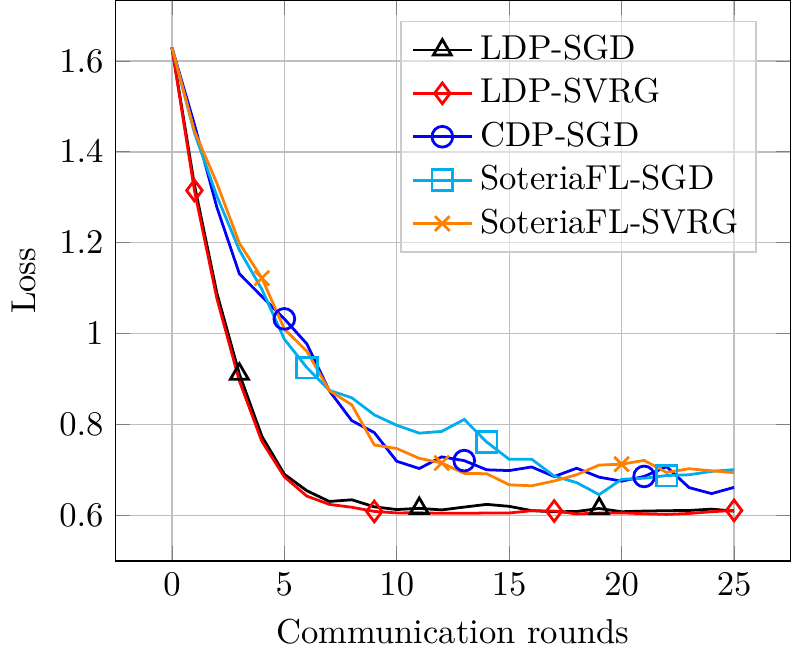} \hspace{5mm}
&
		\includegraphics[width=0.34\textwidth]{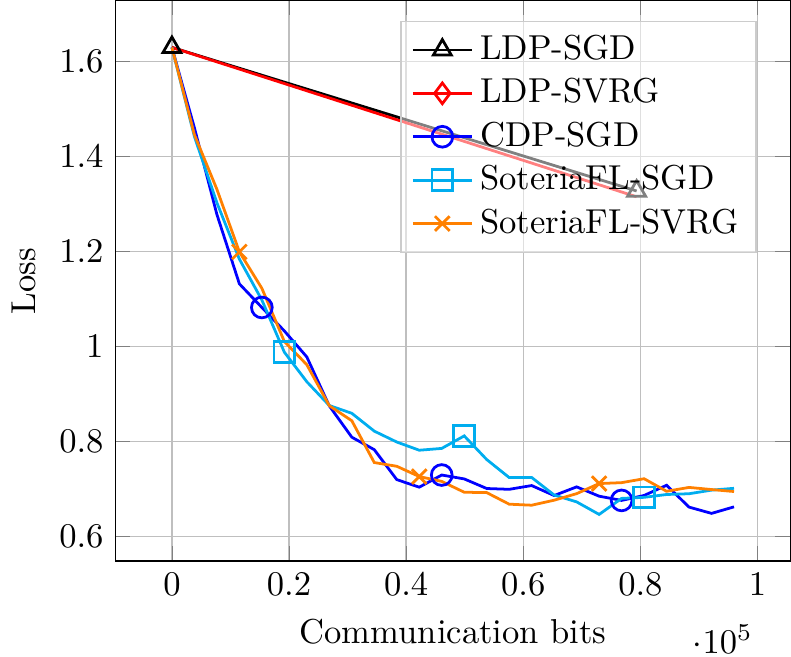}
\end{tabular}
	\caption{Logistic regression with nonconvex regularization on the \dataset{a9a} dataset under ($\epsilon,\delta$)-LDP with \blue{$\epsilon=1$} and $\delta=10^{-3}$. The top (resp. bottom) row is for utility (resp. training loss) vs. communication rounds and communication bits.}
	\label{fig:a9a_eps_1}
% \end{figure}
% 
% \begin{figure}[!htb]
	\vspace{5mm}
	\centering
	\begin{tabular}{cc}
		\includegraphics[width=0.34\textwidth]{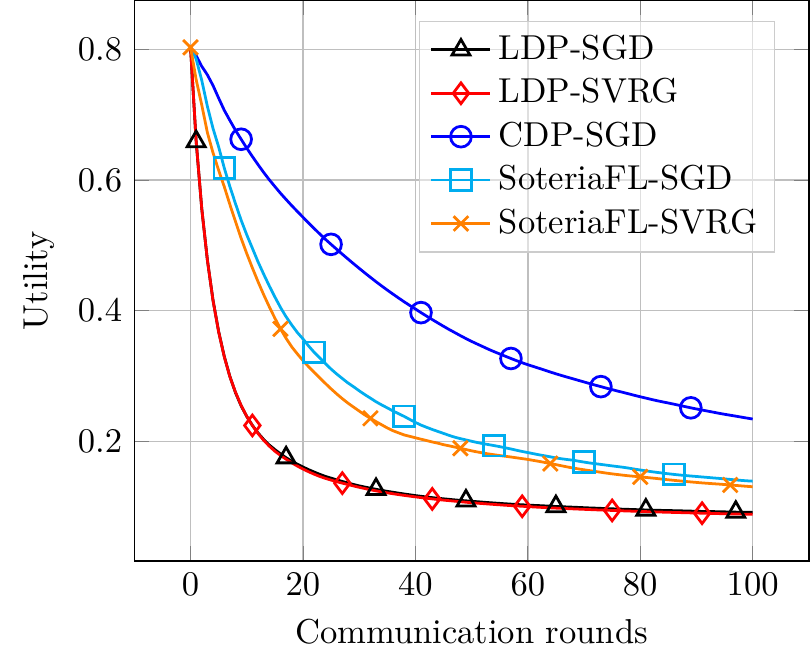} \hspace{5mm}
&
		\includegraphics[width=0.34\textwidth]{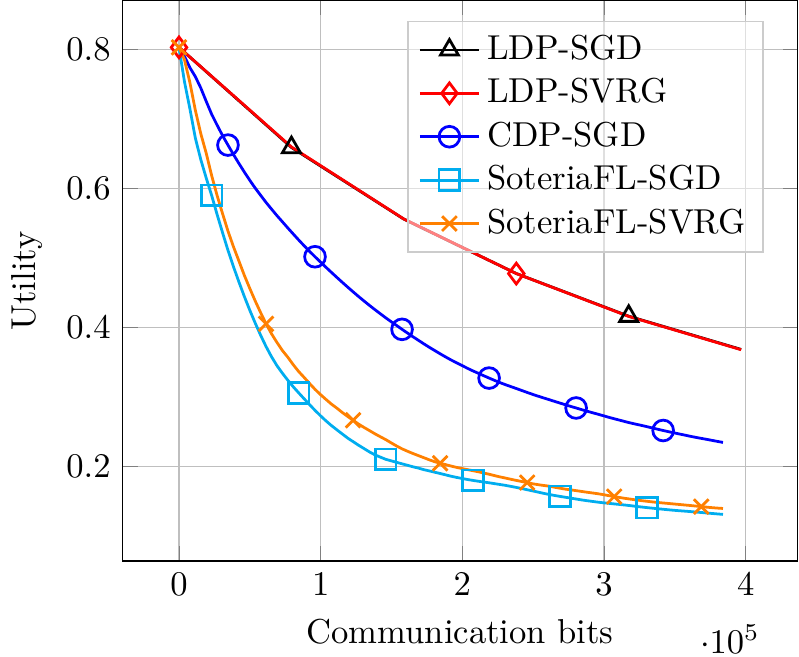} \\
		\includegraphics[width=0.34\textwidth]{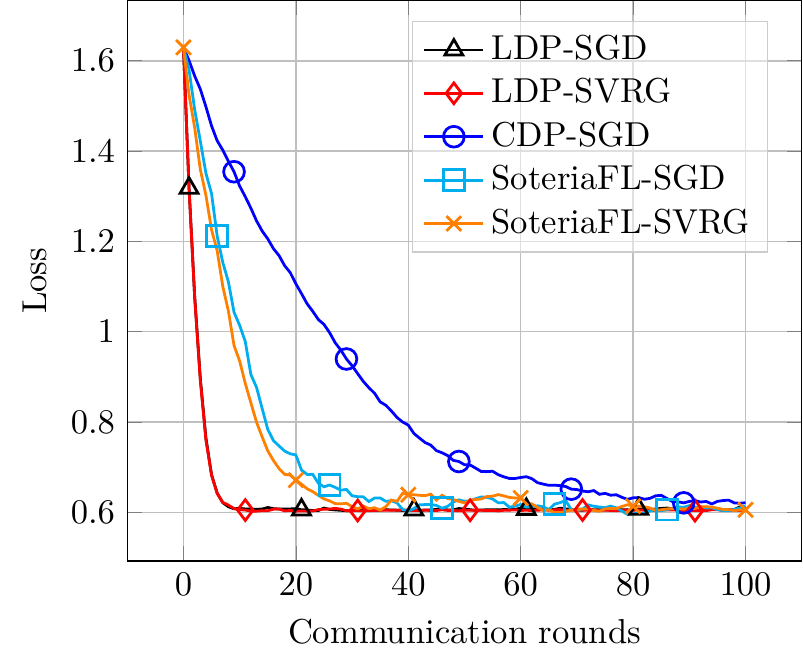} \hspace{5mm}
&
		\includegraphics[width=0.34\textwidth]{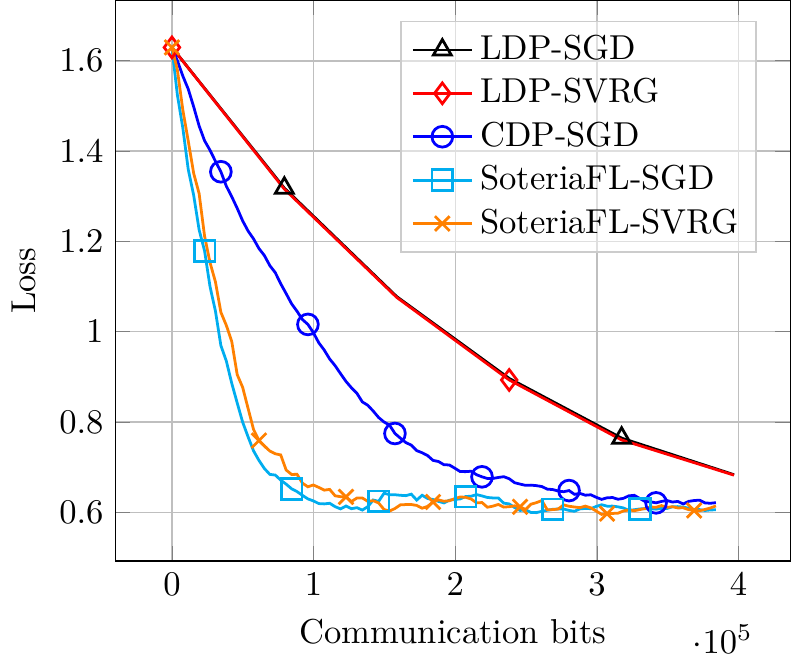}
\end{tabular}
	\caption{Logistic regression with nonconvex regularization on the \dataset{a9a} dataset under ($\epsilon,\delta$)-LDP with \blue{$\epsilon=5$} and $\delta=10^{-3}$. The top (resp. bottom) row is for utility (resp. training loss) vs. communication rounds and  communication bits.}
	\label{fig:a9a_eps_5}
\end{figure}

\clearpage
 \begin{figure}[!htb]
	\centering
	\begin{tabular}{cc}
		\includegraphics[width=0.34\textwidth]{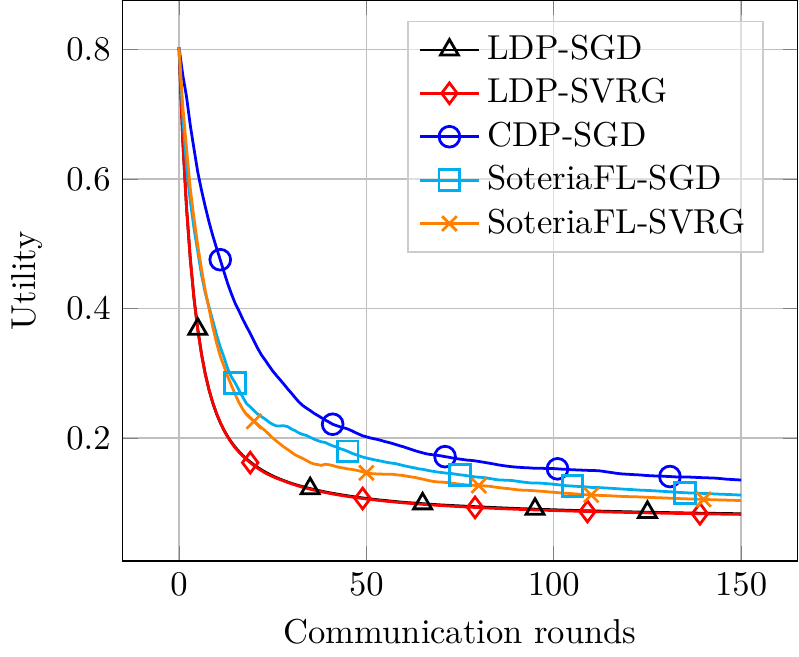} \hspace{5mm}
		&
		\includegraphics[width=0.34\textwidth]{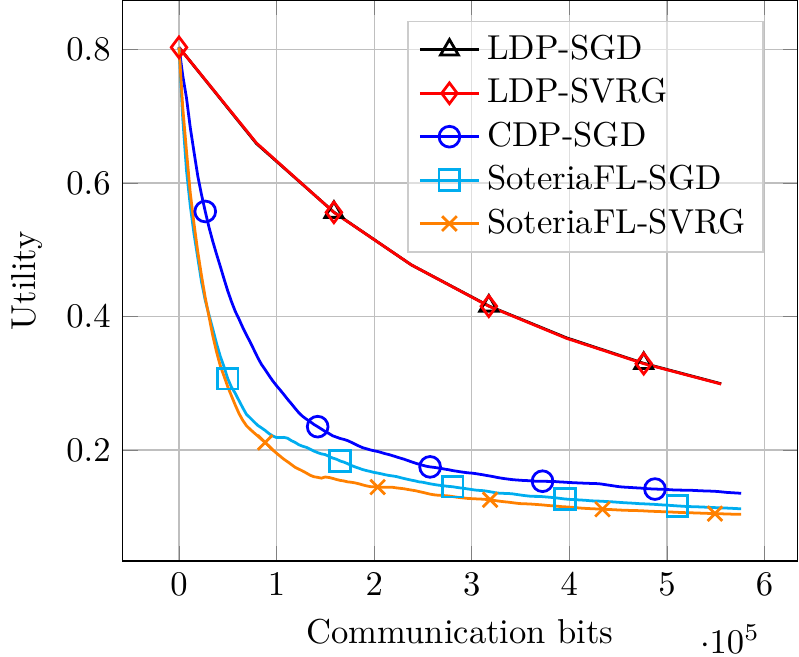} \\
		\includegraphics[width=0.34\textwidth]{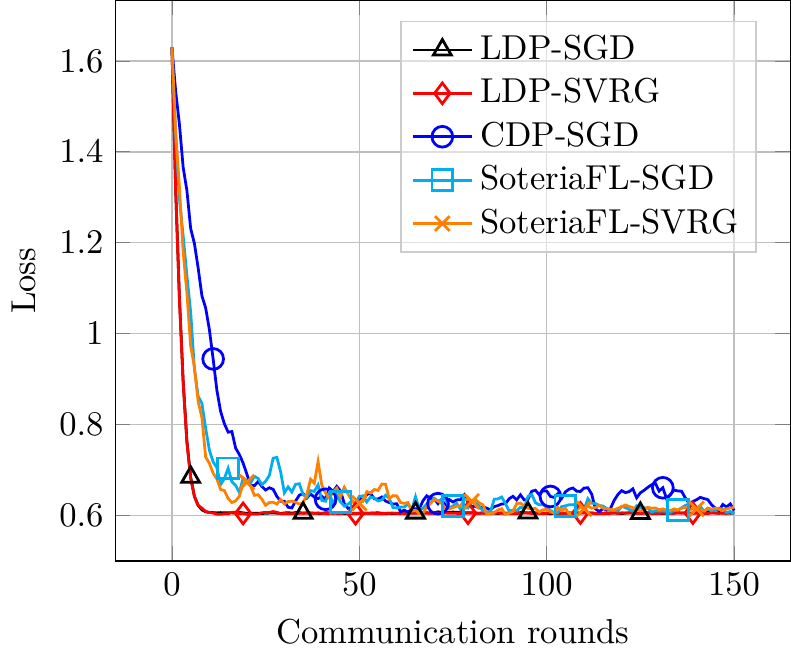} \hspace{5mm}
		&
		\includegraphics[width=0.34\textwidth]{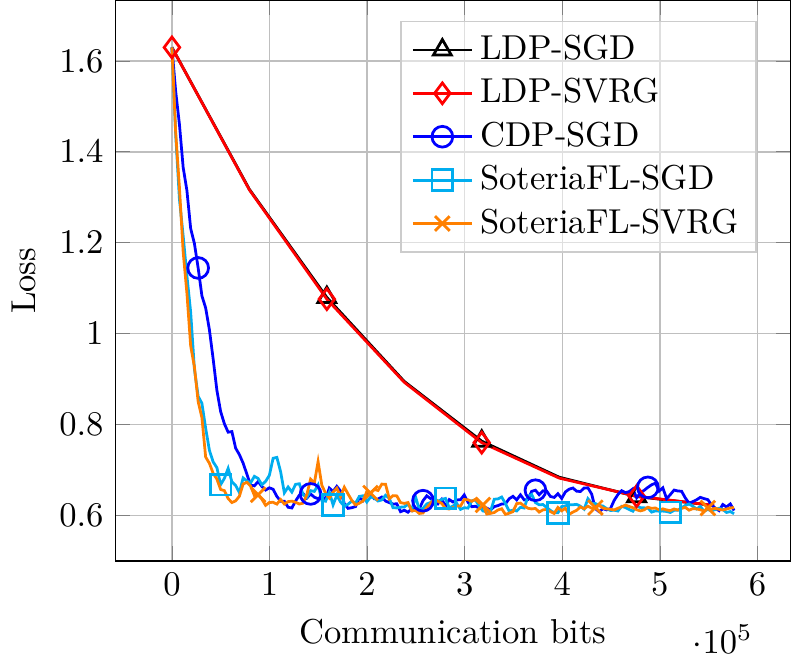}
\end{tabular}
	\caption{Logistic regression with nonconvex regularization on the \dataset{a9a} dataset under ($\epsilon,\delta$)-LDP with \blue{$\epsilon=10$} and $\delta=10^{-3}$. The top (resp. bottom) row is for utility (resp. training loss) vs. communication rounds and  communication bits.}
	\label{fig:a9a_eps_10}
\end{figure}

\paragraph{Remark.} 
From the experimental results, it can be seen that the two uncompressed algorithms ({\sf LDP-SGD} and {\sf LDP-SVRG}) converge faster than the three compressed algorithms (\cdpSGD, \soteriaflSGD, \soteriaflSVRG) in terms of \emph{communication rounds} (see left columns in each figure).
However, in terms of \emph{communication bits} (see right columns in each figure), compressed algorithms perform better than the uncompressed algorithms. This validates that communication compression indeed provide significant savings in terms of communication cost.
The figures also confirm that shifted compression based \soteriafl typically performs better than direct compression based \cdpSGD in both utility and training loss. For \soteriafl-style algorithms, it turns out that \soteriaflSVRG performs slightly better than \soteriaflSGD in the utility (see top rows in each figure). This is quite consistent with our theoretical results.

\vspace{3mm}
\subsection{{Shallow neural network training}}
\label{sec:nn}

We consider a simple 1-hidden layer neural network training task, with $64$ hidden neurons, sigmoid activation functions, and the cross-entropy loss.
The objective function over a data sample $(\va,b)$ is defined as 
$$f(\vx; (\va,b)) = \ell(\mathsf{softmax}(\bm W_2 ~\mathsf{sigmoid}( \bm W_1 \va + \bm c_1) + \bm c_2), b),$$
where $\ell(\cdot, \cdot)$ denotes the cross-entropy loss, the optimization variable is collectively denoted by $\vx = \text{vec}(\bm W_1, \bm c_1, \bm W_2, \bm c_2)$, with the dimensions of the network parameters $\bm W_1$, $\bm c_1$, $\bm W_2$, $\bm c_2$ being $64 \times 784$, $64 \times 1$, $10 \times 64$, and $10 \times 1$, respectively.
Here, we run the experiments on the standard \dataset{MNIST} dataset~\citep{lecun1998gradient}.
To demonstrate the privacy-utility-communication trade-offs, we consider five levels of $(\epsilon, \delta)$-LDP with $\epsilon = 1, 2, 4, 8, 16$ and a common $\delta=10^{-3}$, where
the experimental results are reported in Figures~\ref{fig:NN_eps_1}--\ref{fig:NN_eps_16}, respectively.

\begin{figure}[!htb]
	\centering
	\begin{tabular}{cc}
		\includegraphics[width=0.34\textwidth]{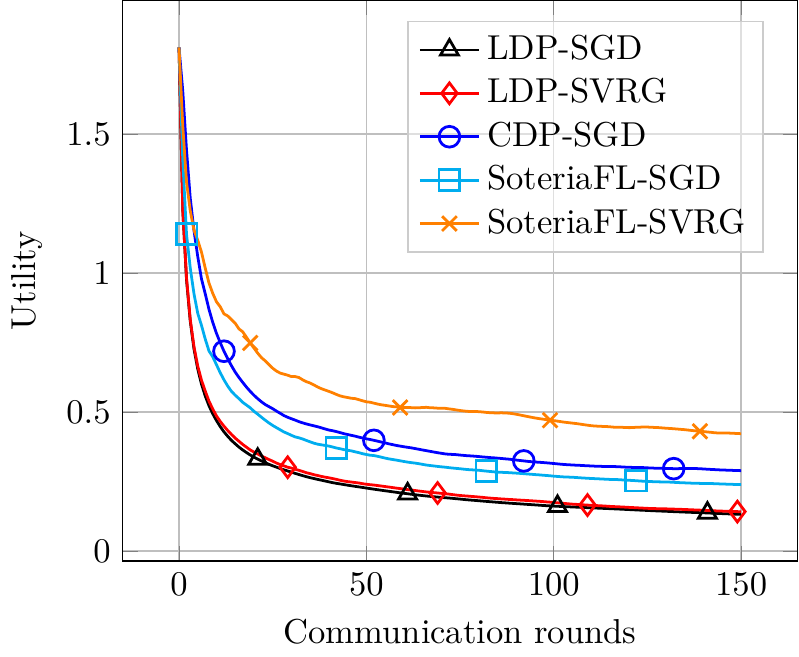} \hspace{5mm}
&
		\includegraphics[width=0.34\textwidth]{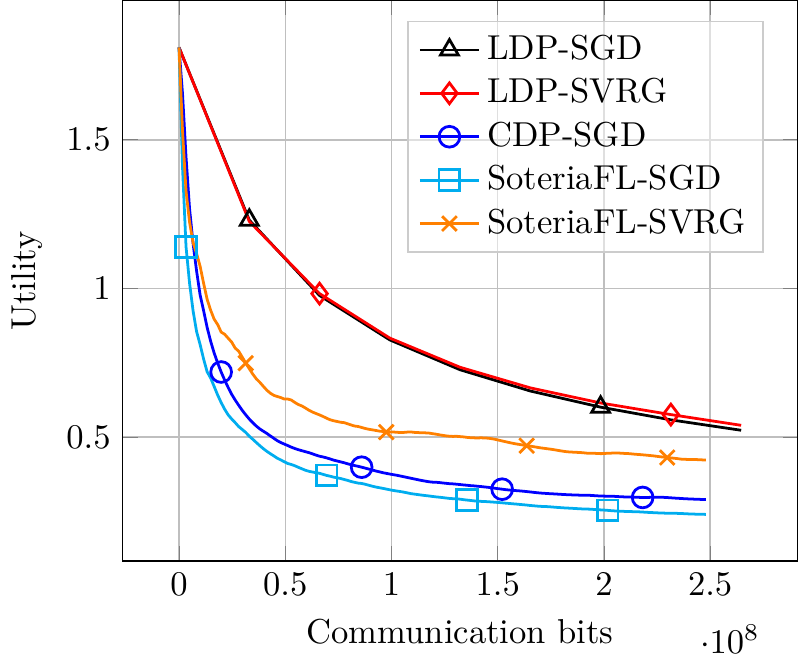} \\
		\includegraphics[width=0.34\textwidth]{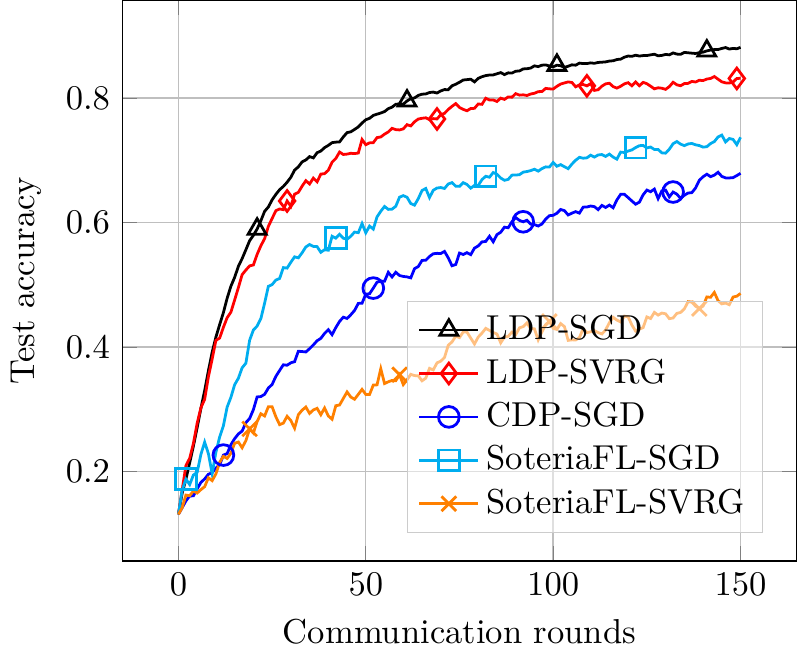} \hspace{5mm}
&
		\includegraphics[width=0.34\textwidth]{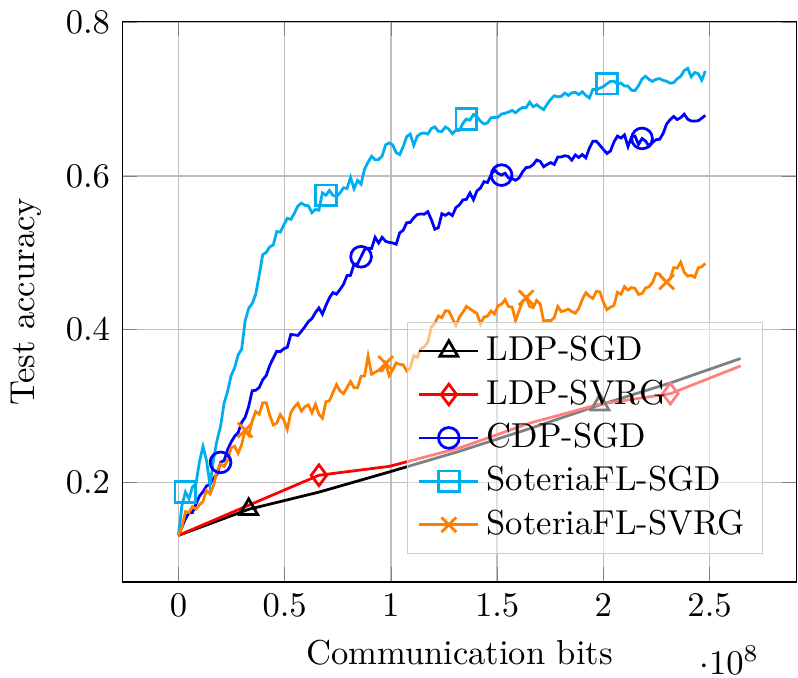}
\end{tabular}
	\caption{Shallow neural network training on the \dataset{MNIST} dataset under ($\epsilon,\delta$)-LDP with \blue{$\epsilon=1$} and $\delta=10^{-3}$. The top (resp. bottom) row is for utility (resp. test accuracy) vs. communication rounds and  communication bits.}
	\label{fig:NN_eps_1}
% \end{figure}
%
%
% \begin{figure}[!htb]
	
	\vspace{5mm}
	\centering
	\begin{tabular}{cc}
		\includegraphics[width=0.34\textwidth]{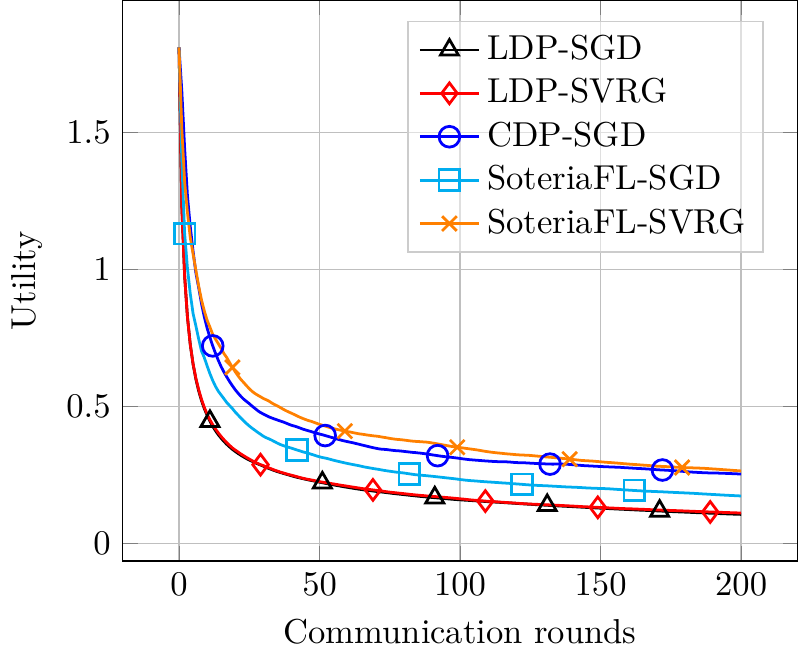} \hspace{5mm}
		&
		\includegraphics[width=0.34\textwidth]{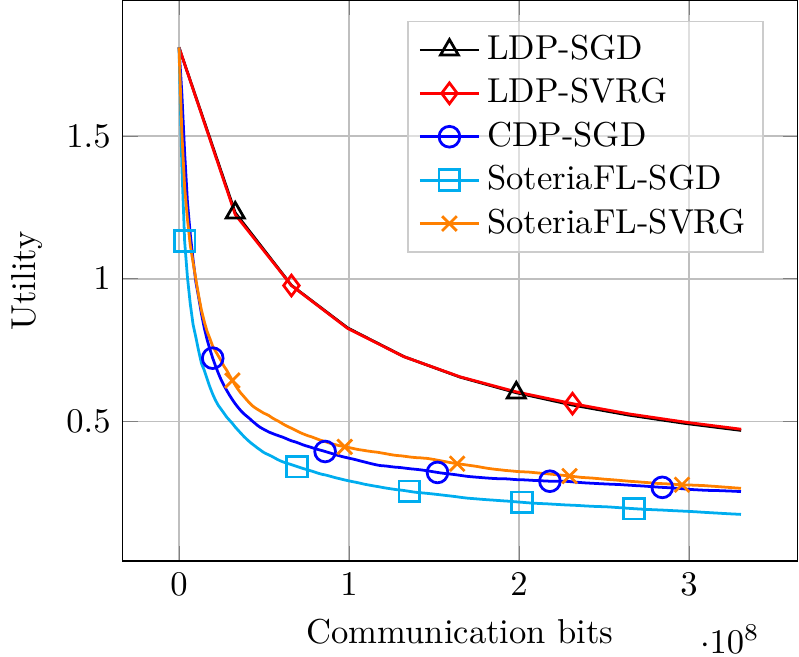} \\
		\includegraphics[width=0.34\textwidth]{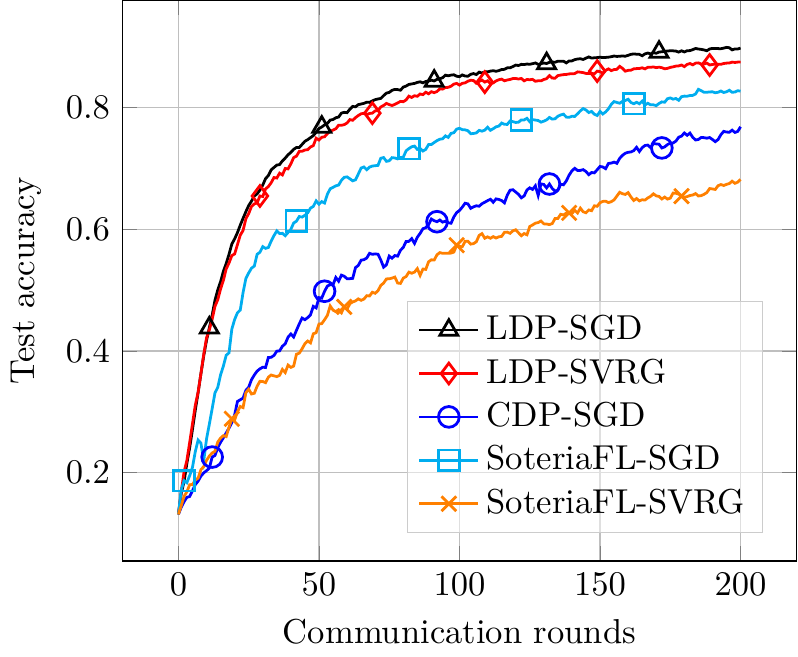} \hspace{5mm}
		&
		\includegraphics[width=0.34\textwidth]{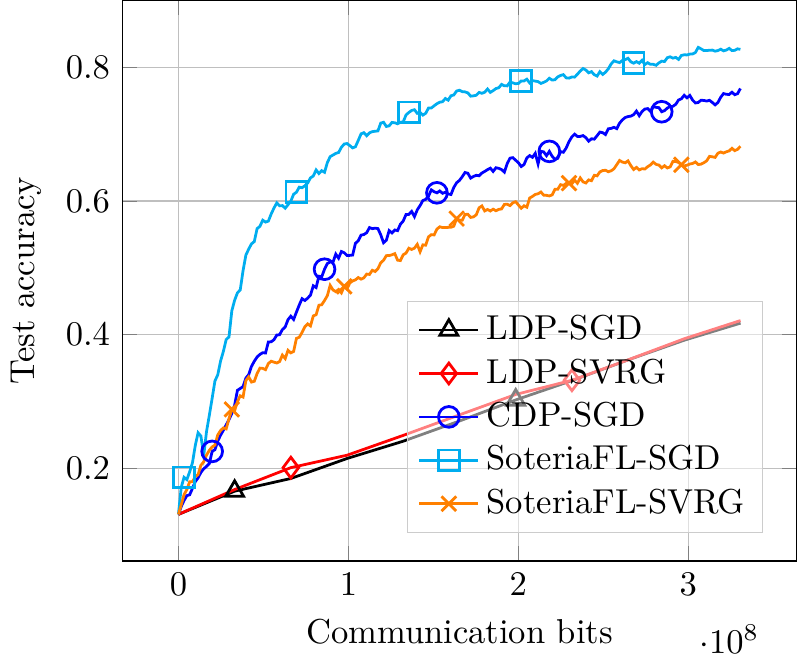}
	\end{tabular}
	\caption{Shallow neural network training on the \dataset{MNIST} dataset under ($\epsilon,\delta$)-LDP with \blue{$\epsilon=2$} and $\delta=10^{-3}$. The top (resp. bottom) row is for utility (resp. test accuracy) vs. communication rounds and  communication bits.}
	\label{fig:NN_eps_2}
\end{figure}

\begin{figure}[!htb]
	\centering
	\begin{tabular}{cc}
		\includegraphics[width=0.34\textwidth]{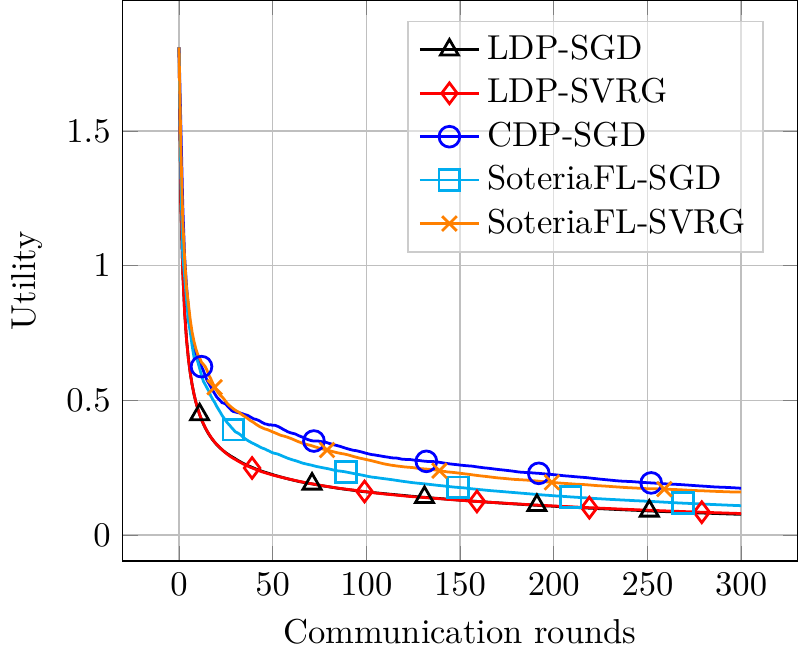} \hspace{5mm}
		&
		\includegraphics[width=0.34\textwidth]{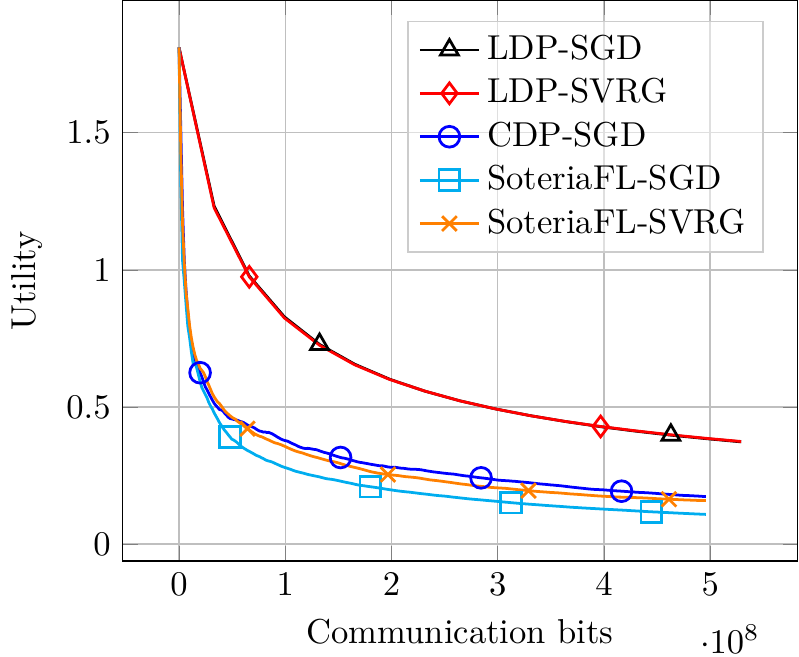} \\
		\includegraphics[width=0.34\textwidth]{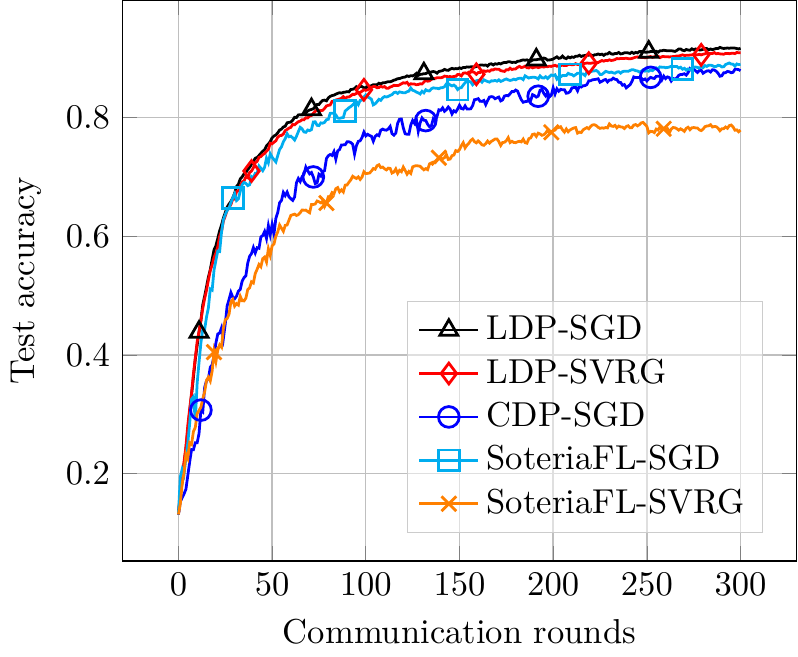} \hspace{5mm}
		&
		\includegraphics[width=0.34\textwidth]{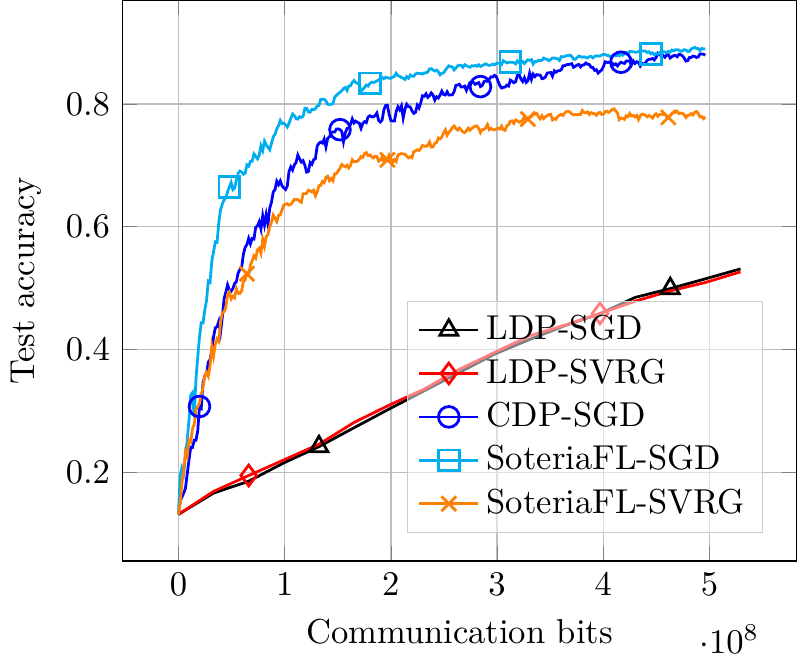}
	\end{tabular}
	\caption{Shallow neural network training on the \dataset{MNIST} dataset under ($\epsilon,\delta$)-LDP with \blue{$\epsilon=4$} and $\delta=10^{-3}$. The top (resp. bottom) row is for utility (resp. test accuracy) vs. communication rounds and  communication bits.}
	\label{fig:NN_eps_4}
% \end{figure}
%
%
% \begin{figure}[!htb]
	
	\vspace{5mm}
	\centering
	\begin{tabular}{cc}
		\includegraphics[width=0.34\textwidth]{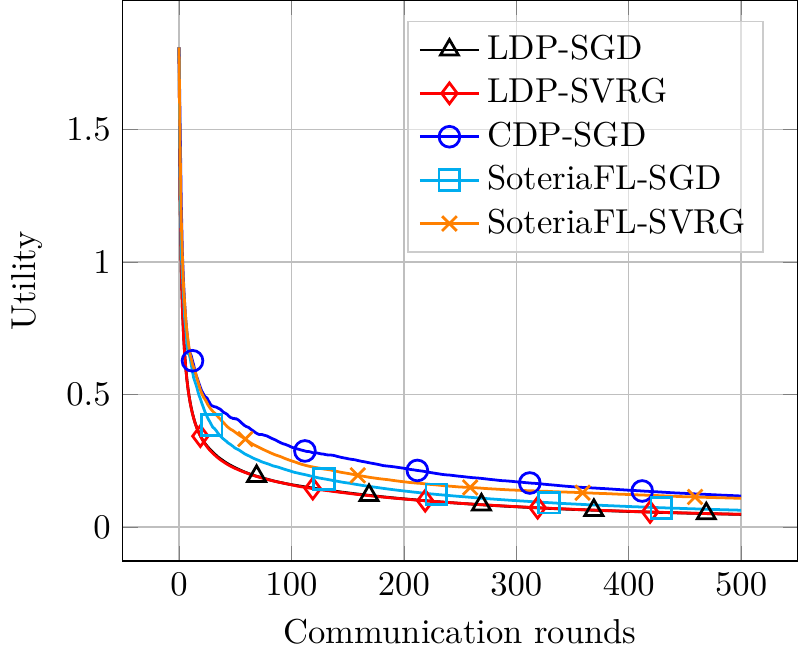} \hspace{5mm}
		&
		\includegraphics[width=0.34\textwidth]{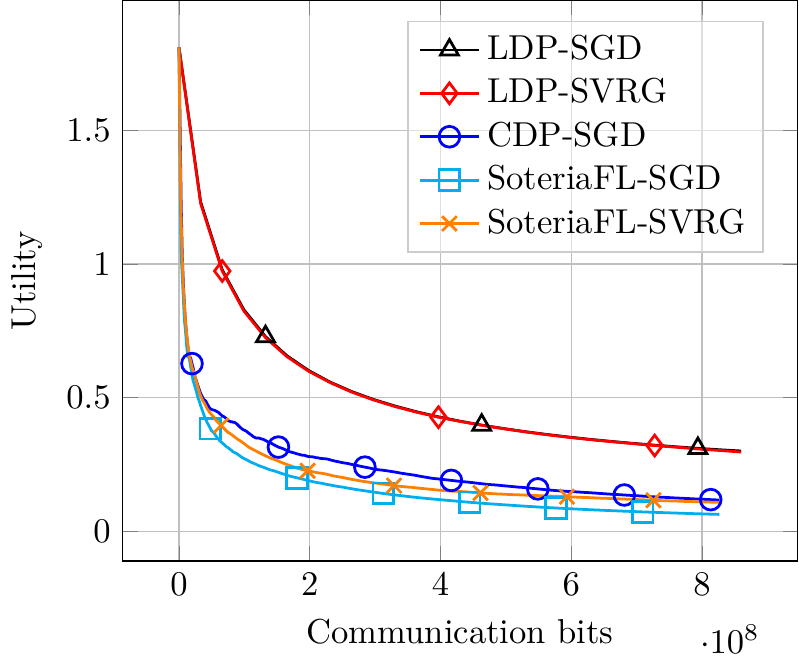} \\
		\includegraphics[width=0.34\textwidth]{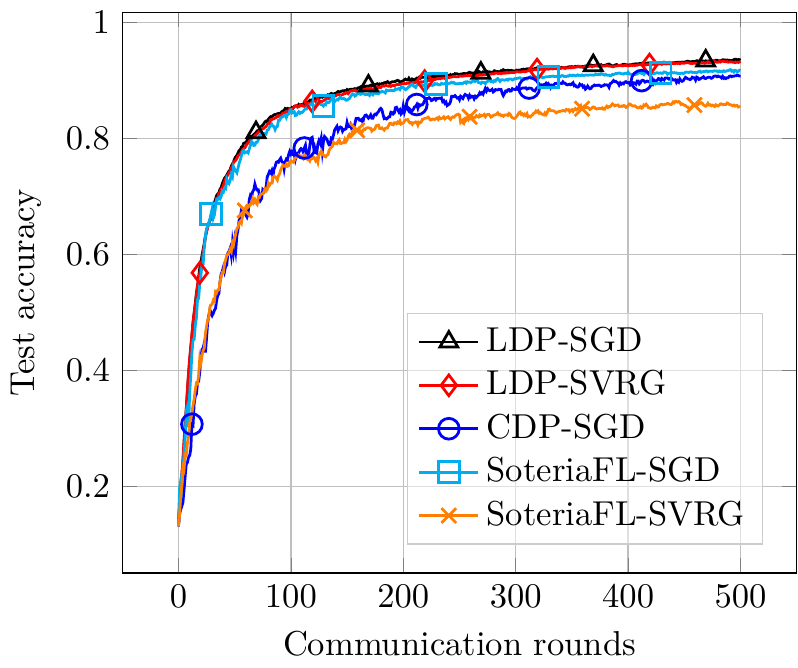} \hspace{5mm}
		&
		\includegraphics[width=0.34\textwidth]{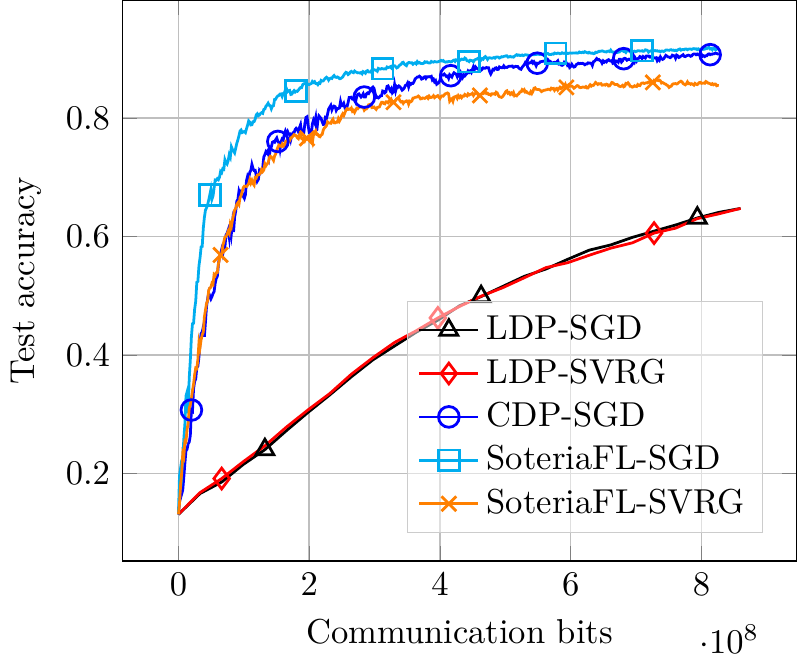}
	\end{tabular}
	\caption{Shallow neural network training on the \dataset{MNIST} dataset under ($\epsilon,\delta$)-LDP with \blue{$\epsilon=8$} and $\delta=10^{-3}$. The top (resp. bottom) row is for utility (resp. test accuracy) vs. communication rounds and  communication bits.}
	\label{fig:NN_eps_8}
\end{figure}

\clearpage
\begin{figure}[!htb]
	\centering
	\begin{tabular}{cc}
		\includegraphics[width=0.34\textwidth]{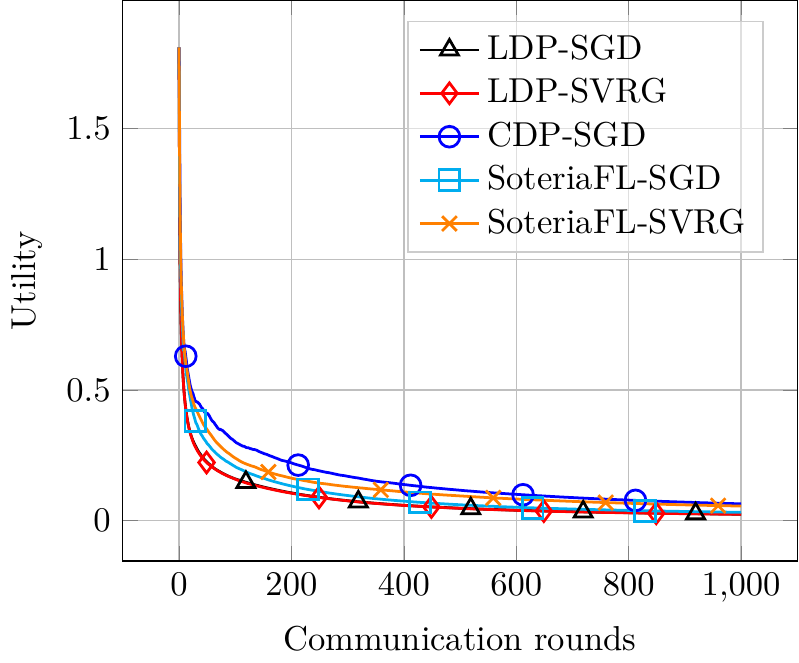} \hspace{5mm}
		&
		\includegraphics[width=0.34\textwidth]{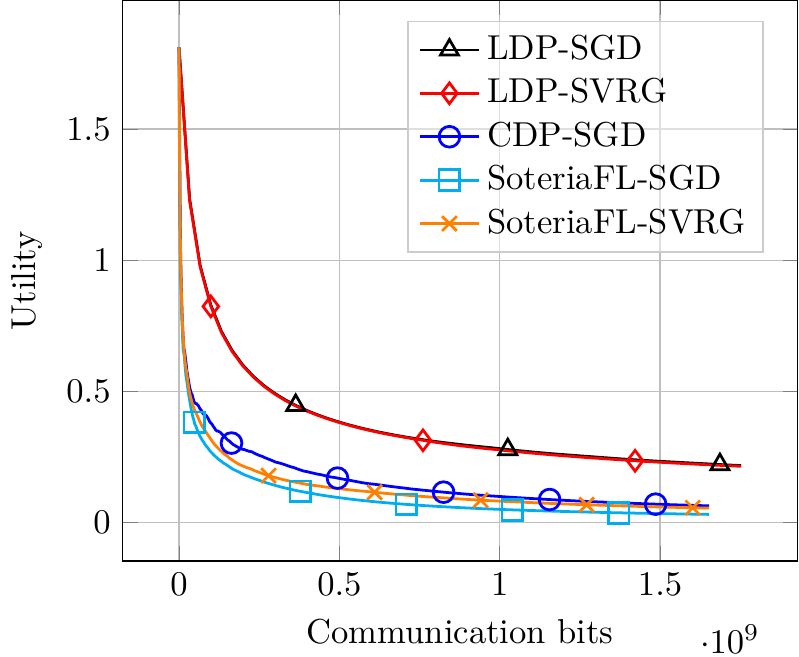} \\
		\includegraphics[width=0.34\textwidth]{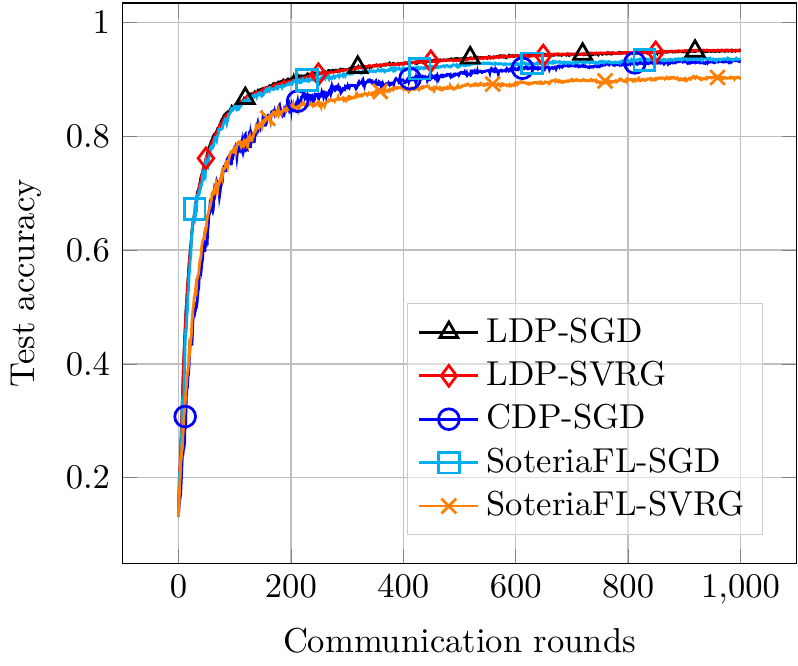} \hspace{5mm}
		&
		\includegraphics[width=0.34\textwidth]{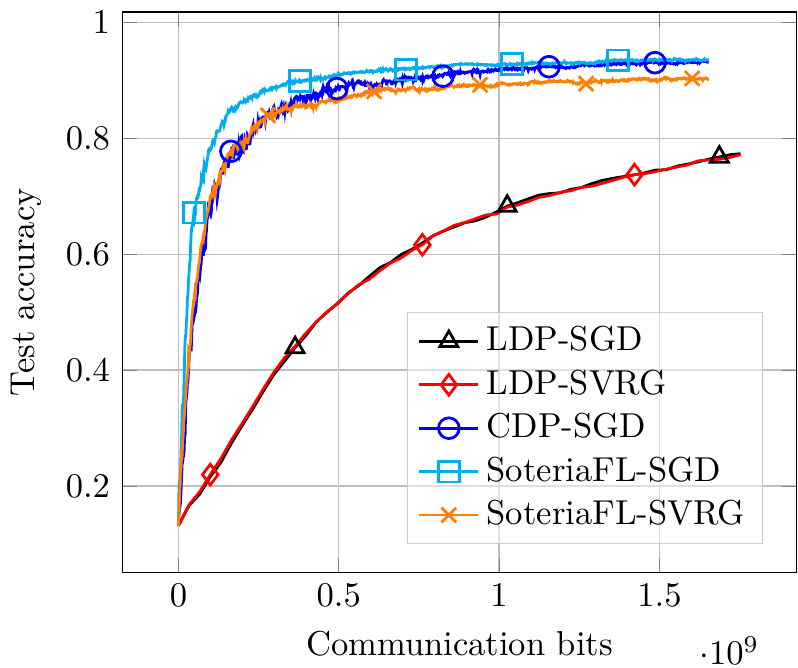}
	\end{tabular}
	\caption{Shallow neural network training on the \dataset{MNIST} dataset under ($\epsilon,\delta$)-LDP with \blue{$\epsilon=16$} and $\delta=10^{-3}$. The top (resp. bottom) row is for utility (resp. test accuracy) vs. communication rounds and  communication bits.}
	\label{fig:NN_eps_16}
\end{figure}

\paragraph{Remark.} 
Note that here we report the test accuracy for training the neural network instead of the training loss as earlier (see bottom rows in Figures~\ref{fig:NN_eps_1}--\ref{fig:NN_eps_16} vs. in Figures~\ref{fig:a9a_eps_1}--\ref{fig:a9a_eps_10}). The takeaways from the experimental results are similar to previous experiments on logistic regression with nonconvex regularization (Figures~\ref{fig:a9a_eps_1}--\ref{fig:a9a_eps_10}). 
Again, the two uncompressed algorithms ({\sf LDP-SGD} and {\sf LDP-SVRG}) converge faster than the three compressed algorithms (\cdpSGD, \soteriaflSGD, \soteriaflSVRG) in terms of \emph{communication rounds} (see left columns in each figure), {but the gap becomes smaller when the privacy level $\epsilon$ gets larger (i.e. less privacy guarantee)}.
However, in terms of \emph{communication bits} (see right columns in each figure), compressed algorithms again perform much better than the uncompressed algorithms, validating the advantage of communication compression schemes. Last but not least, shifted compression based \soteriafl-SGD performs better than direct compression based \cdpSGD in both utility and test accuracy. 
However, it turns out that \soteriaflSVRG may perform worse than \cdpSGD for training this shallow neural network.

\section{Conclusion}
\label{sec:conclusion}

We propose \soteriafl, a unified framework for private FL, which accommodates a general family of local gradient estimators including popular stochastic variance-reduced gradient methods and the state-of-the-art shifted compression scheme. 
A unified characterization of its performance trade-offs in terms of privacy, utility (convergence accuracy), and communication complexity is presented, which is then instantiated to arrive at several new private FL algorithms. All of these algorithms are shown to perform better than the plain \cdpSGD algorithm especially when the local dataset size is large, and have lower communication complexity compared with other private FL  algorithms without compression.

\section*{Acknowledgements}

The work of Z. Li, B. Li and Y. Chi is supported in part by ONR N00014-19-1-2404, by AFRL under FA8750-20-2-0504, and by NSF
under CCF-1901199, CCF-2007911, DMS-2134080 and CNS-2148212.
The work of H. Zhao is supported in part by NSF, ONR, Simons Foundation, DARPA and SRC through awards to S. Arora.
B. Li is also gratefully supported by Wei Shen and Xuehong Zhang
Presidential Fellowship at Carnegie Mellon University.

%\newpage
\bibliographystyle{abbrvnat}
\bibliography{ref}

%%%%%%%%%%%%%%%%%%%%%%%%%%%%%%%%%%%%%%%%%%%%%%%%%%%%%%%%%%%%

\newpage
%\newpage
\appendix

%=========
\newcommand{\Et}[1]{\mathbb{E}_t[ #1 ]}
\newcommand{\Etb}[1]{\mathbb{E}_t\Big[ #1 \Big]}
\newcommand{\EtB}[1]{\mathbb{E}_t\left[ #1 \right]}
\newcommand{\xt}{\vx^t}
\newcommand{\xtn}{\vx^{t+1}}
\newcommand{\xT}{\vx^T}
\newcommand{\xTn}{\vx^{T+1}}
\newcommand{\wt}{\vw^t}
\newcommand{\wtn}{\vw^{t+1}}
\newcommand{\wT}{\vw^T}
\newcommand{\wTn}{\vw^{T+1}}
\newcommand{\wijt}{\vw_{i,j}^t}
\newcommand{\wijtn}{\vw_{i,j}^{t+1}}
\newcommand{\sigp}{{\sigma_p^2}}

\newcommand{\st}{\vs^t}
\newcommand{\stn}{\vs^{t+1}}
\newcommand{\sit}{{\vs_i^t}}
\newcommand{\sitn}{\vs_i^{t+1}}
\newcommand{\siT}{\vs_i^T}
\newcommand{\siTn}{\vs_i^{T+1}}

\newcommand{\git}{\vg_i^{t}}
\newcommand{\tgit}{\tilde{\vg}_i^t}
\newcommand{\xit}{\vxi_i^{t}}
\newcommand{\cit}{\cC_i^{t}}
\newcommand{\vit}{{\vv_i^t}}
\newcommand{\vitn}{\vv_i^{t+1}}
\newcommand{\vvt}{{\vv^t}}
\newcommand{\vvtn}{\vv^{t+1}}

\newcommand{\etat}{\eta_{t}}
\newcommand{\thetat}{\theta_{t}}
\newcommand{\alphat}{\alpha_{t}}
\newcommand{\pt}{p_{t}}
\newcommand{\betat}{\beta_{t}}
\newcommand{\gammat}{\gamma_{t}}
\newcommand{\cst}{\sum_{i=1}^n \ns{\nabla f_i(\xt) - \sit}}
\newcommand{\cstn}{\sum_{i=1}^n \ns{\nabla f_i(\xtn) - \sitn}}
\newcommand{\sumn}{\sum_{i=1}^n}
\newcommand{\gSt}{\gS^t}
\newcommand{\gStn}{\gS^{t+1}}
\newcommand{\Deltat}{\Delta^t}
\newcommand{\Deltatn}{\Delta^{t+1}}

\newcommand{\etaT}{\eta_{T}}
\newcommand{\thetaT}{\theta_{T}}
\newcommand{\alphaT}{\alpha_{T}}
\newcommand{\pT}{p_{T}}
\newcommand{\betaT}{\beta_{T}}
\newcommand{\gammaT}{\gamma_{T}}
\newcommand{\csT}{\sum_{i=1}^n \ns{\nabla f_i(\wT) - \siT}}
\newcommand{\csTn}{\sum_{i=1}^n \ns{\nabla f_i(\wTn) - \siTn}}
\newcommand{\cszero}{\sum_{i=1}^n \ns{\nabla f_i(\vw^0) - \vs_i^0}}
%=========

\section*{\Large Appendix}
We now provide all missing proofs. 
Concretely, Appendix~\ref{sec:proof-privacy} and \ref{sec:proof-utility} provide the detailed proofs for our unified privacy guarantee in Theorem~\ref{thm:privacy} and unified utility and communication complexity analysis in Theorem~\ref{thm:utility}, respectively.
Appendix~\ref{sec:proof-cdp-sgd} provides the proof for \cdpSGD (Theorem~\ref{thm:cdp-sgd}).
Finally, Appendix~\ref{sec:proof-lemma} provides the proofs for Section~\ref{sec:algorithms}, including Lemma~\ref{lem:para-sgd-svrg-saga} (showing that several local gradient estimators satisfy the generic Assumption~\ref{ass:unified}) and Corollaries~\ref{cor:sgd}--\ref{cor:saga} (instantiating Lemma~\ref{lem:para-sgd-svrg-saga} in the unified Theorem~\ref{thm:utility}) for the proposed \soteriafl-style algorithms.

\section{Proof of Theorem \ref{thm:privacy}}
\label{sec:proof-privacy}

In the proof of Theorem \ref{thm:privacy}, we apply a moment argument (similar to \citep{abadi2016deep}) to prove the local differential privacy guarantees. Before going into the detailed proof, we first define some concepts.

\paragraph{Moment generating function.} Assume that there is a mechanism $\gM:\gD\to\gR$. For neighboring datasets $D,D'\in\gD$, a mechanism $\gM$, auxiliary inputs $\text{aux}$, and an outcome $o\in\gR$, we define the private loss at $o$ as
\begin{equation*}
    c(o; \gM, \text{aux}, D, D') := \log \frac{\Pr\{\gM(\text{aux},D) = o\}}{\Pr\{\gM(\text{aux}, D') = o\}}.
\end{equation*}
We also define
\begin{align*}
    &\alpha^{\gM}(\lambda; \text{aux}, D, D')
   :=  \log \E_{o\sim \gM(\text{aux}, D)}\left[\exp\left(\lambda\cdot c(o; \gM, \text{aux}, D, D')\right)\right]
\end{align*}
and
\begin{align*}
    \alpha_i^{\gM}(\lambda) := \max_{\text{aux},D,D'}\alpha^{\gM}(\lambda; \text{aux}, D, D'),
\end{align*}
where $D = (D_{-i},D_i), D' = (D_{-i},D'_i)$ are neighboring datasets that differ only at client $i$, $D_{-i}$ denoting all the data at clients other than client $i$. We call $\alpha^{\gM}(\lambda; \text{aux}, D, D')$ and $\alpha_i^{\gM}(\lambda)$ the log moment generating functions.

\paragraph{Sub-mechanisms.} We assume that there are $n\times T$ sub-mechanisms $\{\gM_i^t\}_{i\in [n], t\le T}$ in $\gM$, where $\gM_i^t$ corresponds to the mechanism for client $i$ in round $t$. We further let $\gM_i^t := \gA\circ \overline{\gM}_i^t$ be the composition of mechanism $\overline{\gM}_i^t$ and the mechanism $\gA$. Here, $\gA:\gR\to\gR$ is a random mechanism that maps an outcome to another outcome, and $\overline{\gM}_i^t$ is possibly an adaptive mechanism that takes the input of all the outputs before time $t$, i.e. $o_i^{s}$ for all $s < t$ and $i\in [n]$. We assume that given all the previous outcomes $o_i^{s}$ for $s <t$, the random mechanisms $\overline{\gM}_i^t$ for all $i\in [n]$ are independent w.r.t. each other (this is satisfied in \soteriafl).
In \soteriafl (Algorithm~\ref{alg:soteriafl}), $\gA$ corresponds to the compression step, and $\overline{\gM}_i^t$ corresponds the Gaussian perturbation.

Before proving Theorem~\ref{thm:privacy}, we first state the following result from \cite{abadi2016deep}.

\begin{proposition}[Theorem 2 in \citep{abadi2016deep}] \label{prop:tail}
    For any $\epsilon > 0$, the mechanism $\gM$ is $(\epsilon, \delta)$-LDP for client $i$ with $\delta = \min_{\lambda} \exp\left(\alpha_i^{\gM}(\lambda) - \lambda \epsilon\right)$.
\end{proposition}
According to Proposition~\ref{prop:tail}, we know that if the log moment generating function $\alpha_i^{\gM}(\lambda)$ is bounded, then we can show that the mechanism $\gM$ satisfies $(\epsilon,\delta)$-LDP with some parameters $\epsilon$ and $\delta$.
To prove that the log moment generating function $\alpha_i^{\gM}(\lambda)$ is bounded, we divide it into two parts: 
1) the log moment generating function $\alpha_i^{\gM}(\lambda)$ for the whole mechanism can be bounded by the summation of the log moment generating function of all sub mechanisms $\alpha_i^{\overline{\gM}_i^t}(\lambda)$ from $t=1$ to $T$;
and 
2) the log moment generating function for each sub mechanism is bounded, i.e. $\alpha_i^{\overline{\gM}_i^t}(\lambda)$ is bounded. 
To this end, we provide the following two lemmas to formalize these two parts respectively.

\begin{restatable}[Privacy for composition]{lemma}{lemcomposition}\label{lem:composition}
    For any client $i$ and any $\lambda$, the following holds
    \begin{equation*}
        \alpha_i^{\gM}(\lambda) \le \sum_{t=1}^T\alpha^{\overline{\gM}_i^t}(\lambda).
    \end{equation*}
\end{restatable}

\begin{restatable}[Privacy for sub mechanism]{lemma}{lemsub}\label{lem:sub}
    Suppose that Assumption \ref{ass:bounded-gradient} and \ref{ass:unified} are satisfied. For any client $i$, let $\sigma_p \ge 1$ and let $\gI_b$ be a random minibatch from local dataset $D_i = \{d_{i,j}\}_{j=1}^m$ where each data sample $d_{i,j}$ is chosen independently with probability $q = \frac{b}{m} < \frac{G_A}{16b\sigma_p}$. Then for any positive integer $\lambda \le \frac{2b^2\sigma_p^2}{3G_A^2}\log\frac{G_A}{qb\sigma_p}$, the sub mechanism $\overline{\gM}_i^t$ satisfies
    \[\alpha^{\overline{\gM}_i^t}(\lambda) \le \frac{6 \lambda(\lambda+1)(G_A^2/4 + G_B^2)}{(1-q)m^2 \sigma_p^2} + O\left(\frac{q^3\lambda^3}{\sigma_p^3}\right).\]
\end{restatable}

The detailed proofs for Lemmas~\ref{lem:composition} and \ref{lem:sub} are provided in Appendix~\ref{sec:composition-2} and \ref{sec:sub-3}, respectively.

\paragraph{Proof of Theorem~\ref{thm:privacy}.} Now, we are ready to prove our privacy guarantee in Theorem~\ref{thm:privacy} using Proposition~\ref{prop:tail}, and Lemmas~\ref{lem:composition} and \ref{lem:sub}.

\begin{proof}[Proof of Theorem \ref{thm:privacy}]
Assume for now that $\sigma_p, \lambda$ satisfy the conditions in Lemma \ref{lem:sub}, namely
\begin{equation}\label{eq:req3}
 \lambda \le \frac{2b^2\sigma_p^2}{3G_A^2}\log\frac{G_A}{q\sigma_p b}.
\end{equation} 
By Lemmas \ref{lem:sub} and \ref{lem:composition}, there exists some constant $\hat c$ such that for small enough $q$, the log moment generating function of Algorithm \ref{alg:soteriafl} can be bounded as follows 
\begin{equation*}
    \alpha_i^{\gM}(\lambda) \le \hat c \frac{T\lambda^2(G_A^2/4 + G_B^2)}{m^2\sigma_p^2}, \qquad \forall i\in [n].
\end{equation*}
Combining the above bound and Proposition \ref{prop:tail}, to guarantee Algorithm \ref{alg:soteriafl} to be
$(\epsilon, \delta)$-LDP, it suffices to establish that there exists some $\lambda$ that satisfies \eqref{eq:req3} and the following two conditions:
\begin{align}
      \hat c \frac{T\lambda^2(G_A^2/4 + G_B^2)}{m^2\sigma_p^2} &\le \frac{\lambda\epsilon}{2} \qquad \mbox{or equivalently}\quad \lambda \le \frac{\epsilon m^2\sigma_p^2}{2 \hat c T(G_A^2/4 + G_B^2)}, \label{eq:req1}\\
    \exp\left(-\frac{\lambda\epsilon}{2}\right) & \le \delta \qquad \mbox{or equivalently}\quad \lambda \ge \frac{2}{\epsilon}\log\frac{1}{\delta}. \label{eq:req2}
\end{align}
It is now easy to verify that when $\epsilon =c'q^2 T$ for some constant $c'$, we can satisfy
all these conditions by setting
\[\sigma_p^2 = c\frac{(G_A^2/4+G_B^2)T\log (1/\delta)}{m^2\epsilon^2}\]
for some constant $c$.
\end{proof}

\subsection{Proof of Lemma \ref{lem:composition}}
\label{sec:composition-2}

Before embarking on the proof of Lemma \ref{lem:composition}, we begin with an observation that connects the log moment generation function with the R\'enyi divergence of distributions $\Pr\{\gM(\text{aux},D) = o\}$ and $\Pr\{\gM(\text{aux},D') = o\}$.

\begin{lemma}\label{lem:mgf-renyi}
Denote the R\'enyi divergence between any two distributions $\gP$ and $\gQ$ with parameter $\lambda +1$ as
\[D_{\lambda+1}^{\text{R\'enyi}}(\gP\|\gQ) = \frac{1}{\lambda}\log\E_{\gP}\left(\frac{\dd \gP}{\dd \gQ}\right)^{\lambda}.\]
Then, the log moment generating function has the following form
    \begin{equation}
        \alpha^{\gM}(\lambda; \text{aux}, D, D') = \lambda D^{\text{R\'enyi}}_{\lambda+1}\left(\gM(\text{aux}, D) \| \gM(\text{aux}, D')\right).
    \end{equation}
\end{lemma}

\begin{proof}[Proof of Lemma \ref{lem:mgf-renyi}]
By direct computation, we have
\begin{align*}
    \alpha^{\gM}(\lambda; \text{aux}, D, D')
    &= \log \E_{o\sim \gM(\text{aux}, D)}\left[\exp\left(\lambda\cdot c(o; \gM, \text{aux}, D, D')\right)\right] \\
    &= \log \E_{o\sim \gM(\text{aux}, D)}\left[\exp\left(\lambda\cdot \log \frac{\Pr\{\gM(\text{aux},D) = o\}}{\Pr\{\gM(\text{aux}, D') = o\}}\right)\right] \\
    &= \log \E_{o\sim \gM(\text{aux}, D)}\left[\left( \frac{\Pr\{\gM(\text{aux},D) = o\}}{\Pr\{\gM(\text{aux}, D') = o\}}\right)^{\lambda}\right] \\
    &= \lambda D^{\text{R\'enyi}}_{\lambda+1}\left(\gM(\text{aux}, D) \| \gM(\text{aux}, D')\right).
\end{align*}
\end{proof}

We will also need the following data processing inequality for R\'enyi divergence.

\begin{restatable}[Data processing inequality for R\'enyi divergence~\citep{van2014renyi}]{lemma}{lemdpi}\label{lem:dpi}
    Let $\gP, \gQ$ be two distributions over $\gR$, $\gS:\gR\to\gR$ be a random mapping, and $D^{\text{R\'enyi}}_{\lambda}$ denote the R\'enyi Divergence, then we have
    \begin{equation*}
        D^{\text{R\'enyi}}_{\lambda+1}(\gS(\gP) \| \gS(\gQ)) \le D^{\text{R\'enyi}}_{\lambda+1}(\gP \| \gQ),
    \end{equation*}
    where $\gS(\gP)$ stands for the resulting distribution of applying random mapping $\gS$ on distribution $\gP$.
\end{restatable}

\begin{proof}[Proof of Lemma \ref{lem:composition}]
We divide the proof of Lemma \ref{lem:composition} into two steps: 1) $\alpha_i^{\gM}(\lambda) \le \sum_{t=1}^T\alpha^{\gM_i^t}(\lambda)$; and 2) $\alpha^{\gM_i^t}(\lambda) \le \alpha^{\overline{\gM}_i^t}(\lambda)$. Combining these two steps directly leads to the declared bound, namely
    \begin{align*}
        \alpha_i^{\gM}(\lambda) & \le \sum_{t=1}^T\alpha^{\gM_i^t}(\lambda)  \le \sum_{t=1}^T \alpha^{\overline{\gM}_i^t}(\lambda).
    \end{align*}

The rest of this proof is thus dedicated to establishing the two steps. For simplicity, we use $o_{1:n}^{1:T}$ to denote the outcomes $\{o_i^t\}_{i\in [n], t\in [T]}$, and $\gM_{1:n}^{1:T}$ to denote the mechanisms $\{\gM^t_i\}_{i\in [n], t\in [T]}$.

    \paragraph{Step 1: establishing $\alpha_i^{\gM}(\lambda) \le \sum_{t=1}^T\alpha^{\gM_i^t}(\lambda)$.} For neighboring datasets $D = (D_{-i}, D_i), D' = (D_{-i}, D'_i)$ that differ only on client $i$, we have
    \begin{align*}
        c(o_{1:n}^{1:T}; \gM_{1:n}^{1:T}, o_{1:n}^{1:T-1}, D, D') &= \log \frac{\Pr\{\gM_{1:n}^{1:T}(o_{1:n}^{1:T-1},D) = o_{1:n}^{1:T}\}}{\Pr\{\gM_{1:n}^{1:T}(o_{1:n}^{1:T-1},D') = o_{1:n}^{1:T}\}} \\
      &  = \log \prod_{t=1}^T\prod_{j=1}^n\frac{\Pr\{\gM_j^{t}(o_{1:n}^{1:T-1},D) = o_j^{t}\}}{\Pr\{\gM_j^{t}(o_{1:n}^{1:T-1},D') = o_j^{t}\}} \\
      &  = \log \prod_{t=1}^T\frac{\Pr\{\gM_i^{t}(o_{1:n}^{1:T-1},D) = o_i^{t}\}}{\Pr\{\gM_i^{t}(o_{1:n}^{1:T-1},D') = o_i^{t}\}} \\
      &  = \sum_{t=1}^T\log \frac{\Pr\{\gM_i^{t}(o_{1:n}^{1:T-1},D) = o_i^{t}\}}{\Pr\{\gM_i^{t}(o_{1:n}^{1:T-1},D') = o_i^{t}\}} \\
      &  = \sum_{t=1}^T c(o_{i}^{t}; \gM_{i}^{t}, o_{1:n}^{1:T-1}, D, D').
    \end{align*}
    Here, the second line comes from the fact that the mechanisms of different clients at the same round are independent, and the third line comes from the fact that for any client $j\neq i$, $D_j = D'_j$, and thus $\frac{\Pr\{\gM_j^{t}(o_{1:n}^{1:T-1},D) = o_j^{t}\}}{\Pr\{\gM_j^{t}(o_{1:n}^{1:T-1},D') = o_j^{t}\}} = 1$.
    Then we have
    \begin{align*}
        \E_{o_{1:n}^{1:T}\sim\gM_{1:n}^{1:T}}\left[\exp\left(\lambda c(o_{1:n}^{1:T}; \gM_{1:n}^{1:T}, o_{1:n}^{1:T-1}, D, D') \right)\right] 
        &= \E_{o_{1:n}^{1:T}\sim\gM_{1:n}^{1:T}}\left[\exp\left(\lambda \sum_{t=1}^T c(o_{i}^{t}; \gM_{i}^{t}, o_{1:n}^{1:T-1}, D, D') \right)\right] \\
        &= \E_{o_{1:n}^{1:T}\sim\gM_{1:n}^{1:T}}\left[\prod_{t=1}^T\exp\left(\lambda  c(o_{i}^{t}; \gM_{i}^{t}, o_{1:n}^{1:T-1}, D, D') \right)\right] \\
        &= \prod_{t=1}^T\E_{o_{1:n}^{1:T}\sim\gM_{1:n}^{1:T}}\left[\exp\left(\lambda  c(o_{i}^{t}; \gM_{i}^{t}, o_{1:n}^{1:T-1}, D, D') \right)\right] \\
        &= \prod_{t=1}^T\exp\left(\alpha^{\gM_i^t}(\lambda; o_{1:n}^{1:T-1},D,D')\right) \\
        &= \exp\left(\sum_{t=1}^T\alpha^{\gM_i^t}(\lambda; o_{1:n}^{1:T-1},D,D')\right).
    \end{align*}
    Taking logarithm on both sides and maximizing over $o_{1:n}^{1:T-1}, D, D'$, we can show that
    \[\alpha_i^{\gM}(\lambda) \le \sum_{t=1}^T\alpha^{\gM_i^t}(\lambda).\]
    
    \paragraph{Step 2: establishing $\alpha^{\gM_i^t}(\lambda) \le \alpha^{\overline{\gM}_i^t}(\lambda)$.}  This step follows directly from Lemma \ref{lem:dpi}. Namely, for fixed $i$ and $t$, we can compute
    \begin{align*}
        \alpha^{\gM_i^t}(\lambda; o_{1:n}^{1:T-1},D,D') &= \lambda D^{\text{R\'enyi}}_{\lambda+1}\left(\gM_i^t(o_{1:n}^{1:T-1}, D) \| \gM_i^t(o_{1:n}^{1:T-1}, D')\right) \\
       & = \lambda D^{\text{R\'enyi}}_{\lambda+1}\left((\gA\circ \overline{\gM}_i^t)(o_{1:n}^{1:T-1}, D) \| (\gA\circ \overline{\gM}_i^t)(o_{1:n}^{1:T-1}, D')\right) \\
       & \le \lambda D^{\text{R\'enyi}}_{\lambda+1}\left( \overline{\gM}_i^t(o_{1:n}^{1:T-1}, D) \|  \overline{\gM}_i^t(o_{1:n}^{1:T-1}, D')\right) \\
       & = \alpha^{\overline{\gM}_i^t}(\lambda; o_{1:n}^{1:T-1},D,D').
    \end{align*}
    Then, taking the maximum over $o_{1:n}^{1:T-1}, D, D'$, we have
    \[\alpha^{\gM_i^t}(\lambda) \le \alpha^{\overline{\gM}_i^t}(\lambda).\]
\end{proof}

\subsection{Proof of Lemma \ref{lem:sub}}
\label{sec:sub-3}

It is worth noting that the proof does not requires $f(\vx; d_{i,j})$ to be a function with respect to the data sample $d_{i,j}$ at point $\vx$, it can be any function related to $d_{i,j}$, for example, $\phi_{i,j}$ in Assumption \ref{ass:unified}.
Inspired by \citep{wang2017differentially}, we decompose the gradient estimator into two parts and bound the privacy respectively. 
Now we provide the detailed proofs below.
\begin{proof}[Proof of Lemma \ref{lem:sub}]
    From Assumption \ref{ass:unified}, we first write out and decouple the sub-mechanism $\overline{\gM}_i^t$ (corresponding to the Gaussian perturbation) as
    \begin{equation}\label{eq:mec-decomposition}
        \frac{1}{b} \sum_{j\in \gI_b} \varphi_{i,j}^t + \frac{1}{m} \sum_{j=1}^m \psi_{i,j}^t + \vxi_i^t = \left(\frac{1}{b} \sum_{j\in \gI_b} \varphi_{i,j}^t + \vxi_{i,1}^t\right) + \left(\frac{1}{m} \sum_{j=1}^m \psi_{i,j}^t + \vxi_{i,2}^t\right),
    \end{equation}
    where $\vxi_i^t$ is generated from $\gN(0,\sigma_p^2\mI)$ and $\vxi_{i,1}^t,\vxi_{2,1}^t$ are generated from $\gN(0,\frac{2\sigma_p^2}{3}\mI)$, $\gN(0,\frac{\sigma_p^2}{3}\mI)$ independently. Now, $\overline{\gM}_i^t$ can be viewed as a composition of two mechanisms $\overline{\gM}_{i,1}^t$ and $\overline{\gM}_{i,2}^t$, where $\overline{\gM}_{i,1}^t$ denote the first term and $\overline{\gM}_{i,2}^t$ denote the second term in the right-hand-side (RHS) of \eqref{eq:mec-decomposition}. From \citep[Theorem 2.1]{abadi2016deep}, we have
    \begin{align}
        \alpha^{\overline{\gM}_i^t}(\lambda) \le \alpha^{\overline{\gM}_{i,1}^t}(\lambda) + \alpha^{\overline{\gM}_{i,2}^t}(\lambda). \label{eq:alpha-first}
    \end{align}
    For the first term of \eqref{eq:alpha-first}, according to \citep[Lemma 3]{abadi2016deep}, we have
    \begin{align}
        \alpha^{\overline{\gM}_{i,1}^t}(\lambda) \le \frac{3 \lambda(\lambda+1)G_A^2}{2(1-q)m^2 \sigma_p^2} + O\left(\frac{q^3\lambda^3}{\sigma_p^3}\right), \label{eq:alpha-1}
    \end{align}
    for $q = \frac bm < \frac{G_A}{16\sigma_p b}$ and any positive integer $\lambda \le \frac{2b^2\sigma_p^2}{3G_A^2}\log\frac{G_A}{q\sigma_p b}$, where we set $\sigma^2$ in  \citep[Lemma 3]{abadi2016deep} to be $\frac{2b^2 \sigma_p^2}{3 G_A^2}$.
    
    For the second term of \eqref{eq:alpha-first}, according to Lemma~\ref{lem:mgf-renyi} (the relationship between R\'enyi divergence and the moment generating function), we have
    \begin{align}
        \alpha^{\overline{\gM}_{i,2}^t}(\lambda) = \lambda D_{\lambda+1}^{\text{R\'enyi}}(\gP\|\gQ), \label{eq:alpha-2}
    \end{align}
    where $\gP = \frac{1}{m} \sum_{j=1}^m \psi_{i,j}^t + \gN(0,\frac{2\sigma_p^2}{3}\mI)$ and $\gQ = \frac{1}{m} \sum_{j=1}^m (\psi_{i,j}^t)' + \gN(0,\frac{\sigma_p^2}{3}\mI)$. Here, $\{\psi_{i,j}^t,\; j\in [m]\}$ contains the functions corresponding to the data in dataset $D$, and $\{(\psi_{i,j}^t)',\; j\in [m]\}$ contains the functions corresponding to the data in dataset $D'$. We note that all functions except one in $\{\psi_{i,j}^t,\; j\in [m]\}$ and $\{(\psi_{i,j}^t)',\; j\in [m]\}$ are the same, since the datasets $D$ and $D'$ only differ by one element. According to \citep[Lemma 17]{bun2016concentrated}, we have 
    \begin{align}
        \lambda D_{\lambda+1}^{\text{R\'enyi}}(\gP\|\gQ) = \frac{3\lambda(\lambda+1) \norm{\frac{1}{m} \sum_{j=1}^m \psi_{i,j}^t - \frac{1}{m} \sum_{j=1}^m (\psi_{i,j}^t)'}^2}{2\sigma_p^2} \le \frac{6\lambda(\lambda+1)G_B^2}{m^2\sigma_p^2}. \label{eq:alpha-last}
    \end{align}
    The proof is finished by combining \eqref{eq:alpha-first}--\eqref{eq:alpha-last}.
\end{proof}

\section{Proof of Theorem~\ref{thm:utility}}
\label{sec:proof-utility}
We now provide the detailed proofs for our unified Theorem~\ref{thm:utility}.
First, according to the update rule $\xtn = \xt - \etat \vvt$ (Line~\ref{line:update-soteriafl} in Algorithm~\ref{alg:soteriafl}) and the smoothness assumption (Assumption~\ref{ass:smoothness}), we have 
\begin{align}
	\Et{f(\xtn)} \leq \EtB{f(\xt) - \etat \inner{\nabla f(\xt)}{\vvt} + \frac{L\etat^2}{2}\ns{\vvt}}, \label{eq:first}
\end{align}
where $\E_t$ takes the expectation conditioned on all history before round $t$.
To begin, we show that $\vvt$ is unbiased as follows:
\begin{align}
	\Et{\vvt} 
	= \EtB{\st + \frac{1}{n}\sumn\vit} 
	&= \EtB{\frac{1}{n}\sumn \sit + \frac{1}{n}\sumn \cit(\git - \sit)} \notag\\
	&= \EtB{\frac{1}{n}\sumn \git} 
	= \EtB{\frac{1}{n}\sumn (\tgit + \xit ) }  \label{eq:usecompression} \\
	&= \EtB{\frac{1}{n}\sumn \tgit}   \label{eq:usenoise} \\
	&= \EtB{\frac{1}{n}\sumn \nabla f_i(\xt)} = \nabla f(\xt), \label{eq:useunbiase}   
\end{align}
where \eqref{eq:usecompression} follows from \eqref{eq:comp}, \eqref{eq:usenoise} holds due to $\vxi_i^{t} \sim \gN(\bm{0},\sigma_p^2\mI)$, and \eqref{eq:useunbiase} is due to $\E_t[\tilde{\vg}_i^t] = \nabla f_i(\vx^t)$ from Assumption~\ref{ass:unified}.

Plugging \eqref{eq:usenoise} into \eqref{eq:first}, we get 
\begin{align}
	\Et{f(\xtn)} \leq \EtB{f(\xt) - \etat \ns{\nabla f(\xt)} + \frac{L\etat^2}{2}\ns{\vvt}}. \label{eq:variance-remained}
\end{align}
We then bound the last term $\Et{\ns{\vvt}}$ in the follow lemma, whose proof is provided in Appendix~\ref{sec:proof_vvt}.

\begin{lemma}\label{lem:vvt}
	Suppose that $\vvt$ is defined and computed in Algorithm~\ref{alg:soteriafl}, we have 
	\begin{align}
		\Et{\ns{\vvt}} 
		&\leq \EtB{\frac{(1+\omega)}{n^2}\sumn \ns{\tgit - \nabla f_i(\xt)}} +\frac{\omega}{n^2}\sumn \ns{\nabla f_i(\xt)-\sit} +\ns{\nabla f(\xt)} + \frac{(1+\omega)d \sigp}{n}. \label{eq:vvt}
	\end{align}
\end{lemma}

To continue, we need to bound the first two terms in \eqref{eq:vvt}.
The first term can be controlled via \eqref{ass:var-g} of Assumption~\ref{ass:unified}.
Now we show that the second term will shrink in the following lemma, whose proof is provided in Appendix~\ref{sec:proof_shift}.

\begin{lemma}\label{lem:shift}
	Suppose that Assumption~\ref{ass:smoothness} holds and the shift $\sitn$ is defined and computed in Algorithm~\ref{alg:soteriafl}. Then letting $\gamma_t=\sqrt{\frac{1+2\omega}{2(1+\omega)^3}}$, we have 
	\begin{align}
		\EtB{\frac{1}{n}\sumn \ns{\nabla f_i(\xtn)-\sitn}} 
		&\leq \E_t\bigg[\Big(1-\frac{1}{2(1+\omega)}\Big)\frac{1}{n}\sumn \nsB{\nabla f_i(\xt)-\sit}  \notag\\ 
		&\qquad\quad + \frac{1}{(1+\omega)n}\sumn \ns{\tgit-\nabla f_i(\xt)} \notag\\
		&\qquad\quad + 2(1+\omega)L^2\nsB{\xtn- \xt}   
		+ \frac{d\sigp}{1+\omega} \Big)\bigg].  \label{eq:shift}
	\end{align}
\end{lemma}

To facilitate presentation, let us introduce the short-hand notation $\gSt:=\frac{1}{n}\cst$. Then we define the following potential function
\begin{align}\label{eq:def-potential}
	\Phi_t := f(\xt) - f^* + \alpha L\Deltat + \frac{\beta}{L}\gSt,
\end{align}
for some $\alpha\geq 0, \beta \geq 0$. With the help of Lemmas~\ref{lem:vvt} and \ref{lem:shift},
we show that this potential function decreases in each round in the following lemma, whose proof is provided in Appendix~\ref{sec:proof_potential}.
\begin{lemma}\label{lem:potential}
	Under Assumptions~\ref{ass:smoothness} and \ref{ass:unified}, if we choose the stepsize as
	\begin{align*}
	\eta_t \equiv \eta \leq \min\left\{ \frac{1}{(1+2\alpha C_4 + 4\beta(1+\omega) + 2\alpha C_3/\eta^2)L}, 
	\frac{\sqrt{\beta n}}{\sqrt{1+2\alpha C_4 + 4\beta(1+\omega)}(1+\omega)L}\right\},
	\end{align*}
	where $\alpha = \frac{3\beta C_1}{2(1+\omega)\theta L^2}$,  $\forall \beta>0$, and the shift stepsize as $\gammat\equiv \sqrt{\frac{1+2\omega}{2(1+\omega)^3}}$, then we have for any round $t\geq 0$,
	\begin{align}
		\Et{\Phi_{t+1}} \leq \Phi_t - \frac{\etat}{2}\ns{\nabla f(\xt)} + \frac{3\beta}{2(1+\omega)L}(C_2 + d\sigp). \label{eq:potential}
	\end{align}
\end{lemma}

Given Lemma~\ref{lem:potential} and Theorem~\ref{thm:privacy}, now we are ready to prove Theorem~\ref{thm:utility} regarding the utility and communication complexity for \soteriafl.

\begin{proof}[Proof of Theorem~\ref{thm:utility}]
	First, we sum up \eqref{eq:potential} (Lemma~\ref{lem:potential}) from round $t=0$ to $T-1$,
	\begin{align}
		\sum_{t=0}^{T-1} \frac{\etat}{2}\E\ns{\nabla f(\vx^t)} 
		&\le \Phi_0
		+ \frac{3\beta}{2(1+\omega)L}\left(C_2 + d\sigp\right)T. \label{eq:sum-T}
	\end{align}
	Then by choosing the stepsize $\etat$ as in Lemma~\ref{lem:potential} and the privacy variance $\sigp=\frac{c(G_A^2/4 +G_B^2)T\log(1/\delta))}{m^2\epsilon^2}$ according to Theorem~\ref{thm:privacy}, we obtain
	\begin{align}
	\frac{1}{T}\sum_{t=0}^{T-1} \E\ns{\nabla f(\vx^t)} 
	&\le \frac{2\Phi_0}{\eta T} 
	+ \frac{3\beta}{(1+\omega)L\eta}\left(C_2 + \frac{c(G_A^2/4 +G_B^2)dT\log ({1}/{\delta})}{m^2\epsilon^2}\right), \label{eq:plug-sigp}
	\end{align}
	and \soteriafl (Algorithm~\ref{alg:soteriafl}) satisfies $(\epsilon,\delta)$-LDP.
	
	Finally, the total number of communication rounds $T$ in \eqref{eq:choice-T} comes from the following relations in RHS of \eqref{eq:plug-sigp}
	\begin{align}
		\frac{2\Phi_0}{\eta T} &\leq \frac{3\beta}{(1+\omega)L\eta} \frac{c(G_A^2/4 +G_B^2)dT\log ({1}/{\delta})}{m^2\epsilon^2}, \notag\\
		C_2&\leq \frac{c(G_A^2/4 +G_B^2)dT\log ({1}/{\delta})}{m^2\epsilon^2}. \notag
	\end{align}
	The utility guarantee \eqref{eq:utility} directly follows from \eqref{eq:plug-sigp} by choosing $T$ as in \eqref{eq:choice-T}.
\end{proof}

\subsection{Proof of Lemma~\ref{lem:vvt}}
\label{sec:proof_vvt}
 
	According to the definition of $\vvt$, we have
	\begin{align}
		\Et{\ns{\vvt}} 
		&= \EtB{\nsB{\frac{1}{n}\sumn \sit + \frac{1}{n}\sumn \cit(\git - \sit)}} \notag\\
		&= \EtB{\nsB{\frac{1}{n}\sumn \cit(\git - \sit) - \frac{1}{n}\sumn (\git - \sit) + \frac{1}{n}\sumn \git}} \notag\\
		&\leq \EtB{\frac{\omega}{n^2}\sumn\ns{\git - \sit}} + \EtB{\nsB{\frac{1}{n}\sumn \git}} , \label{eq:usecompression-var}
	\end{align}
where the last line is due to the definition of the compression operator \eqref{eq:comp}. To continue, we bound each term in \eqref{eq:usecompression-var} respectively. 
\begin{itemize}
\item For the first term, we have
	\begin{align}	
	\EtB{\frac{\omega}{n^2}\sumn\ns{\git - \sit}}	&= \EtB{\frac{\omega}{n^2}\sumn\ns{\tgit - \sit + \xit}}  \notag\\
		&= \EtB{\frac{\omega}{n^2}\sumn(\ns{\tgit - \sit} + d\sigp)}  \notag\\
		&= \EtB{\frac{\omega}{n^2}\sumn \ns{\tgit - \sit}}  + \frac{\omega d \sigp}{n} \notag\\
		&= \EtB{\frac{\omega}{n^2}\sumn \ns{\tgit - \nabla f_i(\xt) + \nabla f_i(\xt)-\sit}}  +  \frac{\omega d \sigp}{n} \notag\\
		&= \EtB{\frac{\omega}{n^2}\sumn \ns{\tgit - \nabla f_i(\xt)}} +\frac{\omega}{n^2}\sumn \ns{\nabla f_i(\xt)-\sit} +  \frac{\omega d \sigp}{n}, \label{eq:first_term_var}
	\end{align}
	where the last line is due to $\E_t[\tilde{\vg}_i^t] = \nabla f_i(\vx^t)$ from Assumption~\ref{ass:unified}. 
	\item Similarly, for the second term, we have
	\begin{align}	
	 \EtB{\nsB{\frac{1}{n}\sumn \git}}	&=   \EtB{\nsB{\frac{1}{n}\sumn (\tgit+\xit)}} \notag\\
		&=  \EtB{\nsB{\frac{1}{n}\sumn \tgit} + \frac{d \sigp}{n}} \notag\\
		&=  \EtB{\nsB{\frac{1}{n}\sumn (\tgit -\nabla f_i(\xt) + \nabla f_i(\xt))}} + \frac{d \sigp}{n} \notag\\
		&= \EtB{\frac{1}{n^2}\sumn \ns{\tgit - \nabla f_i(\xt)}} 
	+\ns{\nabla f(\xt)} + \frac{d \sigp}{n}. \label{eq:second_term_var}
	\end{align}
\end{itemize}	
The proof is completed by plugging \eqref{eq:first_term_var} and \eqref{eq:second_term_var} into \eqref{eq:usecompression-var}.

\subsection{Proof of Lemma~\ref{lem:shift}}
\label{sec:proof_shift}

	According to the shift update (Line~\ref{line:shift-soteriafl} in Algorithm~\ref{alg:soteriafl}), we have 
	\begin{align}
		&\EtB{\frac{1}{n}\sumn \ns{\nabla f_i(\xtn)-\sitn}} \notag\\
		&= \EtB{\frac{1}{n}\sumn \nsB{\nabla f_i(\xtn)-\sit - \gammat \cit(\git-\sit)}} \notag\\
		&= \EtB{\frac{1}{n}\sumn \nsB{\nabla f_i(\xtn)- \nabla f_i(\xt) +\nabla f_i(\xt)-\sit - \gammat \cit(\git-\sit)}} \notag\\
		&\leq \E_t\bigg[\frac{1}{n}\sumn \Big( (1+\frac{1}{\betat})\nsB{\nabla f_i(\xtn)- \nabla f_i(\xt)} + (1+\betat)\nsB{\nabla f_i(\xt)-\sit - \gammat \cit(\git-\sit)} \Big)\bigg] \label{eq:useyoung}\\
		&\leq \E_t\bigg[(1+\frac{1}{\betat})L^2\nsB{\xtn- \xt} + (1+\betat)\frac{1}{n}\sumn \nsB{\nabla f_i(\xt)-\sit - \gammat \cit(\git-\sit)}\bigg], \label{eq:usesmooth}
	\end{align}
	where \eqref{eq:useyoung} uses Young's inequality with any $\betat>0$ (its choice will be specified momentarily), and \eqref{eq:usesmooth} uses Assumption~\ref{ass:smoothness}. The second term of \eqref{eq:usesmooth} can be further bounded as follows:
	\begin{align}
	    &\EtB{\frac{1}{n}\sumn \nsB{\nabla f_i(\xt)-\sit - \gammat \cit(\git-\sit)}} \notag\\
		&= \E_t\bigg[\frac{1}{n}\sumn \Big((1-2\gammat)\nsB{\nabla f_i(\xt)-\sit} +\gammat^2 \ns{\cit(\git-\sit)} \Big)\bigg] \notag\\
		&\overset{\eqref{eq:comp}}{\le} \E_t\bigg[\frac{1}{n}\sumn \Big((1-2\gammat)\nsB{\nabla f_i(\xt)-\sit} +\gammat^2 (1+\omega) \ns{\git-\sit} \Big)\bigg] \notag\\
	%	&= \E_t\bigg[\frac{1}{n}\sumn \Big((1-2\gammat)\nsB{\nabla f_i(\xt)-\sit} +\gammat^2 (1+\omega) \ns{\tgit-\sit +\xit} \Big)\bigg] \notag\\
	%	&= \E_t\bigg[\frac{1}{n}\sumn \Big((1-2\gammat)\nsB{\nabla f_i(\xt)-\sit} +\gammat^2(1+\omega) \ns{\tgit-\sit} + \gammat^2(1+\omega)d\sigp \Big)\bigg] \notag\\
		&= \E_t\bigg[\frac{1}{n}\sumn \Big(\big(1-2\gammat + \gammat^2(1+\omega)\big)\nsB{\nabla f_i(\xt)-\sit} +\gammat^2(1+\omega) \ns{\tgit-\nabla f_i(\xt)} + \gammat^2(1+\omega)d\sigp \Big)\bigg], \label{eq:second-term}
	\end{align}
	where the first equality follows from
	$$ \EtB{\left\langle \nabla f_i(\xt)-\sit ,\, \cit(\git-\sit) \right\rangle} = \EtB{ \nsB{\nabla f_i(\xt)-\sit}}, $$
	and the last line follows from \eqref{eq:first_term_var}.
	The proof is completed by plugging \eqref{eq:second-term} into \eqref{eq:usesmooth} and choosing $\betat=\frac{1}{1+2\omega}$ and
	$\gammat=\sqrt{\frac{1+2\omega}{2(1+\omega)^3}}$.

\subsection{Proof of Lemma~\ref{lem:potential}}
\label{sec:proof_potential}

Recalling $\gSt:=\frac{1}{n}\cst$, $\gSt$ can be recursively bounded by Lemma~\ref{lem:shift} as
\begin{align}
\Et{\gStn} & \leq \E_t\bigg[\big(1-\frac{1}{2(1+\omega)}\big)\gSt + \frac{1}{(1+\omega)n}\sumn \ns{\tgit-\nabla f_i(\xt)} \notag\\
			&\qquad\qquad+  2(1+\omega)L^2\nsB{\xtn- \xt} + \frac{d\sigp}{1+\omega} \bigg] .
\end{align}
Note that the second term can be bounded by \eqref{ass:var-g} of Assumption~\ref{ass:unified}, namely
\begin{align*}
\E_t\Big[\frac{1}{n}\sum_{i=1}^n \ns{\tilde{\vg}_i^t - \nabla f_i(\vx^t)}\Big] \leq C_1 \Delta^t + C_2, %\label{ass:var-g-appendix}
\end{align*}
leading to
\begin{align*}
\Et{\gStn} & \leq \E_t\bigg[\big(1-\frac{1}{2(1+\omega)}\big)\gSt + \frac{C_1 \Delta^t + C_2}{(1+\omega)n} +  2(1+\omega)L^2\nsB{\xtn- \xt} + \frac{d\sigp}{1+\omega} \bigg] .
\end{align*}
Combined with \eqref{eq:variance-remained}, we can bound the potential function \eqref{eq:def-potential} as
	\begin{align}
		\Et{\Phi_{t+1}} 
		&:=\EtB{f(\xtn) - f^* + \alpha L\Deltatn + \frac{\beta}{L}\gStn} \notag\\
		& \leq
			\E_t\bigg[f(\xt) - f^* - \etat \ns{\nabla f(\xt)} + \frac{L\etat^2}{2}\ns{\vvt} +\alpha L\Deltatn \notag\\
			&\qquad\qquad  + \frac{\beta}{L}\Big(\big(1-\frac{1}{2(1+\omega)}\big)\gSt + \frac{C_1\Deltat + C_2}{1+\omega} +  2(1+\omega)L^2\nsB{\xtn- \xt} + \frac{d\sigp}{1+\omega} \Big)\bigg] \notag\\
		&\overset{\eqref{ass:var-Delta}}{\leq} 
			\E_t\bigg[f(\xt) - f^* - \etat \ns{\nabla f(\xt)} + \frac{L\etat^2}{2}\ns{\vvt} \notag\\
			&\qquad\qquad +\alpha L \Big( (1-\theta)\Deltat + C_3\ns{\nabla f(\xt)} + C_4\ns{\xtn-\xt}\Big) \notag\\
			&\qquad\qquad  + \frac{\beta}{L}\Big(\big(1-\frac{1}{2(1+\omega)}\big)\gSt + \frac{C_1\Deltat + C_2}{1+\omega} +  2(1+\omega)L^2\nsB{\xtn- \xt} + \frac{d\sigp}{1+\omega} \Big)\bigg] \notag\\
		&= \E_t\bigg[f(\xt) - f^* - \etat \ns{\nabla f(\xt)} 
			+  \Big(\frac{1}{2}+\alpha C_4  + 2\beta(1+\omega) \Big)L\etat^2 \ns{\vvt} \notag\\
			&\qquad\qquad +\alpha L \Big( (1-\theta)\Deltat + C_3\ns{\nabla f(\xt)} \Big) \notag\\
			&\qquad\qquad  + \frac{\beta}{L}\Big(\big(1-\frac{1}{2(1+\omega)}\big)\gSt + \frac{C_1\Deltat + C_2}{1+\omega} + \frac{d\sigp}{1+\omega} \Big)\bigg], \label{eq:intermediate_potential}
			\end{align}
	where the last line follows from the update rule $\xtn = \xt -\etat \vvt$ (Line~\ref{line:update-soteriafl} of Algorithm~\ref{alg:soteriafl}). To continue, we invoke Lemma~\ref{lem:vvt}, which gives
\begin{align*}
\Et{\ns{\vvt}} & \leq \EtB{\frac{(1+\omega)}{n^2}\sumn \ns{\tgit - \nabla f_i(\xt)} + \frac{\omega}{n} \gSt  + \ns{\nabla f(\xt)} + \frac{(1+\omega)d \sigp}{n}} \\
& \leq \EtB{ \frac{(1+\omega)}{n}C_1\Deltat + \frac{\omega}{n} \gSt   + \ns{\nabla f(\xt)} + \frac{(1+\omega)(C_2+d \sigp)}{n}  },
\end{align*}
where the second line uses again \eqref{ass:var-g} of Assumption~\ref{ass:unified}. Plugging this back into \eqref{eq:intermediate_potential}, we arrive at
\begin{align}
	\Et{\Phi_{t+1}} & \leq
			f(\xt) - f^* + \bigg[ \alpha (1-\theta) + \frac{\beta C_1}{(1+\omega)L^2} + \Big(\frac{1}{2}+\alpha C_4 + 2\beta(1+\omega)\Big)\frac{(1+\omega)C_1\etat^2}{n}\bigg]L\Deltat \notag\\
			&\qquad\qquad + \bigg[ \beta\big(1-\frac{1}{2(1+\omega)}\big) + \Big(\frac{1}{2}+\alpha C_4 + 2\beta(1+\omega)\Big)\frac{\omega L^2\etat^2}{n} \bigg]\frac{\gSt}{L} \notag\\
			&\qquad\qquad - \bigg[ \etat - \alpha L C_3 - \Big(\frac{1}{2}+\alpha C_4 + 2\beta(1+\omega)\Big)L\etat^2 \bigg]\ns{\nabla f(\xt)} \notag\\
			&\qquad\qquad + \bigg[ \frac{\beta}{(1+\omega)L} + \Big(\frac{1}{2}+\alpha C_4 + 2\beta(1+\omega)\Big)\frac{(1+\omega)L\etat^2}{n}\bigg] (C_2+d\sigp). \label{eq:final-potential}
	\end{align}

	Now we choose the appropriate parameters satisfying 
	\begin{align}
		\alpha (1-\theta) + \frac{\beta C_1}{(1+\omega)L^2} + \Big(\frac{1}{2}+\alpha C_4 + 2\beta(1+\omega)\Big)\frac{(1+\omega)C_1\etat^2}{n} &\leq \alpha, \label{eq:alpha-potential}\\
		\beta\big(1-\frac{1}{2(1+\omega)}\big) + \Big(\frac{1}{2}+\alpha C_4 + 2\beta(1+\omega)\Big)\frac{\omega L^2\etat^2}{n} &\leq \beta, \label{eq:beta-potential}
	\end{align}
	so that the RHS of \eqref{eq:final-potential} can lead to the potential function $\Phi_t := f(\xt) - f^* + \alpha L\Deltat + \frac{\beta}{L}\gSt$.
	 It is not hard to verify that the following choice of $\alpha,\beta,\eta_t$ satisfy \eqref{eq:alpha-potential} and \eqref{eq:beta-potential}:
	 \begin{align}
	 	&\alpha \geq \frac{3\beta C_1}{2(1+\omega)L^2\theta}, \qquad \forall \beta>0, \label{eq:alpha-beta}\\
	 	&\etat \equiv \eta \leq \frac{\sqrt{\beta n}}{\sqrt{1+2\alpha C_4 + 4\beta(1+\omega)}(1+\omega)L}.\label{eq:eta-1}
	 \end{align}
	 Note that \eqref{eq:eta-1} implies 
	 \begin{align}
	 	\Big(\frac{1}{2}+\alpha C_4 + 2\beta(1+\omega)\Big)\frac{(1+\omega)\etat^2}{n} \leq \frac{\beta}{2(1+\omega)L^2}. \label{eq:eta-1-result}
	 \end{align}
	 If we further choose the stepsize 
	 \begin{align}
	 &\etat \equiv \eta \leq \frac{1}{(1+2\alpha C_4 + 4\beta(1+\omega) + 2\alpha C_3/\eta^2)L}, \label{eq:eta-2}
	 \end{align}
	 then the proof is finished by combining \eqref{eq:final-potential}--\eqref{eq:eta-2} since \eqref{eq:final-potential} simplifies to 
	 \begin{align*}
	 	\Et{\Phi_{t+1}} \leq \Phi_t - \frac{\etat}{2}\ns{\nabla f(\xt)} + \frac{3\beta}{2(1+\omega)L}(C_2 + d\sigp).
	 \end{align*}

\section{Proof of Theorem \ref{thm:cdp-sgd}}
\label{sec:proof-cdp-sgd}
We now give the detailed proof for Theorem \ref{thm:cdp-sgd}. We first show the privacy guarantee of \CDP and then derive the utility guarantee.

\subsection{Privacy guarantee of \CDP}

\begin{restatable}[Privacy guarantee for \CDP]{theorem}{thmprivacy}\label{thm:privacy-cdpsgd}
    Suppose Assumption \ref{ass:bounded-gradient} holds. There exist constants $c'$ and $c$ so that given the sampling probability $q = b/m$ and the number of steps $T$, for any $\epsilon < c'q^2T$ and $\delta \in (0,1)$, \CDP (Algorithm~\ref{alg:cdp-sgd}) is $(\epsilon, \delta)$-LDP if we choose
    \[\sigma_p^2 = c\frac{G^2T\log (1/\delta)}{m^2\epsilon^2}.\]
\end{restatable}

The proof of Theorem \ref{thm:privacy-cdpsgd} is very similar to the proof of Theorem \ref{thm:privacy}. Thus here we just point out some differences between the proof of Theorem \ref{thm:privacy-cdpsgd} and \ref{thm:privacy}.

\paragraph{Sub-mechanisms.} Similar to Theorem \ref{thm:privacy}, we define the sub-mechanisms in the following way. We assume that there are $n\times T$ sub-mechanisms $\{\gM_i^t\}_{i\in [n], t\le T}$ in $\gM$, where $\gM_i^t$ corresponds to the mechanism for client $i$ in round $t$. We further let $\gM_i^t := \gA\circ \overline{\gM}_i^t$ be the composition of mechanism $\overline{\gM}_i^t$ and the mechanism $\gA$. Here, $\gA:\gR\to\gR$ is a random mechanism that maps an outcome to another outcome, and $\overline{\gM}_i^t$ is possibly an adaptive mechanism that takes the input of all the outputs before time $t$, i.e. $o_i^{s}$ for all $s < t$ and $i\in [n]$. We assume that given all the previous outcomes $o_i^{s}$ for $s <t$, the random mechanisms $\overline{\gM}_i^t$ for all $i\in [n]$ are independent w.r.t. each other (this is satisfied in \CDP).
In \CDP (Algorithm~\ref{alg:cdp-sgd}), $\gA$ corresponds to the compression operator, and $\overline{\gM}_i^t$ corresponds the Gaussian perturbation. The difference between the sub-mechanisms for \soteriafl and \CDP is the presence of the shift. However as the shift is known to the central server, we can omit that during the analysis of privacy.

\paragraph{Privacy for composition (Lemma~\ref{lem:composition}).} Here \CDP can use  exactly the same previous Lemma~\ref{lem:composition} since the relationship between the final mechanism and the sub-mechanism does not change.

\paragraph{Privacy for sub-mechanisms (Lemma~\ref{lem:sub}).} The privacy guarantee for sub-mechanisms of Theorem \ref{thm:privacy-cdpsgd} is simpler than that for Theorem \ref{thm:privacy}, since we can simply apply Lemma 3 of \citep{abadi2016deep} to obtain the following bound
\[\alpha^{\overline{\gM}_i^t}(\lambda) \le \frac{\lambda(\lambda+1)G^2}{(1-q)m^2 \sigma_p^2} + O\left(\frac{q^3\lambda^3}{\sigma_p^3}\right).\]

\subsection{Utility guarantee of \CDP}
To prove the convergence result, we first give the following lemma providing the mean and variance of the stochastic gradient $ \tgit = \frac{1}{b}\sum_{j\in \gI_b}\nabla f_{i,j}(\vx^t)$ (Line~\ref{line:sgd} in Algorithm~\ref{alg:cdp-sgd}).

\begin{lemma}[Variance]\label{lem:variance}
    Under Assumption \ref{ass:bounded-gradient}, for any client $i$, the stochastic gradient estimator $\tgit = \frac{1}{b}\sum_{j\in \gI_b}\nabla f_{i,j}(\vx^t)$ is unbiased, i.e.
    \[\E_t \Bigg[ \frac{1}{b}\sum_{j\in \gI_b}\nabla f_{i,j}(\xt) \Bigg] = \nabla f_i(\xt), \]
    where $\E_t$ takes the expectation conditioned on all history before round $t$. Also, we have   
    \begin{equation}\label{eq:variance}
        \EtB{\nsB{\frac{1}{b}\sum_{j\in \gI_b}\nabla f_{i,j}(\xt) - \nabla f_i(\xt)}} \le \frac{(1-q)G^2}{b},
    \end{equation}
    where $q=b/m$.
\end{lemma}

\begin{proof}
We first show that the estimator is unbiased. Define $m$ independent Bernoulli random variables $X_{i,j}$, where $\Pr\{X_{i,j} = 1\} = q = \frac{b}{m}$. Then,
\begin{align*}
    \E_{t} \Bigg[ \frac{1}{b}\sum_{j\in \gI_b}\nabla f_{i,j}(\xt) \Bigg] 
    = \E_t \Bigg[ \frac{1}{b}\sum_{j=1}^m X_{i,j} \nabla f_{i,j}(\xt) \Bigg] 
    = \frac{1}{m} \sum_{j=1}^m \nabla f_{i,j}(\xt) 
    = \nabla f_i(\xt).
\end{align*}
Moving onto the variance bound, we have
\begin{align}
    \EtB{\nsB{\frac{1}{b}\sum_{j\in \gI_b}\nabla f_{i,j}(\xt) - \nabla f_i(\xt)}}
    &= \EtB{\nsB{\frac{1}{b}\sum_{j=1}^m X_{i,j} \nabla f_{i,j}(\xt) - \nabla f_i(\xt)}} \nonumber \\
    & = \EtB{\nsB{\sum_{j=1}^m\left( \frac{1}{b}X_{i,j} \nabla f_{i,j}(\xt) - \frac{1}{m}\nabla f_{i,j}(\xt)\right)}} \nonumber \\
    & = \EtB{\sum_{j=1}^m \nsB{ \frac{1}{b} (X_{i,j} -q) \nabla f_{i,j}(\xt)  }} \label{eq:variance-xi-indep} \\
    & = \sum_{j=1}^m \frac{(1-q)q}{b^2}\norm{\nabla f_{i,j}(\xt)}^2
    \le \frac{(1-q)G^2}{b},\nonumber
\end{align}
where \eqref{eq:variance-xi-indep} comes from the fact that random variables $X_{i,j}$ are independent, and the last line follows from the variance of Bernoulli random variables as well as Assumption \ref{ass:bounded-gradient}.
\end{proof}

With the help of the above lemma, we now prove Theorem \ref{thm:cdp-sgd}.%, which is reproduced below for convenience.

%\thmnonconvex*

\begin{proof}[Proof of Theorem~\ref{thm:cdp-sgd}]

First, from the smoothness Assumption \ref{ass:smoothness}, we have
\begin{align*}
    f(\vx^{t+1}) \le f(\vx^t) - \etat \dotp{ \nabla f(\vx^t)}{\vv^t} + \frac{L\etat^2}{2}\norm{\vv^t}^2.
\end{align*}
Taking the expectation on both sides of the above inequality, we have (note that we choose constant stepsize $\etat \equiv \eta$ for simplicity)
\begin{align}\label{eq:exp_smooth}
    \E_t [ f(\vx^{t+1})] \le f(\vx^t) - \eta \E_t \dotp{ \nabla f(\vx^t)}{\vv^t} + \frac{L\eta^2}{2}\E_t\norm{\vv^t}^2.
\end{align}
To control $\E_t \dotp{ \nabla f(\vx^t)}{\vv^t}$, notice that
\begin{align*}
    \E_t \dotp{ \nabla f(\vx^t)}{\vv^t} = \E_t \dotp{ \nabla f(\vx^t)}{ \frac{1}{n}\sum_{i=1}^n \vv_i^{t}} 
   &  = \E_t \dotp{ \nabla f(\vx^t)}{\frac{1}{n}\sum_{i=1}^n \gC_i^t(\vg_i^t)} \\
   & \overset{\eqref{eq:comp}}{=} \E_t \dotp{ \nabla f(\vx^t)}{\frac{1}{n}\sum_{i=1}^n \vg_i^t} \\
   &  = \E_t \dotp{ \nabla f(\vx^t)}{\frac{1}{n}\sum_{i=1}^n ( \tgit  + \vxi_i^{t})} \\
   & =  \norm{\nabla f(\vx^t) }^2,
\end{align*}
where the last line follows from Lemma~\ref{lem:variance} as well as the independence of the added Gaussian perturbation. Next, using the definition $\vv_i^{t} = \gC_i^t(\vg_i^t)$ and the properties of the compression operator, we compute $\E_t\norm{\vv^t}^2$ as follows,
\begin{align}
    \E_t\norm{\vv^t}^2 &= \E_t \norm{\frac{1}{n}\sum_{i=1}^n \vv_i^{t}}^2 \nonumber \\
  &  = \E_t\left[ \norm{\frac{1}{n}\sum_{i=1}^n (\vv_i^{t}-\vg_i^t)}^2  +  \norm{\frac{1}{n}\sum_{i=1}^n \vg_i^t}^2 \right] \nonumber \\
  & \overset{\eqref{eq:comp}}{\le} \E_t\left[\frac{1}{n^2}\sum_{i=1}^n\omega\norm{ \vg_i^t}^2 + \norm{\frac{1}{n}\sum_{i=1}^n \vg_i^t}^2\right] \nonumber \\
  & \le \frac{1}{n^2}\sum_{i=1}^n\omega\left(\norm{\nabla f_i(\vx^t)}^2 + \frac{(1-q) \E_t \norm{\tgit - 
  \nabla f_i(\vx^t)}^2}{b} + d\sigma_p^2\right) \nonumber \\
  & \quad\quad + \norm{\nabla f(\vx^t)}^2 + \frac{(1-q) \E_t \norm{\tgit - 
  \nabla f_i(\vx^t)}^2}{bn} + \frac{ d\sigma_p^2}{n} \nonumber \\
  &  \overset{\eqref{eq:variance}}{\le} \frac{1}{n^2}\sum_{i=1}^n\omega\left(\norm{\nabla f_i(\vx^t)}^2 + \frac{(1-q)G^2}{b} + d\sigma_p^2\right) + \norm{\nabla f(\vx^t)}^2 + \frac{(1-q)G^2}{bn} + \frac{ d\sigma_p^2}{n} \nonumber \\
   &~\le \norm{\nabla f(\vx^t)}^2 + \frac{1}{n}\left(\omega G^2 + (1+\omega)\frac{(1-q)G^2}{b} + (1+\omega) d\sigma_p^2 \right), \label{eq:vt-squared-norm}
\end{align}
where the last inequality \eqref{eq:vt-squared-norm} follows from Assumption \ref{ass:bounded-gradient}.

Plugging the above two relations back to \eqref{eq:exp_smooth}, we obtain
\begin{equation*}
    \E_t [f(\vx^{t+1})] \le f(\vx^t) - \left( \eta -\frac{L \eta^2}{2} \right)  \norm{\nabla f(\vx^t)}^2 + \frac{L\eta^2}{2n}\left(\omega G^2 + (1+\omega)\frac{(1-q)G^2}{b} + (1+\omega) d\sigma_p^2 \right).
\end{equation*}

By choosing $\eta \leq \frac{1}{L}$, we have
\begin{equation*}
     \E_t [f(\vx^{t+1})] \le  f(\vx^t) - \frac{\eta}{2}  \norm{\nabla f(\vx^t)}^2 + \frac{L\eta^2}{2n}\left(\omega G^2 + (1+\omega)\frac{(1-q)G^2}{b} + (1+\omega) d\sigma_p^2 \right).
\end{equation*}
Plugging in $\sigma_p$ (Theorem \ref{thm:privacy-cdpsgd}), telescoping over the iterations $t=1,\ldots, T$, and rearranging terms, we can prove
\begin{align}
 \frac{1}{T}\sum_{t=1}^T \E\norm{\nabla f(\vx^t)}^2 
    & \le \frac{2(f(\vx^0) - f^*)}{\eta T} + \frac{L\eta}{n}\left[\omega G^2 + (1+\omega)\frac{(1-q)G^2}{b} + \frac{(1+\omega) cd G^2 T\log(1/\delta)}{m^2\epsilon^2}\right] \nonumber \\
   &  \le \frac{2 D_f}{\eta T} + \frac{L\eta}{n}\left[\frac{(1+\omega+\omega b)}{b}G^2 + \frac{(1+\omega) cd G^2 T\log(1/\delta)}{m^2\epsilon^2}\right], \label{eq:potato}
\end{align}
where we use the notation $D_f:= f(\vx^0) - f^*$.
We choose $T$ and $\eta$ to satisfy 
\begin{align}
\eta T = \frac{m\epsilon \sqrt{n D_f}}{G\sqrt{L(1+\omega) cd\log(1/\delta)}}, \quad 
T\geq \frac{m^2\epsilon^2}{cd \log\left(1/\delta\right)}. \label{eq:T-cdpsgd}
\end{align}
According to the relation \eqref{eq:T-cdpsgd} and stepsize $\eta \leq \frac{1}{L}$, we set $T=\max\Big\{\frac{m\epsilon \sqrt{n L D_f}}{G\sqrt{(1+\omega) cd\log(1/\delta)}}, \frac{m^2\epsilon^2}{cd\log(1/\delta)}\Big\}$ and $\eta = \min\Big\{\frac{1}{L}, \frac{\sqrt{nD_f cd \log(1/\delta)}}{G m\epsilon\sqrt{(1+\omega)L}}\Big\}$.
Then \eqref{eq:potato} turns out as 
\begin{align*}
    \frac{1}{T}\sum_{t=1}^T \E\norm{\nabla f(\vx^t)}^2 & \le \frac{2 D_f}{\eta T} + \frac{L\eta}{n}\left[2(1+\omega)G^2 + \frac{(1+\omega) cd G^2 T\log(1/\delta)}{m^2 \epsilon^2}\right] \\
    & \overset{\eqref{eq:T-cdpsgd}}{\le} \frac{2 D_f}{\eta T} + \frac{L\eta}{n}\cdot \frac{3(1+\omega) cd G^2 T\log(1/\delta)}{m^2\epsilon^2} \\
   & \overset{\eqref{eq:T-cdpsgd}}{\le} \frac{2G\sqrt{D_f L(1+\omega) cd\log(1/\delta)}}{m\epsilon\sqrt{n}} + \frac{3G\sqrt{D_f L(1+\omega) cd\log(1/\delta)}}{m\epsilon\sqrt{n}} \\
   & = O\left(\frac{G\sqrt{L(1+\omega) d\log(1/\delta)}}{m\epsilon\sqrt{n}}\right).
\end{align*}
\end{proof}

\section{Proofs for Section~\ref{sec:algorithms}}
\label{sec:proof-lemma}
Now we provide the proofs for the proposed \soteriafl-style algorithms. 
Appendix~\ref{sec:proof-lemma-1} gives the proofs for Lemma~\ref{lem:para-sgd-svrg-saga} which shows that some classical local gradient estimators (SGD/SVRG/SAGA) satisfy our generic Assumption~\ref{ass:unified}. Appendix~\ref{sec:proof-soterialfl-algorithms} provides the proofs for Corollaries~\ref{cor:sgd}--\ref{cor:saga} which 
instantiate Lemma~\ref{lem:para-sgd-svrg-saga} in the unified Theorem~\ref{thm:utility} for obtaining detailed results for the proposed \soteriafl-style algorithms.

\subsection{Proof of Lemma~\ref{lem:para-sgd-svrg-saga}}
\label{sec:proof-lemma-1}

% For convenience, we recall Lemma~\ref{lem:para-sgd-svrg-saga} below.

% \restatelem{\ref{lem:para-sgd-svrg-saga}}{  
%     \begin{lemma}[SGD/SVRG/SAGA estimators satisfy Assumption~\ref{ass:unified}]
%     	Suppose that Assumptions~\ref{ass:smoothness} and \ref{ass:bounded-gradient} hold. 
%     	The local SGD estimator $\tilde{\vg}_i^t$ (Option \RomanNumeralCaps{1} in Algorithm~\ref{alg:soteriafl-detail}) satisfies Assumption~\ref{ass:unified} with 
%     	\begin{align*}
%     	G_A = G,~
%     	G_B = C_1 = C_3 = C_4 = 0,~ 
%     	C_2 = \frac{(m-b)G^2}{mb},~ 
%     	\theta = 1,~ 
%     	\Delta^t \equiv 0.
%     	\end{align*}
%     	The local SVRG estimator $\tilde{\vg}_i^t$ (Option \RomanNumeralCaps{2} in Algorithm~\ref{alg:soteriafl-detail}) satisfies Assumption~\ref{ass:unified} with 
%     	\begin{align*}
%     	G_A = 2G,~ 
%     	G_B = G,~
%     	C_1 = \frac{L^2}{b},~
%     	C_2 = 0,~ 
%     	C_3 = \frac{2(1-p)\eta^2}{p},~ 
%     	C_4 = 1,~
%     	\theta = \frac{p}{2},~ 
%     	\Delta^t = \ns{\vx^t-\vw^t}.
%     	\end{align*}
%     	The local SAGA estimator $\tilde{\vg}_i^t$ (Option \RomanNumeralCaps{3} in Algorithm~\ref{alg:soteriafl-detail}) satisfies Assumption~\ref{ass:unified} with 
%     	\begin{equation*}
%     	\begin{split}
%     	&G_A = 2G,~ 
%     	G_B = G,~
%     	C_1 = \frac{L^2}{b},~
%     	C_2 = 0,~ 
%     	C_3 = \frac{2(m-b)\eta^2}{b},~ 
%     	C_4 = 1,~ \\
%     	&\theta = \frac{b}{2m},~ 
%     	\Delta^t = \frac{1}{nm}\sum_{i=1}^n\sum_{j=1}^m\ns{\vx^t-\vw_{i,j}^t}.
%     	\end{split}
%     	\end{equation*}
%     \end{lemma}
% }

% \begin{proof}[Proof of Lemma~\ref{lem:para-sgd-svrg-saga}] 

We shall prove each case one by one. 
	\paragraph{The SGD estimator.} For the local SGD estimator $\tilde{\vg}_i^t =\frac{1}{b} \sum_{j\in \gI_b} \nabla f_{i,j}(\vx^t)$ (Option \RomanNumeralCaps{1} in Algorithm~\ref{alg:soteriafl-detail}), we first show that it is unbiased. To facilitate analysis, for client $i$, we introduce $m$ independent Bernoulli random variables $X_{i,j}$, where $\Pr\{X_{i,j} = 1\} = \frac{b}{m}$.
	 We have
	\begin{align*}
    	\EtB{ \frac{1}{b}\sum_{j\in \gI_b}\nabla f_{i,j}(\xt)}
    	= \EtB{ \frac{1}{b}\sum_{j=1}^m X_{i,j} \nabla f_{i,j}(\xt) }
    	= \frac{1}{m}\sum_{j=1}^m \nabla f_{i,j}(\xt) 
    	= \nabla f_i(\xt).
	\end{align*}
	Then we show that \eqref{ass:decompose}--\eqref{ass:var-Delta} are satisfied for some concrete parameters.
	For \eqref{ass:decompose}, let 
	$$ \gA_i^t =\frac{1}{b} \sum_{j\in \gI_b} \nabla f_{i,j}(\vx^t), \quad \mbox{and} \quad \gB_i^t =0,$$ 
	i.e., $\varphi_{i,j}^t = \nabla f_{i,j}(\vx^t)$ and $\psi_{i,j}^t=0$. Then, $G_A=G$ (Assumption~\ref{ass:bounded-gradient}) and $G_b=0$.
	For \eqref{ass:var-g}, we have
	\begin{align}
	    \EtB{\frac{1}{n}\sumn \ns{\tgit - \nabla f_i(\xt)}} 
	    &=  \EtB{\frac{1}{n}\sumn \nsb{\frac{1}{b} \sum_{j\in \gI_b} \nabla f_{i,j}(\vx^t)- \nabla f_i(\xt)}} \notag\\
	    &= \EtB{ \frac{1}{n}\sumn \nsb{\frac{1}{b} \sum_{j=1}^m X_{i,j}\nabla f_{i,j}(\vx^t)- \nabla f_i(\xt)} } \notag\\
	    &= \EtB{ \frac{1}{n}\sumn \nsb{\frac{1}{m} \sum_{j=1}^m \left( \frac{m}{b}X_{i,j}-1\right)\nabla f_{i,j}(\vx^t)} } \notag\\
	    &= \frac{1}{n}\sumn \frac{m-b}{m^2b}\sum_{j=1}^m \nsb{\nabla f_{i,j}(\vx^t)}  \notag\\
	    &\leq \frac{(m-b)G^2}{mb}, \label{eq:use-boundg-1}
	\end{align}
	where \eqref{eq:use-boundg-1} uses Assumption~\ref{ass:bounded-gradient}.
	According to \eqref{eq:use-boundg-1}, we know that the SGD estimator $\tilde{\vg}_i^t$ satisfies \eqref{ass:var-g} and \eqref{ass:var-Delta} with
	\begin{align*}
	C_1 = C_3 = C_4 = 0,~ 
	C_2 = \frac{(m-b)G^2}{mb},~ 
	\theta = 1,~ 
	\Delta^t \equiv 0.
	\end{align*}
	
	\paragraph{The SVRG estimator.} For the local SVRG estimator $\tilde{\vg}_i^t =\frac{1}{b} \sum_{j\in \gI_b} (\nabla f_{i,j}(\vx^t)- \nabla f_{i,j}(\vw^t)) +\nabla f_i(\vw^t)$ (Option \RomanNumeralCaps{2} in Algorithm~\ref{alg:soteriafl-detail}), similarly we first show that it is unbiased as follows,
	\begin{align*}
    	\EtB{ \frac{1}{b} \sum_{j\in \gI_b} (\nabla f_{i,j}(\xt)- \nabla f_{i,j}(\wt)) +\nabla f_i(\wt)}
    	&= \EtB{ \frac{1}{b} \sum_{j=1}^m X_{i,j} (\nabla f_{i,j}(\xt)- \nabla f_{i,j}(\wt)) +\nabla f_i(\wt)} \notag\\
    	&= \frac{1}{m}\sum_{j=1}^m (\nabla f_{i,j}(\xt) - \nabla f_{i,j}(\wt)) + \nabla f_i(\wt) \notag\\
    	&= \nabla f_i(\xt) - \nabla f_i(\wt) + \nabla f_i(\wt) \notag\\
    	&= \nabla f_i(\xt).
	\end{align*}
	Then we show that \eqref{ass:decompose}--\eqref{ass:var-Delta} are satisfied for some concrete parameters.
	For \eqref{ass:decompose}, let 
	$$\gA_i^t =\frac{1}{b} \sum_{j\in \gI_b} (\nabla f_{i,j}(\vx^t) - \nabla f_{i,j}(\vw^t)), \quad \mbox{and} \quad \gB_i^t =\frac{1}{m}\sum_{j=1}^m \nabla f_{i,j}(\wt),$$ 
	i.e., $\varphi_{i,j}^t = \nabla f_{i,j}(\vx^t) - \nabla f_{i,j}(\vw^t)$ and $\psi_{i,j}^t=\nabla f_{i,j}(\wt)$. Then, $G_A=2G$ and $G_b=G$ due to Assumption~\ref{ass:bounded-gradient}.
	For \eqref{ass:var-g}, we have 
	\begin{align}
	    &\EtB{\frac{1}{n}\sumn \ns{\tgit - \nabla f_i(\xt)}} \notag\\
	    &=  \EtB{\frac{1}{n}\sumn \nsb{\frac{1}{b} \sum_{j\in \gI_b} \big(\nabla f_{i,j}(\xt)- \nabla f_{i,j}(\wt)\big) +\nabla f_i(\wt) - \nabla f_i(\xt)}} \notag\\
	    &= \EtB{ \frac{1}{n}\sumn \nsb{\frac{1}{b} \sum_{j=1}^m X_{i,j}\big(\nabla f_{i,j}(\xt)- \nabla f_{i,j}(\wt)\big)- \big(\nabla f_i(\xt)-\nabla f_i(\wt)\big)} } \notag\\
	    &= \EtB{ \frac{1}{n}\sumn \nsb{\frac{1}{m} \sum_{j=1}^m \left( \frac{m}{b}X_{i,j}-1\right)\big(\nabla f_{i,j}(\xt)- \nabla f_{i,j}(\wt)\big)} } \notag\\
	    &= \frac{1}{n}\sumn \frac{m-b}{m^2b}\sum_{j=1}^m \nsb{\nabla f_{i,j}(\xt)- \nabla f_{i,j}(\wt)}  \notag\\
	    &\leq \frac{L^2}{b}\ns{\xt-\wt}, \label{eq:use-smooth-2}
	\end{align}
	where \eqref{eq:use-smooth-2} uses Assumption~\ref{ass:smoothness}.
	According to \eqref{eq:use-smooth-2}, we know that the SVRG estimator $\tilde{\vg}_i^t$ satisfies \eqref{ass:var-g} with
	\begin{align*}
	C_1 = \frac{L^2}{b},~ 
	C_2 = 0,~ 
	\Delta^t = \ns{\vx^t-\vw^t}.
	\end{align*}
	Finally, for \eqref{ass:var-Delta}, we have 
	\begin{align}
	    \E_t\left[\Delta^{t+1}\right]
	    &= \EtB{ \ns{\xtn-\wtn} } \notag\\
	    &= \EtB{p \ns{\xtn -\xt} +(1-p)\ns{\xtn - \wt} } \label{eq:use-prob-2}\\
	    &= \EtB{p \ns{\xtn -\xt} +(1-p)\ns{\xtn - \xt + \xt - \wt}} \notag\\
	    &= \EtB{\ns{\xtn -\xt}} + \EtB{(1-p)\ns{\xt - \wt}+2(1-p)\innerB{\xtn -\xt}{\xt - \wt}} \notag\\
	    &= \EtB{\ns{\xtn -\xt}} + \EtB{(1-p)\ns{\xt - \wt}+2(1-p)\innerB{-\etat\vvt}{\xt - \wt}} \notag\\
	    &\overset{\eqref{eq:useunbiase}}{=} \EtB{\ns{\xtn -\xt}} + \EtB{(1-p)\ns{\xt - \wt}+2(1-p)\innerB{-\etat\nabla f(\xt)}{\xt - \wt}} \notag\\
	    &~\leq \EtB{\ns{\xtn -\xt}} \notag\\ 
	        &\qquad + \EtB{(1-p)\ns{\xt - \wt}
	        +\frac{(1-p)p}{2}\ns{\xt - \wt} + \frac{2(1-p)\etat^2}{p}\ns{\nabla f(\xt)}} \label{eq:use-young-2} \\
	    &~\leq \big(1-\frac{p}{2}\big)\ns{\xt - \wt} + \frac{2(1-p)\eta^2}{p}\ns{\nabla f(\xt)} + \EtB{\ns{\xtn -\xt}} \label{eq:var-svrg}
	\end{align}
	where \eqref{eq:use-prob-2} uses the update rule of $\wtn$ (Line~\ref{line:w_prob-svrg} of Algorithm~\ref{alg:soteriafl-detail}), \eqref{eq:use-young-2} uses Young's inequality, and the last inequality holds by choosing $\eta \geq \etat$.
	According to \eqref{eq:var-svrg}, we know that the SVRG estimator $\tilde{\vg}_i^t$ satisfies \eqref{ass:var-Delta} with
	\begin{align*}
	\theta = \frac{p}{2},~
	C_3 = \frac{2(1-p)\eta^2}{p},~ 
	C_4 = 1.
	\end{align*}

	\paragraph{The SAGA estimator.} For the local SAGA estimator $\tilde{\vg}_i^t  = \frac{1}{b} \sum_{j\in \gI_b} (\nabla f_{i,j}(\vx^t)- \nabla f_{i,j}(\vw_{i,j}^t)) + \frac{1}{m}\sum_{j=1}^m\nabla f_{i,j}(\vw_{i,j}^t)$ (Option \RomanNumeralCaps{3} in Algorithm~\ref{alg:soteriafl-detail}), similarly we first show that it is unbiased as follows,
	\begin{align*}
    	&\EtB{ \frac{1}{b} \sum_{j\in \gI_b} (\nabla f_{i,j}(\xt)- \nabla f_{i,j}(\wijt)) + \frac{1}{m}\sum_{j=1}^m\nabla f_{i,j}(\wijt)} \notag\\
    	&= \EtB{ \frac{1}{b} \sum_{j=1}^m X_{i,j} (\nabla f_{i,j}(\xt)- \nabla f_{i,j}(\wijt)) + \frac{1}{m}\sum_{j=1}^m\nabla f_{i,j}(\wijt)} \notag\\
    	&= \frac{1}{m}\sum_{j=1}^m (\nabla f_{i,j}(\xt) - \nabla f_{i,j}(\wijt)) +  \frac{1}{m}\sum_{j=1}^m\nabla f_{i,j}(\wijt) \notag\\
    	&= \frac{1}{m}\sum_{j=1}^m \nabla f_{i,j}(\xt)
    	= \nabla f_i(\xt).
	\end{align*}
	Then we show that \eqref{ass:decompose}--\eqref{ass:var-Delta} are satisfied for some concrete parameters.
	For \eqref{ass:decompose}, let 
	$$\gA_i^t =\frac{1}{b} \sum_{j\in \gI_b} (\nabla f_{i,j}(\vx^t)- \nabla f_{i,j}(\vw_{i,j}^t))\quad \mbox{and} \quad \gB_i^t =\frac{1}{m}\sum_{j=1}^m \nabla f_{i,j}(\vw_{i,j}^t),$$ 
	i.e., $\varphi_{i,j}^t = \nabla f_{i,j}(\vx^t) - \nabla f_{i,j}(\vw_{i,j}^t)$ and $\psi_{i,j}^t=\nabla f_{i,j}(\vw_{i,j}^t)$. Then, $G_A=2G$ and $G_b=G$ due to Assumption~\ref{ass:bounded-gradient}.
	For \eqref{ass:var-g}, we have 
	\begin{align}
	    &\EtB{\frac{1}{n}\sumn \ns{\tgit - \nabla f_i(\xt)}} \notag\\
	    &=  \EtB{\frac{1}{n}\sumn \nsb{\frac{1}{b} \sum_{j\in \gI_b} \big(\nabla f_{i,j}(\xt)- \nabla f_{i,j}(\wijt)\big) + \frac{1}{m} \sum_{j=1}^m \big(\nabla f_{i,j}(\wijt) - \nabla f_{i,j}(\xt)\big)}} \notag\\
	    &= \EtB{ \frac{1}{n}\sumn \nsb{\frac{1}{b} \sum_{j=1}^m X_{i,j}\big(\nabla f_{i,j}(\xt)- \nabla f_{i,j}(\wt)\big) - \frac{1}{m} \sum_{j=1}^m \big(\nabla f_{i,j}(\xt) - \nabla f_{i,j}(\wijt)\big)} } \notag\\
	    &= \EtB{ \frac{1}{n}\sumn \nsb{\frac{1}{m} \sum_{j=1}^m \left( \frac{m}{b}X_{i,j}-1\right)\big(\nabla f_{i,j}(\xt)- \nabla f_{i,j}(\wijt)\big)} } \notag\\
	    &= \frac{1}{n}\sumn \frac{m-b}{m^2b}\sum_{j=1}^m \nsb{\nabla f_{i,j}(\xt)- \nabla f_{i,j}(\wijt)}  \notag\\
	    &\leq \frac{L^2}{b}\frac{1}{nm}\sum_{i=1}^n\sum_{j=1}^m\ns{\xt-\wijt}, \label{eq:use-smooth-3}
	\end{align}
	where \eqref{eq:use-smooth-3} uses Assumption~\ref{ass:smoothness}.
	According to \eqref{eq:use-smooth-3}, we know that the SAGA estimator $\tilde{\vg}_i^t$ satisfies \eqref{ass:var-g} with
	\begin{align*}
	C_1 = \frac{L^2}{b},~ 
	C_2 = 0,~ 
	\Delta^t = \frac{1}{nm}\sum_{i=1}^n\sum_{j=1}^m\ns{\xt-\wijt}.
	\end{align*}
	Finally, for \eqref{ass:var-Delta}, we have 
	\begin{align}
	    &\E_t\left[\Delta^{t+1}\right] \notag\\
	    &= \EtB{ \frac{1}{nm}\sum_{i=1}^n\sum_{j=1}^m\ns{\xtn-\wijtn}} \notag\\
	    &= \EtB{\frac{1}{nm}\sum_{i=1}^n\sum_{j=1}^m \Big(\frac{b}{m}\ns{\xtn-\xt} + \Big(1-\frac{b}{m}\Big)\ns{\xtn -\wijt}\Big) } \label{eq:use-prob-3}\\
	    &= \EtB{\frac{1}{nm}\sum_{i=1}^n\sum_{j=1}^m \Big(\frac{b}{m}\ns{\xtn-\xt} + \Big(1-\frac{b}{m}\Big)\ns{\xtn - \xt + \xt - \wijt}\Big) } \notag\\
	    &= \EtB{\ns{\xtn -\xt}}   + \EtB{\Big(1-\frac{b}{m}\Big)\frac{1}{nm}\sum_{i=1}^n\sum_{j=1}^m \Big(\ns{\xt - \wijt} + 2 \innerB{\xtn -\xt}{\xt - \wijt}\Big)} \notag\\
	    &= \EtB{\ns{\xtn -\xt}} + \EtB{\Big(1-\frac{b}{m}\Big)\frac{1}{nm}\sum_{i=1}^n\sum_{j=1}^m \Big(\ns{\xt - \wijt} + 2 \innerB{-\etat\vvt}{\xt - \wijt}\Big)} \notag\\
	    &\overset{\eqref{eq:useunbiase}}{=} \EtB{\ns{\xtn -\xt}}  + \EtB{\Big(1-\frac{b}{m}\Big)\frac{1}{nm}\sum_{i=1}^n\sum_{j=1}^m \Big(\ns{\xt - \wijt} + 2 \innerB{-\etat\nabla f(\xt)}{\xt - \wijt}\Big)} \notag\\
	   &~\leq \EtB{\ns{\xtn -\xt}}  + \EtB{\Big(1-\frac{b}{m}\Big)\frac{1}{nm}\sum_{i=1}^n\sum_{j=1}^m \Big(\big(1+\frac{b}{2m} \big)\ns{\xt - \wijt} + \frac{2m\etat^2}{b}\ns{\nabla f(\xt)}\Big)} \label{eq:use-young-3} \\
	    &~\leq \Big(1-\frac{b}{2m}\Big)\frac{1}{nm}\sum_{i=1}^n\sum_{j=1}^m\ns{\xt-\wijt} + \frac{2(m-b)\eta^2}{b}\ns{\nabla f(\xt)} + \EtB{\ns{\xtn -\xt}}, \label{eq:var-saga}
	\end{align}
	where \eqref{eq:use-prob-3} uses the update rule of $\wijtn$ (Line~\ref{line:w_prob-saga} of Algorithm~\ref{alg:soteriafl-detail}), \eqref{eq:use-young-3} uses Young's inequality, and the last inequality holds by choosing $\eta \geq \etat$.
	According to \eqref{eq:var-saga}, we know that the SAGA estimator $\tilde{\vg}_i^t$ satisfies \eqref{ass:var-Delta} with
	\begin{align*}
	\theta = \frac{b}{2m},~
	C_3 = \frac{2(m-b)\eta^2}{b},~ 
	C_4 = 1.
	\end{align*}

%\end{proof}

\subsection{Proofs for \soteriafl-style Algorithms}
\label{sec:proof-soterialfl-algorithms}
We provide detailed corollaries and their proofs for the proposed \soteriafl-style algorithms (\soteriaflGD, \soteriaflSGD, \soteriaflSVRG, and \soteriaflSAGA). These corollaries are obtained by plugging their corresponding parameters given in Lemma~\ref{lem:para-sgd-svrg-saga} into our unified Theorem~\ref{thm:utility}.

\paragraph{Analysis of \soteriaflSGD~/~\soteriaflGD (Proof of Corollary~\ref{cor:sgd}).}
% We first recall the corollary here and then provide its proof.
% \restatecor{\ref{cor:sgd}}{  
%     \begin{corollary}[\soteriaflSGD/\soteriaflGD]
%         Suppose that Assumptions~\ref{ass:smoothness} and \ref{ass:bounded-gradient} hold and we combine Theorem~\ref{thm:utility} and Lemma~\ref{lem:para-sgd-svrg-saga}, i.e., choosing stepsize
%         $
%         \eta_t \equiv \eta \leq  \frac{1}{(1+2\sqrt{(1+\omega)^3/n})L},
%         $
%         where we set $\beta=\frac{\tau}{2(1+\omega)}$ and  $\tau:=\frac{(1+\omega)^{3/2}}{n^{1/2}}$, 
%         shift stepsize $\gamma_t\equiv \sqrt{\frac{1+2\omega}{2(1+\omega)^3}}$, and privacy variance $\sigma_p^2 = O\big(\frac{G^2T\log({1}/{\delta})}{m^2\epsilon^2}\big)$.
%         If we further set the minibatch size $b= \min\Big\{ \frac{m \epsilon G \sqrt{\beta}}{\sqrt{(1+\omega)Ld\log(1/\delta)}}, m\Big\}$ and the total number of communication rounds 
%         $
%             T = O\Big( \frac{\sqrt{nL} m\epsilon}{G\sqrt{(1+\omega)d \log(1/\delta)}}(1+\sqrt{\tau})\Big),
%         $
%         then \soteriaflSGD satisfies $(\epsilon,\delta)$-LDP and the following utility guarantee
%         $
%             \frac{1}{T}\sum_{t=0}^{T-1} \E\|\nabla f(\vx_t)\|^2 \le O\Big(\frac{G\sqrt{(1+\omega)Ld\log(1/\delta)}}{\sqrt{n}m\epsilon} (1+\sqrt{\tau})\Big).
%         $
%         If we choose a minibatch size $b=m$ (local full gradient) in \soteriaflSGD, the result of \soteriaflSGD leads to that of \soteriaflGD.            
%     \end{corollary}
% }

%\begin{proof}[Proof of Corollary~\ref{cor:sgd}]
    We first show that the stepsize $\eta_t$ chosen in this corollary satisfies the conditions in  Theorem~\ref{thm:utility}. 
    According to the corresponding parameters for the SGD estimator in  Lemma~\ref{lem:para-sgd-svrg-saga}
    \begin{align}
        G_A = G,~
    	G_B = C_1 = C_3 = C_4 = 0,~ 
    	C_2 = \frac{(m-b)G^2}{mb},~ 
    	\theta = 1,~ 
    	\Delta^t \equiv 0,\label{eq:para-sgd}
    \end{align}
    we have $\alpha = \frac{3\beta C_1}{2(1+\omega)\theta L^2}=0$.
    Then the stepsize $\eta_t \equiv \eta$ required in Theorem~\ref{thm:utility} reads  
    \begin{align}
        \eta_t \equiv \eta 
        &\leq \min\left\{ \frac{1}{(1+2\alpha C_4 + 4\beta(1+\omega) + 2\alpha C_3/\eta^2)L}, 
        \frac{\sqrt{\beta n}}{\sqrt{1+2\alpha C_4 + 4\beta(1+\omega)}(1+\omega)L}\right\} \notag\\
        &= \min\left\{ \frac{1}{(1+ 4\beta(1+\omega))L}, 
        \frac{\sqrt{\beta n}}{\sqrt{1+ 4\beta(1+\omega)}(1+\omega)L}\right\}. \label{eq:eta-cor-1}
    \end{align}
    Let $\tau:=\frac{(1+\omega)^{3/2}}{n^{1/2}}$. If we set $\beta=\frac{\tau}{2(1+\omega)}$,  then $\etat \equiv \eta \leq  \frac{1}{(1+2\tau)L}$ satisfies \eqref{eq:eta-cor-1}.
    Then according to Theorem~\ref{thm:utility} and the parameters in \eqref{eq:para-sgd}, if we choose the shift stepsize $\gamma_t\equiv \sqrt{\frac{1+2\omega}{2(1+\omega)^3}}$, and the privacy variance $\sigma_p^2 = O\big(\frac{G^2T\log({1}/{\delta})}{m^2\epsilon^2}\big)$, \soteriaflSGD satisfies $(\epsilon,\delta)$-LDP and the following 
    \begin{align*}
        \frac{1}{T}\sum_{t=0}^{T-1} \E\ns{\nabla f(\vx^t)} 
        &\le \frac{2\Phi_0}{\eta T} 
        + \frac{3\beta}{(1+\omega)L\eta}\left(\frac{(m-b)G^2}{mb} + \frac{cG^2dT\log ({1}/{\delta})}{4m^2\epsilon^2}\right).
    \end{align*}
    By further choosing $T$ as  
    \begin{align}
		T &=  \max\left\{\frac{m\epsilon\sqrt{8(1+\omega)L\Phi_0}}{\sqrt{3\beta cdG^2\log({1}/{\delta})}}, \frac{4(m-b)m^2\epsilon^2}{c mbd\log({1}/{\delta})}\right\}, \label{eq:choice-T-cor-1}
	\end{align}
	\soteriafl has the following utility (accuracy) guarantee: 
	\begin{align*}
		\frac{1}{T}\sum_{t=0}^{T-1} \E\ns{\nabla f(\vx^t)} 
		&\le O\left( \max\left\{\frac{\sqrt{\beta dG^2\log({1}/{\delta})}}{\eta m\epsilon\sqrt{(1+\omega)L}}, \frac{(m-b)\beta G^2}{(1+\omega)mbL\eta}\right\}\right). 
	\end{align*}
    If we further set the minibatch size $b= \min\Big\{ \frac{m \epsilon G \sqrt{\beta}}{\sqrt{(1+\omega)Ld\log(1/\delta)}}, m\Big\}$, we have $\frac{(m-b)\beta G^2}{(1+\omega)mbL\eta} \leq \frac{\sqrt{\beta dG^2\log({1}/{\delta})}}{\eta m\epsilon\sqrt{(1+\omega)L}}$ and thus 
    \begin{align}
		\frac{1}{T}\sum_{t=0}^{T-1} \E\ns{\nabla f(\vx^t)} 
		&\le O\left( \frac{\sqrt{\beta dG^2\log({1}/{\delta})}}{\eta m\epsilon\sqrt{(1+\omega)L}}\right). \label{eq:utility-cor-1}
	\end{align}
	Then by plugging the parameters $\beta$, $\eta$, and $b$ into \eqref{eq:choice-T-cor-1} and \eqref{eq:utility-cor-1}, we obtain
    $
        T = O\Big( \frac{\sqrt{nL} m\epsilon}{G\sqrt{(1+\omega)d \log(1/\delta)}}(1+\sqrt{\tau})\Big),
    $
    and 
    $
        \frac{1}{T}\sum_{t=0}^{T-1} \E\|\nabla f(\vx_t)\|^2 \le O\Big(\frac{G\sqrt{(1+\omega)Ld\log(1/\delta)}}{\sqrt{n}m\epsilon} (1+\sqrt{\tau})\Big).
    $
    
    For \soteriaflGD in which the minibatch size $b=m$, we have $\frac{(m-b)\beta G^2}{(1+\omega)mbL\eta} =0 \leq \frac{\sqrt{\beta dG^2\log({1}/{\delta})}}{\eta m\epsilon\sqrt{(1+\omega)L}}$, thus the same results hold for \soteriaflGD as well.

%\end{proof}

\paragraph{Analysis of \soteriaflSVRG (Proof of Corollary~\ref{cor:svrg}).}
%Similarly, we first recall the corollary here and then provide its proof.
% \restatecor{\ref{cor:svrg}}{  
%     \begin{corollary}[\soteriaflSVRG]
%         Suppose that Assumptions~\ref{ass:smoothness} and \ref{ass:bounded-gradient} hold and we combine Theorem~\ref{thm:utility} and Lemma~\ref{lem:para-sgd-svrg-saga}, i.e., choosing stepsize
%         $
%             \eta_t \equiv \eta \leq  \frac{p^{2/3}b^{1/3}{\min\{1, \sqrt{n/(1+\omega)^3}\}}}{2L},
%         $
%         where we set $\beta = \frac{p^{4/3}b^{2/3}(1+\omega)^2\min\{1, n/(1+\omega)^3\}}{n}$, $p^{2/3}b^{1/3}\leq 1/4$ and $p\leq 1/4$,
%         shift stepsize $\gamma_t\equiv \sqrt{\frac{1+2\omega}{2(1+\omega)^3}}$, and privacy variance $\sigma_p^2 = O\big(\frac{G^2T\log({1}/{\delta})}{m^2\epsilon^2}\big)$.
%         If we further let the minibatch size $b=\frac{m^{2/3}}{4}$, the probability $p=b/m$, and the total number of communication rounds 
%         $
%             T 
%              = O\Big( \frac{\sqrt{nL} m\epsilon}{G\sqrt{(1+\omega)d \log(1/\delta)}} \max\big\{1, \tau \big\}\Big),
%         $
%         where $\tau:=\frac{(1+\omega)^{3/2}}{n^{1/2}}$,
%         then \soteriaflSVRG satisfies $(\epsilon,\delta)$-LDP
%         and the following utility guarantee
%         $
%             \frac{1}{T}\sum_{t=0}^{T-1} \E\ns{\nabla f(\vx^t)}
%             \le O\Big( \frac{G\sqrt{{(1+\omega)}Ld\log(1/\delta)}}{\sqrt{{n}}m\epsilon}\Big).
%         $
%     \end{corollary}
% }

%\begin{proof}[Proof of Corollary~\ref{cor:svrg}]
    We first show that the stepsize $\eta_t$ chosen in this corollary satisfies the conditions in  Theorem~\ref{thm:utility}. 
    According to the corresponding parameters for the SVRG estimator in  Lemma~\ref{lem:para-sgd-svrg-saga}
    \begin{align}
        G_A = 2G,~ 
        G_B = G,~
        C_1 = \frac{L^2}{b},~
        C_2 = 0,~ 
        C_3 = \frac{2(1-p)\eta^2}{p},~ 
        C_4 = 1,~
        \theta = \frac{p}{2},~ 
        \Delta^t = \ns{\vx^t-\vw^t}, \label{eq:para-svrg}
    \end{align}
    we have $\alpha = \frac{3\beta C_1}{2(1+\omega)\theta L^2}=\frac{3\beta}{(1+\omega)p b}$. 
    Then the stepsize $\eta_t \equiv \eta$ required in Theorem~\ref{thm:utility} reads 
    \begin{align}
        \eta 
        &\leq \min\left\{ \frac{1}{(1+2\alpha C_4 + 4\beta(1+\omega) + 2\alpha C_3/\eta^2)L}, 
        \frac{\sqrt{\beta n}}{\sqrt{1+2\alpha C_4 + 4\beta(1+\omega)}(1+\omega)L}\right\} \notag\\
        &= \min\left\{ \frac{1}{\big(1+ \frac{6\beta}{(1+\omega)pb} + 4\beta(1+\omega) + \frac{12(1-p)\beta}{(1+\omega)p^2b}\big)L}, 
        \frac{\sqrt{\beta n}}{\sqrt{1+ \frac{6\beta}{(1+\omega)pb} + 4\beta(1+\omega)}(1+\omega)L}\right\}. \label{eq:eta-cor-2}
    \end{align}
    Let $\tau:=\frac{(1+\omega)^{3/2}}{n^{1/2}}$. If we set $\beta=\frac{p^{4/3}b^{2/3}(1+\omega)^2\min\{1, 1/\tau^2\}}{n}$, $p^{2/3}b^{1/3}\leq 1/4$ and $p\leq 1/4$,  then $\etat \equiv \eta \leq  \frac{p^{2/3}b^{1/3}{\min\{1, 1/\tau\}}}{2L}$ satisfies \eqref{eq:eta-cor-2}.
    Then according to Theorem~\ref{thm:utility} and the parameters in \eqref{eq:para-svrg}, if we choose the shift stepsize $\gamma_t\equiv \sqrt{\frac{1+2\omega}{2(1+\omega)^3}}$, and privacy variance $\sigma_p^2 = O\big(\frac{G^2T\log({1}/{\delta})}{m^2\epsilon^2}\big)$, \soteriaflSVRG satisfies $(\epsilon,\delta)$-LDP and the following 
    \begin{align*}
        \frac{1}{T}\sum_{t=0}^{T-1} \E\ns{\nabla f(\vx^t)} 
        &\le \frac{2\Phi_0}{\eta T} 
        + \frac{6\beta cG^2dT\log ({1}/{\delta})}{(1+\omega)L\eta m^2\epsilon^2}.
    \end{align*}
    If we further choose the minibatch size $b=\frac{m^{2/3}}{4}$, the probability $p=b/m$, and the number of communication round 
    \begin{align*}
		T &= \frac{m\epsilon\sqrt{(1+\omega)L\Phi_0}}{\sqrt{3\beta cdG^2\log({1}/{\delta})}} = O\left( \frac{\sqrt{nL} m\epsilon}{G\sqrt{(1+\omega)d \log(1/\delta)}} \max\big\{1, \tau \big\}\right), 
	\end{align*}
	we obtain 
	\begin{align*}
	    \frac{1}{T}\sum_{t=0}^{T-1} \E\ns{\nabla f(\vx^t)}
            \le O\Big( \frac{G\sqrt{{(1+\omega)}Ld\log(1/\delta)}}{\sqrt{{n}}m\epsilon}\Big).
	\end{align*}
%\end{proof}

\paragraph{Analysis of \soteriaflSAGA (Proof of Corollary~\ref{cor:saga}).}
% Now we recall the following Corollary~\ref{cor:saga} for \soteriaflSAGA and then provide its proof.
% \restatecor{\ref{cor:saga}}{  
%     \begin{corollary}[\soteriaflSAGA]
%         Suppose that Assumptions~\ref{ass:smoothness} and \ref{ass:bounded-gradient} hold and we combine Theorem~\ref{thm:utility} and Lemma~\ref{lem:para-sgd-svrg-saga}, i.e., choosing stepsize
%         $
%             \eta_t \equiv \eta \leq  \frac{\min\{1, \sqrt{n/(1+\omega)^3}\}}{3L},
%         $
%         where we set $\beta = \frac{(1+\omega)^2\min\{1, n/(1+\omega)^3\}}{3n}$, minibatch size $b=3m^{2/3}$,
%         shift stepsize $\gamma_t\equiv \sqrt{\frac{1+2\omega}{2(1+\omega)^3}}$, and privacy variance $\sigma_p^2 = O\big(\frac{G^2T\log({1}/{\delta})}{m^2\epsilon^2}\big)$.
%         If we further let the communication rounds 
%         $
%             T 
%              = O\Big( \frac{\sqrt{nL} m\epsilon}{G\sqrt{(1+\omega)d \log(1/\delta)}} \max\big\{1, \tau \big\}\Big),
%         $
%         where $\tau:=\frac{(1+\omega)^{3/2}}{n^{1/2}}$,
%         then \soteriaflSAGA satisfies $(\epsilon,\delta)$-LDP
%         and the following utility guarantee
%         $
%             \frac{1}{T}\sum_{t=0}^{T-1} \E\ns{\nabla f(\vx^t)}
%             \le O\Big( \frac{G\sqrt{{(1+\omega)}Ld\log(1/\delta)}}{\sqrt{{n}}m\epsilon}\Big).
%         $
%     \end{corollary}
% }

%\begin{proof}[Proof of Corollary~\ref{cor:saga}]
    We first show that the stepsize $\eta_t$ chosen in this corollary satisfies the conditions in Theorem~\ref{thm:utility}. 
    According to the corresponding parameters for the SAGA estimator in  Lemma~\ref{lem:para-sgd-svrg-saga}
    \begin{equation}
        \begin{split}
            &G_A = 2G,~ 
            G_B = G,~
            C_1 = \frac{L^2}{b},~
            C_2 = 0,~ 
            C_3 = \frac{2(m-b)\eta^2}{b},~ 
            C_4 = 1,~ \\
            &\theta = \frac{b}{2m},~ 
            \Delta^t = \frac{1}{nm}\sum_{i=1}^n\sum_{j=1}^m\ns{\vx^t-\vw_{i,j}^t},
        \end{split}\label{eq:para-saga}
    \end{equation}
    we have $\alpha = \frac{3\beta C_1}{2(1+\omega)\theta L^2}=\frac{3\beta m}{(1+\omega) b^2}$. 
    Then the stepsize $\eta_t \equiv \eta$ required in Theorem~\ref{thm:utility} becomes  
    \begin{align}
        \eta 
        &\leq \min\left\{ \frac{1}{(1+2\alpha C_4 + 4\beta(1+\omega) + 2\alpha C_3/\eta^2)L}, 
        \frac{\sqrt{\beta n}}{\sqrt{1+2\alpha C_4 + 4\beta(1+\omega)}(1+\omega)L}\right\} \notag\\
        &= \min\left\{ \frac{1}{\big(1+ \frac{6\beta m}{(1+\omega)b^2} + 4\beta(1+\omega) + \frac{12\beta m(m-b)}{(1+\omega)b^3}\big)L}, 
        \frac{\sqrt{\beta n}}{\sqrt{1+ \frac{6\beta m}{(1+\omega)b^2} + 4\beta(1+\omega)}(1+\omega)L}\right\}. \label{eq:eta-cor-3}
    \end{align}
    Let $\tau:=\frac{(1+\omega)^{3/2}}{n^{1/2}}$. If we set $\beta=\frac{(1+\omega)^2\min\{1, 1/\tau^2\}}{3n}$ and $b= 3m^{2/3}$, then $\etat \equiv \eta \leq  \frac{\min\{1, 1/\tau\}}{3L}$ satisfies \eqref{eq:eta-cor-3}.
    Then according to Theorem~\ref{thm:utility} and the parameters in \eqref{eq:para-saga}, if we choose the shift stepsize $\gamma_t\equiv \sqrt{\frac{1+2\omega}{2(1+\omega)^3}}$, and the privacy variance $\sigma_p^2 = O\big(\frac{G^2T\log({1}/{\delta})}{m^2\epsilon^2}\big)$, \soteriaflSAGA satisfies $(\epsilon,\delta)$-LDP and the following 
    \begin{align*}
        \frac{1}{T}\sum_{t=0}^{T-1} \E\ns{\nabla f(\vx^t)} 
        &\le \frac{2\Phi_0}{\eta T} 
        + \frac{2(1+\omega) cG^2dT\log ({1}/{\delta}) \min\{1, 1/\tau^2\}}{nL\eta m^2\epsilon^2}.
    \end{align*}
    If we further choose the number of communication round 
    \begin{align*}
		T &= \frac{m\epsilon\sqrt{n L\Phi_0}}{\sqrt{(1+\omega) cdG^2\log({1}/{\delta}) \min\{1, 1/\tau^2\}}} = O\left( \frac{\sqrt{nL} m\epsilon}{G\sqrt{(1+\omega)d \log(1/\delta)}} \max\big\{1, \tau \big\}\right), 
	\end{align*}
	we obtain 
	\begin{align*}
	    \frac{1}{T}\sum_{t=0}^{T-1} \E\ns{\nabla f(\vx^t)}
            \le O\Big( \frac{G\sqrt{{(1+\omega)}Ld\log(1/\delta)}}{\sqrt{{n}}m\epsilon}\Big).
	\end{align*}
%\end{proof}

\end{document}